\documentclass{article}

\usepackage[preprint]{neurips_2026}

\usepackage[utf8]{inputenc} 
\usepackage[T1]{fontenc}    
\usepackage{url}            
\usepackage{booktabs}       
\usepackage{amsfonts}       
\usepackage{nicefrac}       
\usepackage{xcolor}         

\usepackage{algorithm}
\usepackage{algpseudocode}

\usepackage{graphicx}
\usepackage{subcaption}

\usepackage{amsmath,amssymb,amsthm,mathtools}
\usepackage{bbm}

\usepackage{tabularx}
\usepackage{multirow}
\usepackage{siunitx}

\usepackage{thmtools}
\usepackage{thm-restate}

\declaretheorem[name=Lemma]{lemma}

\declaretheorem[name=Corollary]{corollary}
\newtheorem{Auxiliary Lemma}{Auxiliary Lemma}
\newtheorem{proposition}{Proposition}

\theoremstyle{definition}
\newtheorem{assumption}{Assumption}
\newtheorem{definition}{Definition}
\theoremstyle{remark}
\newtheorem{remark}{Remark}

\usepackage{hyperref} 
\usepackage{cleveref}

\newcommand{\R}{\mathbb{R}}
\newcommand{\E}{\mathcal{E}}
\newcommand{\V}{\mathcal{V}}
\newcommand{\Hcal}{\mathcal{H}}

\newcommand{\supp}{\operatorname{supp}}
\newcommand{\vol}{\operatorname{vol}}

\newcommand{\diag}{\operatorname{diag}}

\title{Thresholded Local Hyper-Flow Diffusion}

\author{
  Meher Chaitanya\thanks{Corresponding author.} \\
  KTH Royal Institute of Technology\\
  Stockholm, Sweden 114 28\\
  \texttt{mcpi@kth.se}\\
  \And
  Sebastian Dalleiger\\
  KTH Royal Institute of Technology\\
  Stockholm, Sweden 114 28\\
  \texttt{sdalle@kth.se}\\
  \AND
  Luana Ruiz\\
  Johns Hopkins University\\
  Baltimore, MD 21210\\
  \texttt{lrubini1@jh.edu}\\
}

\begin{document}

\maketitle
\begin{abstract}
Local Hyper-Flow Diffusion (HFD) \citep{fountoulakis2021local} gives an edge-size-independent Cheeger-type guarantee for seeded clustering in general submodular hypergraphs, but existing HFD solvers do not keep intermediate computation local at every iteration. We introduce \textsc{Thresholded Local HFD} (TL-HFD), a first-order method that maintains an active region around the seeds, performs projected subgradient updates on that region and its immediate boundary, and expands via thresholded (top-$k$) boundary activation. We prove that the local update is exact: the degree-preconditioned projected subgradient step restricted to the active region and its boundary coincides with the unrestricted global update. We establish finite-time dual suboptimality for both exact and thresholded updates, treating the latter as \emph{inexact} projected subgradient steps with explicit skipped-boundary error. We further derive an additive activated-volume bound controlled by realized local subgradient norms and the minimum boundary-push among newly activated vertices, and translate approximate dual optimality with localized support into a robust sweep-cut guarantee for early-stopped iterates. For general submodular cut-costs, each iteration is local in the scanned region and oracle-sensitive in the hyperedge primitive. Empirically, TL-HFD often matches or improves over HFD while activating less volume, with the largest gains on noisy instances where diffusion tends to absorb non-target vertices.
\end{abstract}

\section{Introduction} 
Many real-world systems involve higher-order interactions --- such as co-purchasing users, co-expressed genes, and jointly-consumed items --- for which hypergraphs provide a natural model by allowing each hyperedge to connect an arbitrary subset of vertices \citep{benson2016higher, zhou2006learning, agarwal2006higher}. A central problem is \emph{local clustering}: given a small seed set, identify a nearby low-conductance cluster without computing a global partition. While diffusion-based methods such as approximate Pagerank provide a canonical solution in graphs \citep{andersen2006local, chung2007heat, spielman2004nearly}, extending this paradigm to hypergraphs is more subtle because cutting a hyperedge is inherently ambiguous, with penalties ranging from cardinality-based costs to general submodular splitting functions \citep{lawler1973cutsets, li2018submodular, yoshida2019cheeger}. Local Hyper-Flow Diffusion (HFD) \citep{fountoulakis2021local} addresses this by formulating seeded clustering in general submodular hypergraphs as a convex primal--dual program with an edge-size-independent Cheeger-type guarantee. Its primal also admits an alternating-minimization solver that is strongly local with respect to the solution vector, making HFD a compelling foundation for local hypergraph clustering under arbitrary submodular splitting costs.

This leaves a natural algorithmic question within the HFD framework: can  diffusion be optimized by updates whose locality is explicit at every iteration, rather than only reflected in the sparsity of the final solution? HFD is typically optimized through a primal alternating-minimization procedure \citep{fountoulakis2021local}, which is effective in practice but does not guarantee that intermediate iterates remain local. Existing strongly-local hypergraph methods provide such locality under more restrictive splitting models, such as unit/cardinality cut-costs or bounded hyperedge size \citep{liu2021strongly, ibrahim2020local}. We therefore develop a complementary first-order method based on the HFD dual. Although the dual is non-smooth, its variables are precisely the diffusion heights used for sweep cuts, making it a natural object for local, early-stoppable optimization. Using this insight, we propose \textsc{Thresholded Local Hyper-Flow Diffusion} (TL-HFD), a locality-by-design projected subgradient method that maintains a dynamically growing active region around the seeds and touches only this region and its immediate boundary. Boundary growth is controlled through thresholded top-\(k\) activations. Our analysis establishes finite-time convergence, locality, and sweep-cut guarantees for TL-HFD. For general submodular cut-costs, the per-iteration complexity is local in the scanned region and oracle-sensitive in the per-hyperedge evaluation model. In summary, our contributions are:
\vspace{-0.05in}

\vspace{-0.05in}
\begin{enumerate}
  \item \textbf{A locality-by-design first-order optimizer for HFD.}
  We introduce TL-HFD, a first-order method for the non-smooth HFD dual that maintains a seed-anchored active region, updates only this region and its one-hop boundary, and expands through selective top-\(k\) boundary activation. Lemma \ref{lem:locality} proves that this local degree-preconditioned projected subgradient update is exact: restricting the update to the active region and its boundary gives the same iterate as the unrestricted global projected subgradient step, since all vertices outside this local region are pushed negative and remain zero after projection.

  \item \textbf{Finite-time optimization and sweep guarantees.}
  We prove finite-time dual suboptimality for TL-HFD. Theorem~\ref{thm:eps-optimality} gives the exact local PSGD rate in terms of restricted local subgradient norms, and Corollary~\ref{cor:dual_gap_threshold} extends this to the thresholded top-\(k\) variant by treating skipped boundary vertices as an inexact projected subgradient error. We further show that approximate dual iterates with localized support admit a robust sweep guarantee (Theorem~\ref{thm:robust-sweep}, Corollary~\ref{cor:early-stopping}).

  \item \textbf{Activated-volume accounting.}
  We derive an additive activated-volume bound showing that the total volume promoted into the active region is controlled by realized local subgradient norms and the minimum boundary-push among newly activated vertices (Theorem~\ref{thm:explored_volume_topk}). For general submodular cut-costs, the per-iteration complexity is local in the scanned region and oracle-sensitive in the underlying hyperedge primitive.

  \item \textbf{Empirical evaluation.} Empirically, TL-HFD activates smaller volumes than HFD while improving sweep quality and F1 on noisier instances in both real-world and synthetic hypergraphs (Appendix~\ref{app:sbm},~\ref{app:synth},~\ref{app:active_set_growth}).
\end{enumerate}

\textbf{Related Work.} We defer a detailed discussion of related work to Appendix~\ref{app:related-work}, where we survey hypergraph cut objectives, diffusion and Laplacian primitives, strongly-local clustering algorithms, and flow/convex formulations in context.

\vspace{-0.1in}
\section{Preliminaries}
\vspace{-0.1in}
\textbf{Submodular hypergraphs.}
A weighted hypergraph $\Hcal = (\V, \E, \{(w_e, \vartheta_e)\}_{e \in \E})$ has vertex set $\V$ and hyperedges $e \subseteq \V$ with weights $\vartheta_e > 0$. Unlike ordinary edges, a hyperedge admits many splits of its vertices, each incurring a different penalty; these are encoded by a \emph{cut-cost function} $w_e: 2^e \to \R_{\ge 0}$, assumed to be normalized ($w_e(\emptyset) = w_e(e) = 0$) and submodular. This framework subsumes unit, cardinality-based, and general submodular costs.

\textbf{Degrees, volume, and conductance.}
The vertex degree is $d_v := \sum_{e \ni v} \vartheta_e$, with $D := \diag(d)$. Throughout, we assume $d_v > 0$ for every $v \in \V$ (equivalently, $\V$ contains no isolated vertices); this is without loss of generality for seeded local clustering since isolated vertices are never reached by the diffusion. For $S \subseteq \V$, the \emph{volume} is $\vol(S) := \sum_{v \in S} d_v$ and the \emph{boundary volume} is $\vol(\partial S) := \sum_{e \in \E} \vartheta_e\, w_e(S \cap e)$, aggregating weighted penalties across the cut $(S, \V \setminus S)$. The \emph{conductance} is
$\Phi(S) := \frac{\vol(\partial S)}{\min\{\vol(S), \vol(\V \setminus S)\}}$.
Low conductance indicates a set that is internally well-connected yet weakly coupled to the rest of the hypergraph. Direct minimization of $\Phi$ is  intractable, motivating convex relaxations \citep{andersen2006local, fountoulakis2019variational, fountoulakis2020p, fountoulakis2021local, liu2021strongly}.

\textbf{Lov\'{a}sz extensions.}
To connect submodular cut-costs to convex optimization, \citet{fountoulakis2021local} use the Lov\'{a}sz extension of each $w_e$,
\[
  f_e(x) := \max_{\rho \in B_e} \langle \rho, x \rangle,
  \qquad
  B_e := \bigl\{\rho \in \R^{|e|} : \langle \rho, \mathbbm{1}_S \rangle \le w_e(S)\ \forall S \subseteq e,\ \langle \rho, \mathbbm{1}_e \rangle = 0\bigr\},
\]
with $\rho$ extended by zero outside $e$. Since $f_e(\mathbbm{1}_S) = w_e(S \cap e)$, the Lov\'{a}sz extension is the tightest convex extension of $w_e$ to indicator vectors and provides a principled relaxation of discrete cut costs to continuous optimization.

\textbf{Local hypergraph clustering:} Given a submodular hypergraph $\Hcal$ and a seed set $S \subset \V$, identify a set $\hat{C} \subset \V$ of low conductance $\Phi(\hat{C})$ that lies in the vicinity of $S$.

\textbf{HFD diffusion objective.}
Following \citet{fountoulakis2021local}, we adopt a convex relaxation that injects mass at the seeds and penalizes hyperedge splits via Lov\'{a}sz extensions. Given $\delta > 0$, define the seed injection $\Delta(v) := \delta d_v$ for $v \in S$ and $\Delta(v) := 0$ otherwise. The degree-weighted injection calibrates the seed signal to vertex connectivity. The quadratic term $\frac{\sigma}{2}x^\top D x$ makes $F$  $\sigma$-strongly convex with respect to $\|\cdot\|_D$. The negative of the HFD dual objective is
\begin{equation}\label{eq:cut_HFD}
  \min_{x \in \R^{|\V|}_+} F(x)
  := \tfrac{1}{2} \sum_{e \in \E} \vartheta_e f_e(x)^2
   + \tfrac{\sigma}{2} x^\top D x
   - \langle \Delta - d,\, x \rangle,
\end{equation}
where $\sigma > 0$ controls degree-weighted regularization. Smaller $\sigma$ lets the diffusion spread further; Equation~\eqref{eq:cut_HFD} is the negative of the dual in HFD's primal--dual flow formulation \citep{fountoulakis2021local}; we write $\mathcal{D}(x):=-F(x)$ when connecting to sweep guarantees that build on HFD's dual-based analysis. Our algorithm minimizes $F$ directly. The solution $x^\star \in \R_+^{|\V|}$ assigns nonnegative heights to vertices --- the term $-\langle \Delta-d,x\rangle$ rewards height on sufficiently injected seeds and penalizes height on non-seeds, while $\tfrac{1}{2} \sum_e \vartheta_e f_e(x)^2$ penalizes within-hyperedge discrepancies and acts as an $\ell_2$-type penalty on induced cut costs across level sets. We extract a discrete cluster by a standard \emph{sweep}: for thresholds $h$, form level sets $S_h := \{v \in \V : x^\star_v \ge h\}$ and return the one with smallest conductance.

\textbf{Notational convention.} Throughout the paper, $\partial$ denotes either a vertex-set boundary, as in $\partial S$ and $\vol(\partial S)$, or the subdifferential of a convex function, as in $\partial F(x)$; the meaning is clear from context. We also use $e$ and $e^t$ for edges and error vectors, respectively, with context disambiguating the notation. For iterates $\{x^{(t)}\}_{t=0}^{T-1}$, we write $\hat{x}\in \arg\min_{0\le t\le T-1} F(x^{(t)})$ to denote any best iterate, i.e., $\hat{x} = x^{(t^\star)}$ for some $t^\star \in \arg\min_{0 \le t \le T-1} F(x^{(t)})$ and use the analogous convention for $\arg\max$.

\section{Thresholded Local Hyper-Flow Diffusion (TL-HFD)} \label{sec:tlhfd}
Having formulated local hypergraph clustering, we now present
Thresholded Local Hyper-Flow Diffusion (TL-HFD), a locality-by-design solver for the
non-smooth objective~ \eqref{eq:cut_HFD}.

\textbf{Locality.}
For a vector $x \in \mathbb{R}^{|\V|}_+$, define its support
$\operatorname{supp}(x) = \{v \in \V : x_v > 0\}$ and the
\emph{seed-anchored active region} $A(x) := \operatorname{supp}(x) \cup S$.
The \emph{boundary} of $A(x)$ is $\partial A(x) := \bigl\{u \notin A(x) : \exists\, e \in \E,\; u \in e \text{ and } e \cap A(x) \neq \emptyset\bigr\}$.
The active region $A(x)$ contains all vertices currently
participating in the diffusion; its boundary $\partial A(x)$
consists of vertices reachable in one hyperedge hop that have not yet been activated.

\textbf{Local-set terminology.}
We distinguish three co-evolving vertex sets. The \emph{active region}
$A^{(t)}=\supp(x^{(t)})\cup S$ contains the vertices currently carrying mass.
The \emph{local computational region} $L_t:=A^{(t)}\cup\partial A^{(t)}$
is the set touched at iteration~$t$, including boundary vertices that are
scored but not necessarily activated. Let $J_t\subseteq\partial A^{(t)}$
be the vertices first promoted into $A^{(t+1)}$ at iteration~$t$; the
\emph{activated set} $M_t:=S\cup\bigcup_{s<t}J_s$ records every vertex ever
activated. Top-$k$ governs how quickly $M_t$ grows, but does not reduce the cost
of scoring the boundary, which remains proportional to $|\partial A^{(t)}|$
per iteration.

\textbf{Nonsmoothness and subgradients.}
Each $f_e(x) = \max_{\rho \in B_e} \langle \rho, x \rangle$ from equation \eqref{eq:cut_HFD} is nondifferentiable in general; $F$ is nonsmooth.
By Danskin's theorem~\citep{danskin1967directional}, for each hyperedge $e$
any maximizer $\rho_e(x) \in M_e(x) := \arg\max_{\rho \in B_e}
\langle \rho, x \rangle$ yields a subgradient of~$f_e$ at~$x$
(Appendix~~\ref{app:dirder}).
Selecting one such maximizer per hyperedge, define
\begin{equation}\label{eq:subgrad}
  g(x) := \sum_{e \in \E} \vartheta_e\, f_e(x)\, \rho_e(x)
         + \sigma D x - (\Delta - d).
\end{equation}

\begin{restatable}{lemma}{SubGrad}
\label{lem:subgrad_in_F}
Consider $F$ defined in \eqref{eq:cut_HFD}. For any $x$ and any choice $\rho_e(x)\in M_e(x)$, the vector $g(x)$ in \eqref{eq:subgrad} satisfies $g(x)\in \partial F(x)$.
\end{restatable}
\begin{proof}
    See Appendix ~\ref{app:subgrad}.
\end{proof}

\textbf{Degree-preconditioned PSGD.}
To solve equation~\eqref{eq:cut_HFD}, a standard projected subgradient step $x^{(t+1)} = \Pi_{\mathbb{R}^{|\V|}_+}(x^{(t)} - \eta_{t+1}\, g^{(t)})$ treats all vertices equally, but vertex degrees in hypergraphs can vary by orders of magnitude. To normalize the update across heterogeneous degrees, we apply the diagonal preconditioner~$D^{-1}$:
    $x^{(t+1)} = \Pi_{\mathbb{R}^{|\V|}_+}\!\big(x^{(t)} - \eta_{t+1}\, D^{-1} g(x^{(t)})\big)$.\refstepcounter{equation}\hfill(\theequation)\label{eq:psgd}

\textbf{Restricting to the active region.}
Degree preconditioning addresses scaling, but does not by itself
avoid work on distant, irrelevant vertices for local clustering around the seed set. We maintain an explicit
active region $A^{(t)} := A(x^{(t)})$ and its boundary
$\partial A^{(t)} := \partial A(x^{(t)})$, and restrict all
computation to $A^{(t)} \cup \partial A^{(t)}$. The following
lemma shows that this restriction produces \emph{exactly} the same
iterates as the global update~\eqref{eq:psgd}.

\begin{restatable}[Locality of iterates]{lemma}{Locality}
\label{lem:locality}
Let $x\ge 0$ and $A:=\supp(x)\cup S$. For any subgradient selection
$g(x)$ of the form~\eqref{eq:subgrad}, every vertex
$u\notin A\cup\partial A$ satisfies $[g(x)]_u=d_u-\Delta(u)=d_u>0$.
Consequently, for every $\eta>0$, $\left[\Pi_{\R^{|\V|}_+}\!\left(x-\eta D^{-1}g(x)\right)\right]_u=0$.
Thus, the global projected subgradient update can
be implemented exactly by updating only $A^{(t)}\cup\partial A^{(t)}$ at iteration $t$.
\end{restatable}

\begin{proof}
Full proof in Appendix~\ref{app:locality}.
\end{proof}

\textbf{Interpretation.}
Lemma~\ref{lem:locality} shows that locality is exact rather than heuristic:
outside the active region and its one-hop boundary, the subgradient coordinate
is positive, so the degree-preconditioned step moves the coordinate negative
and projection returns it to zero. Hence the global update~\eqref{eq:psgd}
can be implemented by the local update
\begin{equation}\label{eq:local-update}
x^{(t+1)}_v =
\begin{cases}
\left[x^{(t)}_v-\eta_{t+1}g_v(x^{(t)})/d_v\right]_+,
& v\in A^{(t)}\cup\partial A^{(t)},\\
0, & \text{otherwise}.
\end{cases}
\end{equation}
Each coordinate \(g_v(x^{(t)})\) depends only on hyperedges incident to \(v\).
Thus each iteration is local in the scanned region \(L_t=A^{(t)}\cup\partial A^{(t)}\)
and its incident hyperedges, rather than requiring an explicit pass over all of
\(\V\). We now pair the locality result with a standard finite-time analysis of projected subgradient descent, whose per-iteration cost scales with the scanned local region rather than explicitly with the full hypergraph.

\begin{restatable}[{$\varepsilon$-optimality of local PSGD}]{theorem}{optimality}
\label{thm:eps-optimality}
Assume \(d_v>0\) for all \(v\in \V\) and \(\sigma>0\). Let $x^\star\in\arg\min_{x\ge0}F(x)$. Run the local PSGD update $x^{(t+1)} = \Pi_{\mathbb R_+^{|\V|}} \left(  x^{(t)}-\eta_{t+1}D^{-1}g^{(t)} \right), \; \eta_{t+1}=\frac{1}{\sigma(t+1)}$, where \(g^{(t)}\in\partial F(x^{(t)})\) is selected as in equation~\eqref{eq:subgrad}.
Let  $L_t:=A^{(t)}\cup\partial A^{(t)}$ and define the restricted norm $\|g^{(t)}\|_{D^{-1},L_t}^2:= \sum_{v\in L_t}\frac{(g_v^{(t)})^2}{d_v}$. Let $\hat x^{(T)} \in \arg\min_{0\le t\le T-1}F(x^{(t)})$. Then $F(\hat x^{(T)})-F(x^\star) \le \frac{1}{2\sigma T} \sum_{t=0}^{T-1} \frac{\|g^{(t)}\|_{D^{-1},L_t}^2}{t+1}$. In particular, if $G_{\mathrm{loc}}^2(T) := \max_{0\le t<T} \|g^{(t)}\|_{D^{-1},L_t}^2$, then $F(\hat x^{(T)})-F(x^\star) \le \frac{G_{\mathrm{loc}}^2(T)(1+\log T)}{2\sigma T}.$ 
\end{restatable}

\begin{proof}
  Full proof in Appendix~\ref{app:optimality}.
\end{proof}

\begin{remark}
  The $(1+\log T)/T$ rate is standard for projected subgradient
  descent on $\sigma$-strongly convex nonsmooth
  objectives~\citep{shamir2013stochastic}; the contribution here is that (i) it applies to the HFD dual under degree-preconditioning and the specific subgradient selection of Lemma~\ref{lem:subgrad_in_F}, (ii) 
  by Lemma~\ref{lem:locality}, each iteration is algorithmically local and touches only the local region $A^{(t)} \cup \partial A^{(t)}$, so the bound is achieved using only the scanned local region and its incident hyperedges, rather than requiring an explicit pass over the full hypergraph.
\end{remark}

\begin{remark}[Localized norm bound]
Theorem~\ref{thm:eps-optimality} is stated in terms of the realized restricted
local norms \(\|g^{(t)}\|_{D^{-1},L_t}\). Under the box-constraints of Proposition~\ref{prop:Gloc-bound}, $\|g^{(t)}\|_{D^{-1},L_t}^2 \le B^2\vol(L_t), \; B=(\delta-1)(1/\sigma+2).$
Consequently, $G_{\mathrm{loc}}^2(T)\le B^2\max_{0\le t<T}\vol(L_t)$,
and Theorem~\ref{thm:eps-optimality} yields $F(\hat x^{(T)})-F(x^\star) \le \frac{B^2\max_{0\le t<T}\vol(L_t)(1+\log T)}{2\sigma T}.$ Thus the convergence bound depends on the explored local regions, and not on \(\vol(\V)\). The box projection is used only to obtain a uniform coordinatewise local subgradient bound; by
Appendix~\ref{app:box-constraints}, it does not change the optimizer.
\end{remark}

\textbf{Activation bottleneck and per-iteration work.}
The local update~\eqref{eq:local-update} evaluates only hyperedges that intersect 
$L_t=A^{(t)}\cup\partial A^{(t)}$, so the per-iteration work is local in the
scanned region, \(O\bigl(\sum_{e:\,e\cap L_t\neq\emptyset} T_e\bigr)\), where $T_e$ is the
per-hyperedge cost (oracle-sensitive for general submodular cut-costs,
depending on how $f_e$ and $\rho_e(x^{(t)})$ are computed). Each iteration still
inspects the current boundary \(\partial A^{(t)}\). Boundary vertices with
positive push $\kappa^{(t)}(u):=\max\{0,-g_u^{(t)}/d_u\}$ are candidates for promotion into the active region. Promoting all such vertices can rapidly enlarge $\vol(L_t)$ and the support of the
returned iterate. We therefore introduce thresholded top-\(k\) boundary
activation. This controls the number of vertices promoted into $A^{(t+1)}$ but not the number
inspected in $\partial A^{(t)}$. Accordingly, our locality claims have two parts: each iteration is local in the scanned region \(L_t\), while the cumulative growth of the promoted set \(M_t\) is controlled separately by the additive activated-volume accounting bound in Theorem~\ref{thm:explored_volume_topk}.

\textbf{Thresholding the boundary.}
As noted above, the local update~\eqref{eq:local-update} must
still inspect every boundary vertex in $\partial A^{(t)}$.
Boundary vertices have a special structure: every
boundary vertex satisfies $x^{(t)}_u = 0$ by construction, so its
preconditioned projected update simplifies to a one-sided
``push'': $x^{(t+1)}_u = \eta_{t+1}\,\kappa^{(t)}(u)$,\;
where $\kappa^{(t)}(u) := \max\{0,\,-g_u^{(t)}/d_u\}$. The quantity $\kappa^{(t)}(u)$ measures the magnitude of the boundary push: larger values indicate a stronger first-order signal for increasing \(x_u\).  Algorithm~\ref{alg:pq-psgd-theory} activates only the
top-$k$ boundary vertices by a score, setting
$x_u^{(t+1)} = 0$ for the rest.\\

\textbf{Committed boundary selection.}\label{committed}
At each iteration of Algorithm \ref{alg:pq-psgd-theory}, TL-HFD activates the top-$k$ boundary vertices ranked by the \emph{score}, $s^{(t)}(u) =\kappa^{(t)}(u) \left(\frac{d_{\mathrm{in}}^{(t)}(u)}{d_u}\right)^{\!\gamma}$, where $d_{\mathrm{in}}^{(t)}(u)$ is the weighted indegree measuring \emph{structural commitment} to the active set, and $\gamma$ controls the balance between gradient signal and structural membership. For unit cut-cost, $d_{\mathrm{in}}^{(t)}(u)$ counts edges incident to the active set; for cardinality and submodularity cut-costs, $d^{(t)}_{in}(u) = \sum_{e \ni u, e \cap A^{(t)} \neq \emptyset} \vartheta_e w_e(A^{(t)} \cap e)$ weights each edge by its overlap with the active set. Thresholding by selecting Top-k boundary vertices (ties are broken arbitrarily) for updates at each iteration introduces a truncation error supported only on skipped positively pushed boundary vertices, which we quantify in Lemma \ref{lem:threshold-inexact}.

\begin{restatable}[{Thresholding as a truncated inexact step}]{lemma}{thresholding}
\label{lem:threshold-inexact}
For each boundary vertex \(u\in\partial A^{(t)}\), define
$\kappa^{(t)}(u) := \max\left\{0,-\frac{g_u^{(t)}}{d_u}\right\}$. Let \(K_t\subseteq \partial A^{(t)}\) be any subset of boundary vertices selected
by Algorithm~\ref{alg:pq-psgd-theory} for activation at iteration \(t\). Let \(x^{(t+1)}\) be the resulting thresholded iterate, and define the skipped
positively pushed boundary set $ B^{(t)}  := \left\{ u\in \partial A^{(t)}\setminus K_t: \kappa^{(t)}(u)>0 \right\}$. Define \(e^{(t)}\in\mathbb R^{|\V|}\) by $e_u^{(t)} := \eta_{t+1}\,\kappa^{(t)}(u)$ for $u \in B^{(t)}$ and $e_u^{(t)} := 0$ otherwise. Then $x^{(t+1)} =  \Pi_{\mathbb R_+^{|\V|}} \left( x^{(t)}-\eta_{t+1}D^{-1}g^{(t)}-e^{(t)}\right)$. Moreover, $\|e^{(t)}\|_D^2  = \eta_{t+1}^2 \sum_{u\in B^{(t)}}d_u\bigl(\kappa^{(t)}(u)\bigr)^2$, and $\|e^{(t)}\|_{1,D} :=  \sum_{u\in V}d_u|e_u^{(t)}|  =  \eta_{t+1} \sum_{u\in B^{(t)}}d_u\kappa^{(t)}(u)$. If, in addition, the selected subgradient satisfies the coordinatewise local
bound $\left|\frac{g_u^{(t)}}{d_u}\right| \le B_{\rm loc} \; \forall u\in A^{(t)}\cup\partial A^{(t)}$, then $\|e^{(t)}\|_D^2 \le \eta_{t+1}^2 B_{\rm loc}^2\vol(B^{(t)})$ and $\|e^{(t)}\|_{1,D} \le \eta_{t+1}B_{\rm loc}\vol(B^{(t)})$ (under the hypotheses of Proposition~\ref{prop:Gloc-bound}, namely for the
box-constrained variant with $0\le x_v^{(t)}\le \bar x:=\frac{\delta-1}{\sigma} \; \forall v$, one may take $B_{\rm loc} = (\delta-1)(1/\sigma+2)$).
\end{restatable}

\begin{proof}
    Full proof in Appendix \ref{app:thresholding_error}.
\end{proof}
\noindent


\textbf{Interpretation.}
Lemma~\ref{lem:threshold-inexact} interprets thresholding as an inexact
preconditioned PSGD step: its only error is the truncation vector
$e^{(t)}$, supported on skipped positively pushed boundary vertices and
controlled by $\sum_{u\in B^{(t)}} d_u\kappa^{(t)}(u)$, or under
Proposition~\ref{prop:Gloc-bound}, by $\vol(B^{(t)})$ and
$B_{\mathrm{loc}}$. Thus boundary selection enters the convergence
analysis as an explicit optimization penalty
(Corollary~\ref{cor:dual_gap_threshold}). The bound is
selection-rule agnostic: for any selected set, the penalty is determined
by the skipped set $B^{(t)}$; for a fixed activation budget, it is
minimized by retaining vertices with largest $d_u\kappa^{(t)}(u)$. Our implementation ranks by $s^{(t)}(u)$, which weights $\kappa^{(t)}(u)$ by a structural factor ($\gamma$) favoring vertices well attached to the active region (Algorithm \ref{alg:pq-psgd-theory}); $\gamma>0$ controls this bias (App.~\ref{app:params}).

\begin{algorithm}[H]
\caption{Thresholded Local Hyper-Flow Diffusion (TL-HFD)}
\label{alg:pq-psgd-theory}
\begin{algorithmic}[1]
\Require Current iterate $x^{(t)}$, seed set $S$,
         step size $\eta_{t+1}$, activation count $k \ge 0$
\Ensure Next iterate $x^{(t+1)}$

\Statex \textit{\# Active region and its one-hop boundary.}
\State $A^{(t)} \gets \operatorname{supp}(x^{(t)}) \cup S$
\State $\partial A^{(t)} \gets
       \{u \notin A^{(t)} : \exists\, e \in E,\;
       u \in e,\; e \cap A^{(t)} \neq \emptyset\}$

\Statex \textit{\# Compute subgradient only on the local region (Lemma~\ref{lem:locality}).}
\State $g^{(t)} \gets g(x^{(t)})$
       restricted to $A^{(t)} \cup \partial A^{(t)}$
\State $x^{(t+1)} \gets x^{(t)}$

\Statex \textit{\# Exact PSGD update on active vertices.}
\ForAll{$u \in A^{(t)}$}
  \State $x^{(t+1)}_u \gets
         \max\bigl\{0,\;
         x^{(t)}_u - \eta_{t+1}\,[g^{(t)}]_u / d_u
         \bigr\}$
\EndFor

\Statex \textit{\# Scoring for boundary vertices.}
\ForAll{$u \in \partial A^{(t)}$}
  \State $\kappa^{(t)}(u) \gets \max\bigl\{0,\; -[g^{(t)}]_u / d_u\bigr\}$
  \State $c^{(t)}(u) \gets \bigl(d_{\mathrm{in}}^{(t)}(u)/d_u)\bigr)^\gamma$
  \State $s^{(t)}(u) \gets \kappa^{(t)}(u)\, c^{(t)}(u)$
\EndFor

\Statex \textit{\# Activate the top-$k$ (Thresholding).}
\State $K_t \gets \textsc{TopK}(s^{(t)},\, k)$
\ForAll{$u \in K_t$}
  \State $x_u^{(t+1)} \gets \eta_{t+1}\,\kappa^{(t)}(u)$
\EndFor

\State \Return $x^{(t+1)}$
\end{algorithmic}
\end{algorithm}

\begin{restatable}[{Suboptimality under boundary truncation}]{corollary}{thresholdinggap}
\label{cor:dual_gap_threshold}
Assume the hypotheses of Theorem~\ref{thm:eps-optimality}. Run
Algorithm~\ref{alg:pq-psgd-theory} for \(T\) iterations with step sizes
$\eta_{t+1}=1/(\sigma(t+1))$. Let $\hat x^{(T)} \in\arg\min_{0\le t\le T-1}F(x^{(t)})$ be a best iterate, and let
$\bar x\ge \|x^\star\|_\infty$. For each \(t\), let \(B^{(t)}\) be the skipped
positively pushed boundary set from Lemma~\ref{lem:threshold-inexact}. Define
$\varepsilon_{\rm loc}(T):=\frac{1}{2\sigma T}\sum_{t=0}^{T-1}
\frac{\|g^{(t)}\|_{D^{-1},L_t}^2}{t+1}$. Then
$F(\hat x^{(T)})-F(x^\star)\le \varepsilon_{\rm loc}(T)
+\frac{\bar x}{T}\sum_{t=0}^{T-1}\sum_{u\in B^{(t)}}d_u\kappa^{(t)}(u)$.
Moreover, if
$\left|g_u^{(t)}/d_u\right|\le B_{\rm loc}$ for all
$u\in A^{(t)}\cup\partial A^{(t)}$ and all \(t<T\), then
$F(\hat x^{(T)})-F(x^\star)\le \varepsilon_{\rm loc}(T)
+\frac{\bar x B_{\rm loc}}{T}\sum_{t=0}^{T-1}\vol(B^{(t)})
=:\varepsilon_{\rm vol}(T)$.
\end{restatable}

\begin{proof}
  Proof in Appendix \ref{app:dual_gap_threshold}.
\end{proof}

\begin{remark}[Explicit constants] The explicit constants from Proposition~\ref{prop:Gloc-bound} are formal for the box-constrained Algorithm~\ref{alg:app-pq-psgd-theory} and transfer to Algorithm~\ref{alg:pq-psgd-theory} whenever the box is non-binding (remark~\ref{rem:box-nonbinding}).
If $G_{\rm loc}^2(T):=\max_{0\le t<T}\|g^{(t)}\|_{D^{-1},L_t}^2$, then
Theorem~\ref{thm:eps-optimality} gives
$\varepsilon_{\rm loc}(T)\le
G_{\rm loc}^2(T)(1+\log T)/(2\sigma T)$. Under the hypotheses of
Proposition~\ref{prop:Gloc-bound} for the box-constrained variant, one may
take $\bar x=(\delta-1)/\sigma$ and
$B_{\rm loc}=(\delta-1)(1/\sigma+2)$. In that case,
$\varepsilon_{\rm vol}(T)=\varepsilon_{\rm loc}(T)
+\frac{(\delta-1)B_{\rm loc}}{\sigma T}
\sum_{t=0}^{T-1}\vol(B^{(t)})$.
\end{remark}

\vspace{-0.1in}
\subsection{Sweep-Cut Guarantees}
The convergence bounds above control the suboptimality gap, 
$F(\hat{x}^{(T)}) - F(x^\star)$. To translate this into a
clustering guarantee --- i.e., a bound on the conductance of
the sweep set extracted from $\hat{x}^{(T)}$ --- we require
structural assumptions relating the seed set to the target
cluster, following~\citet{fountoulakis2021local}.

\begin{assumption}[Seed overlap]\label{ass:overlap}
  There exists a target cluster $C \subset \V$ such that
  $\operatorname{vol}(S \cap C) \ge \alpha\,\operatorname{vol}(C)$
  and
  $\operatorname{vol}(S \cap C) \ge \beta\,\operatorname{vol}(S)$
  for some $\alpha, \beta \in (0,1]$. \footnote{The experiments in 
Section~\ref{sec:experiments} use seed injections of the form
$\Delta[s^\star]=\delta_{\exp}\cdot\vol(\mathcal T)$ with $s^\star\in C$ and $\mathcal T=C$.
This corresponds to the degree-weighted form $\Delta(v)=\delta d_v$ for
$v\in S=\{s^\star\}$ with a seed-specific
$\delta=\delta_{\exp}\cdot\vol(\mathcal T)/d_{s^\star}$; see
Proposition~\ref{prop:theory-experiment-bridge} in
Appendix~\ref{sec:theory-exp-bridge} for further details.}
\end{assumption}
\vspace{-0.1in}


\begin{assumption}[Target volume]\label{ass:volume}
  $\operatorname{vol}(C) \le \operatorname{vol}(\V \setminus C)$. This is the standard assumption that the target cluster is the
minority side of the cut, so that $\Phi(C) = \operatorname{vol}(\partial C) / \operatorname{vol}(C)$.
\end{assumption}

\begin{assumption}\label{ass:parameter}
Following \citet{fountoulakis2021local}, we set \(\delta=3/\alpha\), assume \(0\le w_e(A)\le1\) for all
\(e\in \E\), \(A\subseteq e\), and require $\sigma\le \beta\Phi(C)/3$.
\end{assumption}

The first condition requires that the seeds have sufficient coverage of the target cluster; the second requires that the seeds are sufficiently concentrated inside~$C$. Throughout, we set the injection coefficient $\delta = 3/\alpha$, so that the seed mass is inversely proportional to the
seed-cluster overlap; this choice is inherited
from~\citet{fountoulakis2021local}. 

\begin{definition}[\(\mathbb{K}\)-sweep locality]
\label{def:K-sweep-locality}
Fix a target cluster \(C\) and a vector \(x\ge0\). Define $\nu(x):=\sum_{e\in E}\vartheta_e f_e(x)^2$, and $\ell(e):= \begin{cases} f_e(x)/\sqrt{\nu(x)}, & \nu(x)>0 \\ 0, & \nu(x)=0 \end{cases}$.  Let $\ell_b(e):= \max\left\{ \frac{1}{\sqrt{\vol(C)}}, \ell(e) \right\}$. We say that \(x\) satisfies \(\mathbb{K}\)-sweep locality if $\sum_{e\in \E}\vartheta_e f_e(x)\ell_b(e) \le \mathbb{K}\sqrt{\nu(x)}$.
\end{definition}
The choice \(\mathbb K=4/\beta\) recovers the HFD sweep-locality condition $\sum_{e\in E} \vartheta_e f_e(x)\,\ell_b(e) \le \frac{4}{\beta}\sqrt{\nu(x)}$ ~\citep[Claim~A.2]{fountoulakis2021local}. 

\begin{restatable}[{Sweep bound under \(\mathbb{K}\)-sweep locality}]{theorem}{Robust}
\label{thm:robust-sweep}
Assume Assumptions \ref{ass:overlap}, \ref{ass:volume}, \ref{ass:parameter} and let $x^\star \in \arg\min_{x\ge 0} F(x)$. Let \(x\ge 0\) satisfy the suboptimality $F(x)-F(x^\star)\le \varepsilon$. Assume that \(x\) satisfies \(\mathbb{K}\)-sweep locality, $\sum_{e\in \E}\vartheta_e f_e(x)\ell_b(e) \le \mathbb{K}\sqrt{\nu(x)}.$ Define $D_C:=\frac{\vol(C)^2}{2\vol(\partial C)}$, and assume \(\vol(\partial C)>0\). If $\varepsilon<D_C$, then there exists a threshold \(h>0\) such that the sweep set $S_h:=\{v\in \V:x_v\ge h\}$ satisfies $\frac{\vol(\partial S_h)}{\vol(S_h)} \le \frac{3\mathbb{K}}{\alpha}\sqrt{\Phi(C)} \cdot  \frac{1}{\sqrt{1-\varepsilon/D_C}}$. If additionally $\vol(S_h)\le \vol(\V\setminus S_h)$, then the same bound holds for conductance: $\Phi(S_h) \le \frac{3\mathbb{K}}{\alpha}\sqrt{\Phi(C)}  \cdot \frac{1}{\sqrt{1-\varepsilon/D_C}}$.
\end{restatable}

\begin{proof}
Full proof in Appendix ~\ref{app:robust-sweep}. 
\end{proof}

Theorem~\ref{thm:robust-sweep} applies to any approximate dual iterate satisfying \(\mathbb K\)-sweep locality. To obtain an early-stopping guarantee
for TL-HFD, we prove a support-volume sufficient condition in Appendix~\ref{app:early_stopping}: for any \(x\ge 0\), \(x\) satisfies \(\mathbb K\)-sweep locality with $\mathbb K = 1+\sqrt{\frac{\vol(\supp(x))}{\vol(C)}}$. Therefore, localized support of the returned iterate directly yields a sweep-locality constant. 

\begin{restatable}[{Early-stopping sweep guarantee for TL-HFD}]{corollary}{earlystopping}
\label{cor:early-stopping}
Assume Assumptions~\ref{ass:overlap}, \ref{ass:volume}, and
\ref{ass:parameter}. Assume also the hypotheses of
Corollary~\ref{cor:dual_gap_threshold}, using the box-constrained
Proposition~\ref{prop:Gloc-bound} so that
\(\varepsilon_{\rm vol}(T)\) is well-defined. Let
\(\hat x^{(T)}\in\arg\min_{0\le t<T}F(x^{(t)})\), and define
\(\zeta_T:=\vol(\supp(\hat x^{(T)}))/\vol(C)\). If
\(\varepsilon_{\rm vol}(T)<D_C :=\frac{\vol(C)^2}{2\vol(\partial C)}\), then there exists a threshold \(h>0\) such
that \(S_h:=\{v:\hat x^{(T)}_v\ge h\}\) satisfies $\frac{\vol(\partial S_h)}{\vol(S_h)}
\le \frac{3}{\alpha}(1+\sqrt{\zeta_T})\sqrt{\Phi(C)}  \cdot  \frac{1}{\sqrt{1-\varepsilon_{\rm vol}(T)/D_C}}$. If additionally \(\vol(S_h)\le\vol(\V\setminus S_h)\), then the same bound holds for \(\Phi(S_h)\).
\end{restatable}

\begin{proof}
Full proof in Appendix~\ref{app:early_stopping}.
\end{proof}

\begin{remark}[Recovery of the HFD sweep-locality constant]
If the returned iterate satisfies $\vol(\supp(\hat x^{(T)})) \le \left(\frac4\beta-1\right)^2\vol(C)$,
then $1+\sqrt{\zeta_T} \le  \frac4\beta$.
In this case Corollary~\ref{cor:early-stopping} recovers the same sweep
constant as the HFD sweep-locality condition: $\frac{3}{\alpha}\left(1+\sqrt{\zeta_T}\right) \le  \frac{12}{\alpha\beta}$.
\end{remark}

\textbf{Bound on cumulative activated volume.} The bounds above are local in the scanned region
\(L_t=A^{(t)}\cup\partial A^{(t)}\), while the sweep guarantee depends on the
support volume of the returned iterate. Top-\(k\) activation limits the number of promoted boundary vertices per iteration, but it does not by itself bound their degree volume or the cost of scoring the current boundary.  We therefore track the set (\(M_t\)) of
vertices ever promoted into the active region. The
following theorem gives an additive trajectory-level bound on the activated
volume \(\vol(M_T)\) in terms of restricted local subgradient energy and the
realized boundary-push margins of newly activated vertices.

\begin{restatable}[{Explored-volume}]{theorem}{exploredVolume}
\label{thm:explored_volume_topk}
Assume \(x^{(0)}=0\). Let $M_0 := S$ and let $J_t:=\{u\in \partial A^{(t)} : u\notin M_t,\ x_u^{(t+1)}>0\}$ be the set of vertices activated for the first time at iteration \(t\), where
$M_{t+1}:=M_t\cup J_t.$  Equivalently, for boundary updates in Algorithm~\ref{alg:pq-psgd-theory},
$J_t = \{u\in K_t : u\notin M_t,\ \kappa^{(t)}(u)>0\}.$
Let $L_t:=A^{(t)}\cup\partial A^{(t)}$
and $\kappa^{(t)}(u) := \max\left\{0,-\frac{g_u^{(t)}}{d_u}\right\}.$ Let \(I_T:=\{0\le t<T:J_t\neq\emptyset\}\). Assume that for each
\(t\in I_T\), there exists \(\rho_t>0\) such that $\kappa^{(t)}(u)\ge \rho_t \; \forall u\in J_t$. Then, for $M_T=S\cup\bigcup_{t=0}^{T-1}J_t$, we have $\vol(M_T) \le \vol(S) + \sum_{t\in I_T} \frac{\|g^{(t)}\|_{D^{-1},L_t}^2}{\rho_t^2}$. In particular, if $\|g^{(t)}\|_{D^{-1},L_t}\le G_{\mathrm{loc}} \; \forall t$,  then $\vol(M_T) \le \vol(S) + G_{\mathrm{loc}}^2 \sum_{t\in I_T}\rho_t^{-2}$.
\end{restatable}

\begin{proof}
  Full proof in Appendix~\ref{strong_local_app}. 
\end{proof}

\begin{remark}[Activated-volume accounting]
Theorem~\ref{thm:explored_volume_topk} gives a trajectory-level accounting
bound for the growth of the activated set. Along the TL-HFD trajectory, the
volume promoted into \(M_T\) is controlled by two realized local quantities:
the restricted subgradient energy \(\|g^{(t)}\|_{D^{-1},L_t}^2\) on the scanned
region and the activation margin \(\rho_t\) of newly promoted boundary
vertices. This separates two notions of locality: each iteration computes only
on \(L_t=A^{(t)}\cup\partial A^{(t)}\) and its incident hyperedges, while
Theorem~\ref{thm:explored_volume_topk} tracks the cumulative growth of the
state that is actually promoted. Since every nonzero coordinate of the returned
best iterate is contained in \(M_T\), this activated-volume bound provides the
support-volume control used in the early-stopping sweep analysis.
\end{remark}

\textbf{From activated volume to sweep locality.}
\label{rem:end-to-end-local-sweep}
Theorem~\ref{thm:explored_volume_topk} controls the activated set \(M_T\).
To use the support-volume sweep guarantee, we need to control the support
volume of the returned best iterate \(\hat x^{(T)}\). Since
\(\hat x^{(T)}\) is one of the iterates generated before time \(T\), every
nonzero coordinate of \(\hat x^{(T)}\) is either a seed or a vertex activated
at some earlier iteration. Hence $\supp(\hat x^{(T)})\subseteq M_T$. Therefore, under the hypotheses of
Theorem~\ref{thm:explored_volume_topk}, $\vol(\supp(\hat x^{(T)})) \le \vol(M_T) \le \vol(S) + \sum_{t \in I_T} \frac{\|g^{(t)}\|_{D^{-1},L_t}^2}{\rho_t^2}$, where \(\rho_t\) lower-bounds the positive boundary push of vertices first activated at iteration \(t\). Combining this with the seed-concentration assumption \(\vol(S\cap C)\ge \beta\vol(S)\) and the inclusion \(S\cap C\subseteq C\), we have \(\vol(S)\le \vol(C)/\beta\). Thus $\frac{\vol(\supp(\hat x^{(T)}))}{\vol(C)} \le \frac{1}{\beta} + \frac{1}{\vol(C)} \sum_{t \in I_T} \frac{\|g^{(t)}\|_{D^{-1},L_t}^2}{\rho_t^2}$. This gives an explicit support-volume bound that can be substituted into the early-stopping sweep guarantee through \(\zeta_T=\vol(\supp(\hat x^{(T)}))/\vol(C)\).

\textbf{Summary of guarantees.} The results above give an end-to-end local first-order interpretation of
TL-HFD. Lemma~\ref{lem:locality} shows that each projected subgradient step can
be implemented exactly on the scanned local region
\(L_t=A^{(t)}\cup\partial A^{(t)}\). Theorem~\ref{thm:eps-optimality} and
Corollary~\ref{cor:dual_gap_threshold} control the dual suboptimality of exact
and thresholded updates. Theorem~\ref{thm:robust-sweep} and
Corollary~\ref{cor:early-stopping} translate approximate dual optimality and
localized support into a sweep-cut guarantee for early-stopped iterates. Finally,
Theorem~\ref{thm:explored_volume_topk} accounts for the cumulative volume
promoted into the activated set \(M_T\), which provides an explicit way to
control the support ratio \(\zeta_T\) appearing in the sweep bound. We next
evaluate the performance of TL-HFD in practice.

\vspace{-0.1in}
\section{Experiments}\label{sec:experiments}
\vspace{-0.1in}
We evaluate TL-HFD (Algorithm~\ref{alg:pq-psgd-theory} with top-\(k\) boundary scoring) on real-world and synthetic hypergraph datasets (Appendix~\ref{app:sbm}, ~\ref{app:synth}), comparing unit (U), cardinality (C), and submodular (S) cut-costs (Appendix~\ref{app:data_cut}). On synthetic hypergraphs with varying target conductance, TL-HFD matches HFD on clean, well-separated clusters and improves over HFD on noisier clusters where unrestricted expansion tends to include non-target vertices (Appendix~\ref{app:sbm}). On real-world hypergraphs, TL-HFD gives the largest F1 gains on clusters where the baselines are weaker. Additional experiments on Amazon-reviews, Florida Bay, contact-high-school, oil trade, and synthetic HSBM are provided in Appendix~\ref{app:exp-details}. The sensitivity analysis of $\gamma$ from Algorithm \ref{alg:pq-psgd-theory} is in Appendix~\ref{app:gamma-ablation}.

\textbf{Activation-scale selection.}
Following \citep{fountoulakis2021local}, we use target-volume-scaled seed
injection: the injected mass is proportional to \(\vol(\mathcal T)\), so the effective
single-seed injection coefficient varies with the target cluster and seed
degree. TL-HFD introduces one additional activation-scale parameter, the
top-\(k\) fraction \(f\). Rather than tuning the integer \(k\) directly, we set
\(k=\max(1,\lfloor f\cdot \vol_{\rm est}\rceil)\) for unit cut-costs and
\(k=\max(1,\lfloor f\cdot n_{\rm est}\rceil)\) for cardinality cut-costs,
where \(\vol_{\rm est}=\max_v\Delta(v)/\delta_{\exp}\) and
\(n_{\rm est}=\vol_{\rm est}/\bar d\). We select \(f\) from a small grid and report the F1 score of the \emph{lowest conductance cluster}. The sensitivity analysis of $f$ is given in Appendix~\ref{app:exp-details}. 

\textbf{Baseline Methods and parameters.}
We compare TL-HFD against HFD Alternating Minimization approach~\citep{fountoulakis2021local},
LH-$p$~\citep{liu2021strongly} (localized quadratic hypergraph diffusions
with $p\in\{1.4,\,2.0\}$), and ACL~\citep{andersen2006local}. Complete details of all parameters used are provided in Appendix, \ref{app:params}. We report median F1 score over all trials. 
 

\textbf{Trivago-Clicks.}
\label{sec:exp-trivago}
The Trivago-clicks dataset~\citep{chodrow2021generative} is a hypergraph of hotel browsing sessions with \(n=172{,}738\) nodes, \(m=233{,}202\) hyperedges, and mean degree \(\bar d\approx 5.6\). Following
\citet{fountoulakis2021local}, we evaluate on $10$ destination-country clusters with \(|\mathcal T|\in[144,945]\). For each cluster, we run $100$ single-seed  trials with seeds selected uniformly at random from that cluster. We use injection factor \(\delta_{\exp}=3\), set \(\sigma=10^{-4}\) and \(\gamma=1.0\), and run TL-HFD for $500$ iterations for unit cut-costs and $1000$ iterations for cardinality cut-costs. For each cluster we sweep $f$ across the grid of fractions and select the configuration with the lowest median output conductance over the trials and report its median F1. Further hyperparameter details are provided in Appendix~\ref{app:params}.

\vspace{-0.2in}
\begin{table}[!ht]
\centering
\caption{Trivago-clicks: median F1 score per cluster. Best result per
cluster within each cut-cost block in \textbf{bold}. Cluster codes, and hyperparameters are
reported in Appendix~\ref{app:trivago-sensitivity}.}
\label{tab:trivago-f1}
\setlength{\tabcolsep}{4pt}
\begin{tabular}{l c c c c c c c c c c}
\toprule
Method    & KOR           & ISL           & PRI           & CRM           & VNM           & HKG           & MLT           & GTM           & UKR           & EST           \\
\midrule
U-HFD     & 0.73          & 0.96          & 0.85          & 0.82          & 0.22          & \textbf{0.95} & \textbf{0.97} & 0.88          & 0.58          & \textbf{0.93} \\
TL-UHFD   & \textbf{0.86} & \textbf{0.98} & \textbf{0.97} & \textbf{0.85} & \textbf{0.58} & 0.65          & \textbf{0.97} & \textbf{0.95} & \textbf{0.65} & 0.79          \\
U-LH-2.0  & 0.69          & 0.85          & 0.78          & 0.70          & 0.24          & 0.90          & 0.88          & 0.82          & 0.49          & 0.90          \\
U-LH-1.4  & 0.69          & 0.84         & 0.79          & 0.75          & 0.28          & 0.87          & 0.92          & 0.83          & 0.47          & 0.90          \\
ACL       & 0.64          & 0.84          & 0.75          & 0.68          & 0.23          & 0.90          & 0.83          & 0.69          & 0.50          & 0.88          \\
\midrule
C-HFD     & 0.75          & \textbf{0.98} & 0.93          & \textbf{0.93} & 0.29          & 0.80 & \textbf{0.97} & \textbf{0.97} & 0.68          & 0.53      \\
TL-CHFD   & \textbf{0.82} & \textbf{0.98} & \textbf{0.97} & 0.85          & \textbf{0.32} & 0.76          & \textbf{0.97} & \textbf{0.97} & \textbf{0.73} & 0.58 \\
C-LH-2.0  & 0.73          & 0.90          & 0.84          & 0.78          & 0.27          & \textbf{0.94} & 0.93          & 0.88          & 0.51          & 0.76          \\
C-LH-1.4  & 0.71          & 0.88          & 0.84          & 0.78          & 0.27          & 0.88          & 0.91          & 0.85          & 0.50          & \textbf{0.78}          \\
\bottomrule
\end{tabular}
\end{table}
\vspace{-0.1in}
\textbf{Results.}
Table~\ref{tab:trivago-f1} summarizes the Trivago-clicks results. For unit
cut-cost, TL-UHFD improves over U-HFD on \(7/10\) clusters and ties on one
additional cluster. The largest gains occur on Vietnam, Puerto Rico, and South
Korea, where the diffusion appears to benefit from a more selective
boundary expansion. Intuitively, top-\(k\) scoring helps in these
cases by preventing the diffusion from spreading through weakly related
session hyperedges and instead prioritizing boundary vertices with both a
large first-order push and stronger attachment to the current active region.
The two clear losses, Hong Kong and Estonia, occur where U-HFD already obtains
very high F1. These cases suggest a different failure mode: when the target
cluster is already well recovered by the U-HFD, the
top-\(k\) restriction can become too conservative or can delay activation of
useful boundary vertices. 

For cardinality cut-cost, TL-CHFD improves over C-HFD on \(5/10\) clusters,
ties on \(3/10\), and loses on \(2/10\). Counting ties, TL-CHFD attains the
best result within the cardinality block on \(7/10\) clusters. The gains are
smaller than in the unit setting but more evenly distributed, consistent with
the role of the cardinality weighting ($d_{\mathrm{in}}(u) = \sum_{e \ni u, e \cap A \neq \emptyset} \min(|A \cap e|, |e \setminus A|) / \lfloor |e|/2 \rfloor$): rather than treating every
nontrivial hyperedge split equally, it gives finer-grained credit to boundary
vertices connected through hyperedges that overlap substantially with the
active region. 

\textbf{Convergence and computational cost.}
HFD's primal alternating-minimization solver converges in few iterations, while
TL-HFD uses slower projected subgradient updates with the
\(O((1+\log T)/T)\) rate of Theorem~\ref{thm:eps-optimality}. The benefit of
TL-HFD is that expansion is explicitly controlled: in single-seed runs, at most
\(k\) new boundary vertices are promoted per iteration, so the maintained
support grows much more gradually, although each iteration still scores the
current boundary. Appendix~\ref{app:runtime-locality} studies this tradeoff on
Amazon-reviews dataset. We observe that TL-HFD often reaches HFD-level F1 after only a small fraction of
its $1000$ iterations, while maintaining substantially smaller nonzero supports
than HFD's transient active sets.

\textbf{Conclusion.}
We introduced \textsc{Thresholded Local Hyper-Flow Diffusion} (TL-HFD), a
locality-by-design first-order method for the non-smooth HFD dual on general
submodular hypergraphs. TL-HFD maintains a seed-anchored active region, performs
degree-preconditioned projected subgradient updates on this region and its
one-hop boundary, and expands via thresholded top-\(k\) boundary activation. We
provide an exact locality invariant, finite-time dual suboptimality for exact
and thresholded updates, and a robust sweep guarantee for approximate iterates
with localized support. We provide an activated-volume accounting bound in
terms of realized local subgradient norms and boundary-push margins. Empirically,
TL-HFD improves sweep quality and F1 on real-world and synthetic hypergraphs.

\textbf{Limitations.}  Although top-$k$ slows the growth of the active and activated sets, it does not remove the cost of scoring the current boundary. Thus the per-iteration work remains local in the scanned region \(L_t\), rather than being controlled by top-$k$ alone. Similarly, Theorem~\ref{thm:explored_volume_topk} is an activated-volume accounting bound: it controls the promoted volume in terms of realized local
subgradient norms and boundary-push margins, but it is not an unconditional graph-size-independent runtime guarantee. A second limitation is optimization speed: projected subgradient descent has the finite-time rate \(O((1+\log T)/T)\), so it may require more iterations to reach high accuracy compared with HFD's alternating minimization. Like HFD, TL-HFD uses a target-volume estimate; removing this dependence is left for future work. 
 
\bibliographystyle{plainnat}
\bibliography{references}


\newpage

\appendix
\section*{\center Appendix}

\section{Related Work}
\label{app:related-work}

\paragraph{Hypergraph cut objectives and hyperedge splitting models.}
A key challenge in hypergraph clustering is that cutting a hyperedge is inherently ambiguous: different applications induce different splitting penalties, which can be modeled through \emph{cut-cost functions}. This viewpoint yields a useful hierarchy of models---from unit and cardinality-based penalties to general \emph{submodular} costs that encode rich multiway structure \citep{lawler1973cutsets, li2018submodular, yoshida2019cheeger, veldt2020minimizing, veldt2022hypergraph}. Many practical approaches either restrict to simple penalties or reduce hypergraphs to graphs via clique/star expansions, which can introduce approximation effects that scale with hyperedge size and obscure higher-order structure.

\paragraph{Diffusion methods and Laplacian primitives for hypergraphs.}
Beyond cut-based formulations, numerous works extend random-walk, Pagerank, and heat-kernel ideas to hypergraphs by defining suitable Laplacians/transition rules and then applying diffusion-and-sweep pipelines \citep{chitra2019random, hayashi2020hypergraph, takai2020hypergraph, chien2021landing}. \citet{zhu2023hypergraphs} studied $p$-Laplacians for edge-dependent vertex weight (EDVW) hypergraphs by expressing EDVW splitting rules as submodular cut functions, enabling 1-Laplacian-based spectral clustering for richer higher-order structure. Recent work has developed efficient numerical primitives for \emph{norm-based} (often non-smooth) hypergraph Laplacians, including algorithms for simulating heat diffusions and computing resolvents that recover Personalized Pagerank-type vectors and hypergraph Laplacian-system analogues \citep{ameranis2023hypergraph}. Subsequent advances provide nearly-linear-time first-order algorithms for hypergraph Pagerank vectors and Laplacian systems over broad families of hypergraph potentials, including those induced by submodular cut functions \citep{ameranis_icml}. However, these methods address a different class of hypergraph Laplacian objectives and are designed as global solvers. In contrast, we focus on locality-by-design optimization of the non-smooth dual arising in the HFD convex diffusion framework \citep{fountoulakis2021local}.

\paragraph{Local clustering and strongly-local algorithms.}
Local clustering in graphs is commonly studied through diffusion vectors and sweep cuts, with canonical guarantees linking personalized Pagerank-like quantities to low-conductance sets near a seed set \citep{andersen2006local, chung2007heat, spielman2004nearly}. A complementary line emphasizes \emph{strong locality}, where runtime depends primarily on the explored/output volume rather than the full graph size, as in capacity-releasing diffusion \citep{wang2017capacity}.  \citet{wei2025high} proposes alternative high-order conductance measures for hypergraphs and greedy seeded procedures aimed at real-time local clustering, offering a heuristic counterpart to diffusion/flow-based methods. In hypergraphs, scalable local clustering methods have been developed under additional structural restrictions (e.g., bounded hyperedge size) or for specialized objectives \citep{ibrahim2020local, liu2021strongly, veldt2020minimizing};  \citet{huang2024densest} study a different local primitive, namely localized densest subhypergraph discovery with negative supermodular objectives, and develop strongly local algorithms when the locality parameter exceeds a critical threshold. However, their framework targets density rather than conductance and is not aimed at optimizing the HFD dual. By contrast, we target the general submodular setting while designing updates that are local-by-design.

\paragraph{Flow/convex formulations and HFD-style local clustering.}
Flow-based and convex-optimization views unify diffusion dynamics with cut objectives in graphs \citep{arora2009expander, fountoulakis2019variational, fountoulakis2020p, fountoulakis2023flow}. For hypergraphs with expressive splitting models, \emph{Local Hyper-Flow Diffusion (HFD)}~\citep{fountoulakis2021local} provides a convex primal--dual framework that supports general submodular cut-costs and yields \emph{edge-size--independent} Cheeger-type guarantees for seeded clustering. Recent flow-based local hypergraph clustering variants include penalized flow objectives \citep{zhong2023penalized} and approximation algorithms for generalized hypergraph ratio-cut objectives via max-flow reductions \citep{veldt2023cut}. In the global setting, \citet{chen2025submodular} generalize  cut-matching games to hypergraphs with polymatroidal cut functions, yielding the first almost-linear-time $\mathcal{O}{(log n)}$ approximation for partitioning general submodular hypergraphs. Their results address global balanced partitioning, rather than the seeded local clustering setting considered here. Our contribution is orthogonal to these formulations: we build directly on HFD's dual objective anddevelop a locality-by-design dual optimizer with finite-time suboptimality certificates that translate into robust sweep quality under early stopping.

\paragraph{Flow-based local hypergraph clustering beyond HFD.}
Several recent methods pursue flow-based seeded clustering objectives on hypergraphs that are adjacent to HFD but differ in formulation and optimization. For example, Penalized Flow Hypergraph Local Clustering introduces random-walk information as a penalization term within a flow-based framework and provides conductance-style guarantees \citep{zhong2023penalized}. Related algorithmic work studies generalized hypergraph ratio-cut objectives via max-flow style reductions, e.g., using cut-matching games \citep{veldt2023cut}.

\paragraph{Strongly-local hypergraph diffusions under restricted models.}
A complementary line develops strongly-local diffusion solvers for hypergraphs, typically for quadratic/cardinality-based cut objectives and/or under bounded hyperedge-size assumptions. Representative examples include strongly-local hypergraph quadratic diffusions \citep{liu2021strongly} and capacity-releasing diffusion variants for hypergraphs \citep{ibrahim2020local}. More recent work explores alternative higher-order primitives (e.g., motif-based objectives) for local clustering in hypergraphs \citep{italiano2025local}. 

\paragraph{Random-walk and Pagerank-style local clustering on hypergraphs.}
Beyond convex/flow formulations, Pagerank-style \citep{andersen2006local} and random-walk-based \citep{chung2007heat, gleich2012vertex, kloster2014heat} seeded clustering has been extended to more general graph models and applied to hypergraphs via EDVW-type representations, with proofs leveraging Lov\'asz--Simonovits curve techniques \citep{li2024provably}. Related theoretical developments establish Lov\'asz--Simonovits style results for hypergraph random walks \citep{kamal2024lovasz}.

\paragraph{Spectral and variational learning on hypergraphs.} A widely used approach to hypergraph learning constructs Laplacian-like operators (or relaxations) and applies spectral methods for clustering, classification, and embedding \citep{zhou2006learning, hein2013total}. This includes both practical algorithmic formulations and theoretical investigations of spectral partitioning in generative models \citep{ghoshdastidar2014consistency}. Variational formulations-including hypergraph total variation and related regularizers- offer an alternative route to learning with higher-order structure, often enabling robust semi-supervised learning and segmentation-style objectives \citep{hein2013total}.

\newpage

\section{Theoretical Proofs}
\label{app:proofs}

This appendix provides proofs for all formal results stated in the main body.
Throughout, let $D=\mathrm{diag}(d)$ with $d_u>0$ for all $u\in V$, and define the weighted
inner product and norms
\[
\langle a,b\rangle_D := a^\top D b,\qquad
\|a\|_D := \sqrt{a^\top D a},\qquad
\|g\|_{D^{-1}} := \sqrt{g^\top D^{-1} g}.
\]
For a convex function $F$, we write $\partial F(x)$ for its subdifferential at $x$.


\subsection{ Proof of Lemma \ref{lem:subgrad_in_F} }
\label{app:subgrad}

Recall the subgradient selection in \eqref{eq:subgrad}:
\[
g(x)=\sum_{e\in E}\vartheta_e f_e(x)\rho_e(x) + \sigma D x - (\Delta-d),
\qquad \rho_e(x)\in M_e(x).
\]

To show that $g(x)=\sum_{e\in E}\vartheta_e f_e(x)\rho_e(x) + \sigma D x-(\Delta-d)\ \in\ \partial F(x)$, we first show $f_e(x)\rho_e(x)\in \partial g_e(x)$. \newline
\begin{Auxiliary Lemma}[Subgradient of the squared Lov\'asz term]
\label{lem:subgrad-ge}
Fix $e$ and define $f_e(x)=\max_{\rho\in B_e}\langle \rho,x\rangle$ and $g_e(x)=\tfrac12 f_e(x)^2$.
For any $x\in \mathbb{R}^V$ and any choice $\rho_e(x)\in M_e(x)$, the vector $f_e(x)\rho_e(x)$ is a subgradient of
$g_e$ at $x$:
\[
f_e(x)\rho_e(x)\in \partial g_e(x).
\]
\end{Auxiliary Lemma}

\begin{proof}
We use the following standard results.

\paragraph{Fact 1.}
For the support function $f_e(x)=\max_{\rho\in B_e}\langle \rho,x\rangle$ of a compact convex set $B_e$,
\[
\partial f_e(x)=\mathrm{conv}\big(M_e(x)\big).
\]
In particular, any maximizer $\rho_e(x)\in M_e(x)$ belongs to $\partial f_e(x)$.

\paragraph{Fact 2.}
Let $\phi : \mathbb{R} \to \mathbb{R}$ be convex, differentiable, and non-decreasing on the range of $f$, and let $f$ be convex. Then
\[
\partial(\phi \circ f)(x) = \phi'(f(x)) \cdot \partial f(x),
\]
where the product denotes scalar multiplication of a set.

Here $\phi(s) = \tfrac{1}{2}s^2$ with $\phi'(s) = s$. Since the cut-cost is normalized, $0 \in B_e$, so $f_e(x) = \max_{\rho \in B_e}\langle \rho, x\rangle \geq \langle 0, x\rangle = 0$ for all $x$. Hence $f_e$ takes values in $[0, \infty)$, where $\phi$ is non-decreasing, and Fact \textbf{2} applies. Therefore
\[
\partial g_e(x) = \partial(\phi \circ f_e)(x) = \phi'(f_e(x)) \, \partial f_e(x) = f_e(x) \, \partial f_e(x).
\]
Choosing $\rho_e(x) \in \partial f_e(x)$ gives $f_e(x)\rho_e(x) \in \partial g_e(x)$.

\end{proof}

\paragraph{\textsc{\underline{Proof of $g(x) \in \partial F(x)$:} }}

\SubGrad*
\begin{proof}
Write
\[
F(x) = \sum_{e\in E}\vartheta_e g_e(x) \;+\; \underbrace{\left(\frac{\sigma}{2}x^\top D x - \langle \Delta-d,x\rangle\right)}_{=:q(x)}.
\]
By Auxiliary Lemma ~\ref{lem:subgrad-ge}, for each $e$ we have $\vartheta_e f_e(x)\rho_e(x)\in \partial(\vartheta_e g_e)(x)$.
Since the subdifferential of a finite sum of convex functions equals the Minkowski sum of subdifferentials,
\[
\sum_{e\in E}\vartheta_e f_e(x)\rho_e(x) \in \partial\!\left(\sum_{e\in E}\vartheta_e g_e\right)(x).
\]
Moreover $q$ is differentiable with $\nabla q(x)=\sigma D x-(\Delta-d)$, and for a differentiable convex function
$\partial q(x)=\{\nabla q(x)\}$.
Therefore,
\[
\partial F(x)=\partial\!\left(\sum_{e\in E}\vartheta_e g_e\right)(x) + \nabla q(x),
\]
and selecting the element $\sum_{e}\vartheta_e f_e(x)\rho_e(x)$ from the first term gives
\[
g(x)=\sum_{e\in E}\vartheta_e f_e(x)\rho_e(x) + \sigma D x-(\Delta-d)\ \in\ \partial F(x).
\]
\end{proof}

\subsection{Proof of Lemma \ref{lem:locality}: Locality of iterates}
\label{app:locality}
Without loss of generality, for the following proof we assume that there are no isolated vertices: as we compute $D^{-1}$ (i.e., for any vertex $u$, $d_u>0$).
\Locality*
\begin{proof}
Fix an iteration $t$ and suppose $\mathrm{supp}(x^{(t)})\subseteq A^{(t)}$, where $A^{(t)}:=\mathrm{supp}(x^{(t)})\cup S$.
Let $u\notin A^{(t)}\cup \partial A^{(t)}$. Then for every hyperedge $e\ni u$ we have $e\cap A^{(t)}=\emptyset$ by definition
of the boundary. Since $x^{(t)}$ is zero outside $A^{(t)}$, it follows that $x^{(t)}|_e\equiv 0$ on each incident hyperedge.
Because $0\in B_e$ (normalized base polytopes contain the origin), we have $f_e(0)=0$, hence $f_e(x^{(t)})=0$ for
every hyperedge incident to $u$.

Consider the subgradient $g^{(t)}\in \partial F(x^{(t)})$ constructed in equation \eqref{eq:subgrad}.
At coordinate $u$, all hyperedge contributions vanish because they are multiplied by $f_e(x^{(t)})=0$. Therefore,
\[
[g^{(t)}]_u=\sigma d_u x^{(t)}_u - (\Delta(u)-d_u).
\]
Since $u\notin A(t)$, we have $x^{(t)}_u=0$.
Also $u\notin S$ because $S\subseteq A(t)$, hence $\Delta(u)=0$, so $\Delta(u)-d_u=-d_u$.
Thus $[g^{(t)}]_u=d_u>0$.

The unprojected update at $u$ is
\[
y_u := x^{(t)}_u - \eta_{t+1}\frac{[g^{(t)}]_u}{d_u} = 0 - \eta_{t+1}\cdot 1 = -\eta_{t+1}<0,
\]
so after projection onto $\mathbb{R}^V_+$ we obtain $x^{(t+1)}_u=\max\{y_u,0\}=0$.
Hence all coordinates outside $A(t)\cup \partial A(t)$ remain identically zero.
\end{proof}

\newpage
\subsection{Proof of Theorem \ref{thm:eps-optimality}: $\epsilon$-optimality of local PSGD.}
\label{app:optimality}

\begin{Auxiliary Lemma}[Strong convexity and subgradients]
\label{lem:sc-subgrad}
Let $F$ be $\sigma$-strongly convex w.r.t.\ $\|\cdot\|_D$. Then for every $x,y$ and every $g\in \partial F(x)$,
\[
F(y)\ge F(x) + \langle g, y-x\rangle + \frac{\sigma}{2}\|y-x\|_D^2.
\]
Equivalently,
\[
\langle g, x-y\rangle \ge F(x)-F(y) + \frac{\sigma}{2}\|x-y\|_D^2.
\]
\end{Auxiliary Lemma}

\begin{proof}
This is a standard equivalent characterization of strong convexity: apply the definition of strong convexity using the
supporting hyperplane at $x$ given by a subgradient $g\in \partial F(x)$.  
\end{proof}

\optimality*

\begin{proof}
Let
\[
    r^{(t)}:=x^{(t)}-x^\star.
\]
Fix an iteration \(t\), and write
\[
    L_t:=A^{(t)}\cup\partial A^{(t)},
    \qquad
    F_t:=V\setminus L_t.
\]
Define the unprojected point
\[
    y^{(t)}
    :=
    x^{(t)}-\eta_{t+1}D^{-1}g^{(t)}.
\]

By Lemma~\ref{lem:locality}, for every \(v\in F_t\), we have
\[
    x_v^{(t)}=0,
    \qquad
    g_v^{(t)}=d_v.
\]
Therefore
\[
    y_v^{(t)}
    =
    x_v^{(t)}-\eta_{t+1}\frac{g_v^{(t)}}{d_v}
    =
    -\eta_{t+1}<0.
\]
Projection onto the nonnegative orthant gives
\[
    x_v^{(t+1)}=0
    \qquad
    \forall v\in F_t.
\]
Hence, for every \(v\in F_t\),
\[
    r_v^{(t+1)}
    =
    x_v^{(t+1)}-x_v^\star
    =
    -x_v^\star
    =
    x_v^{(t)}-x_v^\star
    =
    r_v^{(t)}.
\]
Thus the contribution of \(F_t\) to the squared \(D\)-norm is identical at
iterations \(t\) and \(t+1\).

On the local region \(L_t\), projection is nonexpansive coordinatewise in the
\(D\)-norm. Therefore
\[
\begin{aligned}
    \|r^{(t+1)}\|_D^2
    &\le
    \|r^{(t)}\|_D^2
    -
    2\eta_{t+1}
    \sum_{v\in L_t}g_v^{(t)}r_v^{(t)}
    +
    \eta_{t+1}^2
    \sum_{v\in L_t}\frac{(g_v^{(t)})^2}{d_v}  \\
    &=
    \|r^{(t)}\|_D^2
    -
    2\eta_{t+1}
    \sum_{v\in L_t}g_v^{(t)}r_v^{(t)}
    +
    \eta_{t+1}^2
    \|g^{(t)}\|_{D^{-1},L_t}^2.
\end{aligned}
\]
We now compare the local inner product to the full inner product. Since
\(g_v^{(t)}=d_v\) and \(r_v^{(t)}=-x_v^\star\le0\) for \(v\in F_t\),
\[
    \sum_{v\in F_t}g_v^{(t)}r_v^{(t)}
    =
    -\sum_{v\in F_t}d_vx_v^\star
    \le0.
\]
Consequently,
\[
    \sum_{v\in L_t}g_v^{(t)}r_v^{(t)}
    =
    \langle g^{(t)},r^{(t)}\rangle
    -
    \sum_{v\in F_t}g_v^{(t)}r_v^{(t)}
    \ge
    \langle g^{(t)},r^{(t)}\rangle.
\]
Since \(F\) is \(\sigma\)-strongly convex with respect to \(\|\cdot\|_D\),
\[
    \langle g^{(t)},r^{(t)}\rangle
    \ge
    F(x^{(t)})-F(x^\star)
    +
    \frac{\sigma}{2}\|r^{(t)}\|_D^2.
\]
Combining the preceding inequalities gives
\[
\begin{aligned}
    \|r^{(t+1)}\|_D^2
    &\le
    \|r^{(t)}\|_D^2
    -
    2\eta_{t+1}
    \left(
        F(x^{(t)})-F(x^\star)
        +
        \frac{\sigma}{2}\|r^{(t)}\|_D^2
    \right)
    +
    \eta_{t+1}^2
    \|g^{(t)}\|_{D^{-1},L_t}^2.
\end{aligned}
\]
Rearranging,
\[
    2\eta_{t+1}
    \bigl(F(x^{(t)})-F(x^\star)\bigr)
    \le
    (1-\sigma\eta_{t+1})\|r^{(t)}\|_D^2
    -
    \|r^{(t+1)}\|_D^2
    +
    \eta_{t+1}^2
    \|g^{(t)}\|_{D^{-1},L_t}^2.
\]
Using
\[
    \eta_{t+1}=\frac{1}{\sigma(t+1)}
\]
gives
\[
    1-\sigma\eta_{t+1}
    =
    \frac{t}{t+1}.
\]
Multiplying by \(t+1\), we obtain
\[
    \frac{2}{\sigma}
    \bigl(F(x^{(t)})-F(x^\star)\bigr)
    \le
    t\|r^{(t)}\|_D^2
    -
    (t+1)\|r^{(t+1)}\|_D^2
    +
    \frac{\|g^{(t)}\|_{D^{-1},L_t}^2}{\sigma^2(t+1)}.
\]
Summing over \(t=0,\ldots,T-1\), the norm terms telescope:
\[
    \sum_{t=0}^{T-1}
    \left(
        t\|r^{(t)}\|_D^2
        -
        (t+1)\|r^{(t+1)}\|_D^2
    \right)
    =
    -T\|r^{(T)}\|_D^2
    \le0.
\]
Therefore
\[
    \frac{2}{\sigma}
    \sum_{t=0}^{T-1}
    \bigl(F(x^{(t)})-F(x^\star)\bigr)
    \le
    \frac{1}{\sigma^2}
    \sum_{t=0}^{T-1}
    \frac{\|g^{(t)}\|_{D^{-1},L_t}^2}{t+1}.
\]
Multiplying by \(\sigma/2\),
\[
    \sum_{t=0}^{T-1}
    \bigl(F(x^{(t)})-F(x^\star)\bigr)
    \le
    \frac{1}{2\sigma}
    \sum_{t=0}^{T-1}
    \frac{\|g^{(t)}\|_{D^{-1},L_t}^2}{t+1}.
\]
Since the best iterate is no worse than the average,
\[
    F(\hat x^{(T)})-F(x^\star)
    \le
    \frac{1}{2\sigma T}
    \sum_{t=0}^{T-1}
    \frac{\|g^{(t)}\|_{D^{-1},L_t}^2}{t+1}.
\]
Finally, if
\[
    G_{\mathrm{loc}}^2(T)
    =
    \max_{0\le t<T}
    \|g^{(t)}\|_{D^{-1},L_t}^2,
\]
then
\[
    \sum_{t=0}^{T-1}
    \frac{\|g^{(t)}\|_{D^{-1},L_t}^2}{t+1}
    \le
    G_{\mathrm{loc}}^2(T)
    \sum_{t=0}^{T-1}\frac1{t+1}
    \le
    G_{\mathrm{loc}}^2(T)(1+\log T).
\]
Thus
\[
    F(\hat x^{(T)})-F(x^\star)
    \le
    \frac{G_{\mathrm{loc}}^2(T)(1+\log T)}{2\sigma T}.
\]
\end{proof}

\newpage
\subsection{Proof of Lemma \ref{lem:threshold-inexact}.  Thresholding equals a truncated preconditioned step}
\label{app:thresholding_error}

The argument follows the standard inexact projected-gradient analysis \cite{schmidt2011convergence}, specialized to the truncation error from Lemma~\ref{lem:threshold-inexact}.

\thresholding*

\begin{proof}
Let
\[
    y^{(t)}
    :=
    x^{(t)}-\eta_{t+1}D^{-1}g^{(t)}.
\]
By definition of \(e^{(t)}\), the vector
\[
    \bar y^{(t)}
    :=
    y^{(t)}-e^{(t)}
\]
agrees with \(y^{(t)}\) on all coordinates outside \(B^{(t)}\), and satisfies
\[
    \bar y_u^{(t)}=0
    \qquad
    \forall u\in B^{(t)}.
\]
Indeed, if \(u\in B^{(t)}\), then \(u\in\partial A^{(t)}\), so
\(x_u^{(t)}=0\). Moreover \(\kappa^{(t)}(u)>0\), hence
\(g_u^{(t)}<0\), and therefore
\[
    y_u^{(t)}
    =
    -\eta_{t+1}\frac{g_u^{(t)}}{d_u}
    =
    \eta_{t+1}\kappa^{(t)}(u)
    =
    e_u^{(t)}.
\]
Thus \(\bar y_u^{(t)}=y_u^{(t)}-e_u^{(t)}=0\) on \(B^{(t)}\).

We now check the projected update coordinatewise.

First, if \(u\in A^{(t)}\), then \(u\notin B^{(t)}\), so
\(\bar y_u^{(t)}=y_u^{(t)}\). Algorithm~\ref{alg:pq-psgd-theory} performs the
exact projected PSGD update on active vertices, hence
\[
    x_u^{(t+1)}
    =
    \max\{0,y_u^{(t)}\}
    =
    \max\{0,\bar y_u^{(t)}\}.
\]

Second, if \(u\in K_t\), then \(u\notin B^{(t)}\), so again
\(\bar y_u^{(t)}=y_u^{(t)}\). Since \(u\in\partial A^{(t)}\), we have
\(x_u^{(t)}=0\). Therefore
\[
    \max\{0,y_u^{(t)}\}
    =
    \eta_{t+1}\kappa^{(t)}(u),
\]
which is exactly the boundary activation performed by the algorithm.

Third, if \(u\in \partial A^{(t)}\setminus K_t\) and
\(\kappa^{(t)}(u)=0\), then \(g_u^{(t)}\ge0\). Since \(x_u^{(t)}=0\),
\[
    y_u^{(t)}
    =
    -\eta_{t+1}\frac{g_u^{(t)}}{d_u}
    \le0.
\]
The algorithm leaves this coordinate at zero, and
\[
    \max\{0,\bar y_u^{(t)}\}
    =
    \max\{0,y_u^{(t)}\}
    =
    0.
\]

Fourth, if \(u\in B^{(t)}\), then by construction
\[
    \bar y_u^{(t)}=0.
\]
The algorithm skips this positively pushed boundary vertex and leaves
\(x_u^{(t+1)}=0\). Hence
\[
    x_u^{(t+1)}
    =
    \max\{0,\bar y_u^{(t)}\}.
\]

Finally, if \(u\notin A^{(t)}\cup\partial A^{(t)}\), then Lemma~\ref{lem:locality}
gives
\[
    g_u^{(t)}=d_u>0,
    \qquad
    x_u^{(t)}=0.
\]
Hence
\[
    y_u^{(t)}
    =
    -\eta_{t+1}<0.
\]
The algorithm leaves this coordinate at zero, and
\[
    \max\{0,\bar y_u^{(t)}\}
    =
    \max\{0,y_u^{(t)}\}
    =
    0.
\]

Combining the cases,
\[
    x^{(t+1)}
    =
    \Pi_{\mathbb R_+^{|\V|}}(\bar y^{(t)})
    =
    \Pi_{\mathbb R_+^{|\V|}}
    \left(
        x^{(t)}-\eta_{t+1}D^{-1}g^{(t)}-e^{(t)}
    \right).
\]

The norm identities follow directly from the definition of \(e^{(t)}\):
\[
    \|e^{(t)}\|_D^2
    =
    \sum_{u\in V}d_u(e_u^{(t)})^2
    =
    \eta_{t+1}^2
    \sum_{u\in B^{(t)}}d_u\bigl(\kappa^{(t)}(u)\bigr)^2,
\]
and
\[
    \|e^{(t)}\|_{1,D}
    =
    \sum_{u\in V}d_u|e_u^{(t)}|
    =
    \eta_{t+1}
    \sum_{u\in B^{(t)}}d_u\kappa^{(t)}(u).
\]

If the coordinatewise bound
\[
    |g_u^{(t)}/d_u|\le B_{\rm loc}
\]
holds on \(A^{(t)}\cup\partial A^{(t)}\), then for every
\(u\in B^{(t)}\subseteq \partial A^{(t)}\),
\[
    \kappa^{(t)}(u)
    =
    -\frac{g_u^{(t)}}{d_u}
    \le
    \left|\frac{g_u^{(t)}}{d_u}\right|
    \le
    B_{\rm loc}.
\]
Therefore
\[
    \|e^{(t)}\|_D^2
    \le
    \eta_{t+1}^2B_{\rm loc}^2
    \sum_{u\in B^{(t)}}d_u
    =
    \eta_{t+1}^2B_{\rm loc}^2\vol(B^{(t)}),
\]
and
\[
    \|e^{(t)}\|_{1,D}
    \le
    \eta_{t+1}B_{\rm loc}
    \sum_{u\in B^{(t)}}d_u
    =
    \eta_{t+1}B_{\rm loc}\vol(B^{(t)}).
\]
The final statement follows from Proposition~\ref{prop:Gloc-bound}.
\end{proof}
\newpage

\subsection{Proof of Corollary \ref{cor:dual_gap_threshold}: Suboptimality under boundary truncation}
\label{app:dual_gap_threshold}

\thresholdinggap*

\begin{proof}

\textbf{Algorithm and box-constraint scope.} The bound, involving the accumulated truncation error
$\sum_{t<T}\|e^{(t)}\|_{1,D}$, uses only the hypotheses of
Theorem~\ref{thm:eps-optimality}, the thresholding-error characterization
of Lemma~\ref{lem:threshold-inexact}, and the coordinate-wise bound
$\|x^\star\|_\infty \le \bar{x}$ on the optimizer
(Auxiliary Lemma~\ref{lem:xstar-bound}). It applies to the unconstrained
Algorithm~\ref{alg:pq-psgd-theory} because $\Pi_{\mathbb{R}^{|V|}_+}$ is
coordinatewise and 1-Lipschitz in the $D$-norm. The same argument
applies to the box-constrained Algorithm~\ref{alg:app-pq-psgd-theory}
because $\Pi_{\mathcal X}$ is also coordinatewise and 1-Lipschitz in
the $D$-norm (Auxiliary Lemma~\ref{lem:box-locality}); for skipped
boundary coordinates, the box-projected truncation error is
$\min\{\bar{x},\eta_{t+1}\kappa^{(t)}(u)\}$ and is therefore bounded by
$\eta_{t+1}\kappa^{(t)}(u)$. The $\varepsilon_{\mathrm{vol}}(T)$ specialization additionally invokes the uniform bound $|g_u^{(t)}/d_u| \le B_{\mathrm{loc}}$ from Proposition~\ref{prop:Gloc-bound}, which holds for Algorithm~\ref{alg:app-pq-psgd-theory}; for runs of Algorithm~\ref{alg:pq-psgd-theory} whose iterates stay below $\bar{x}$, the two algorithms coincide (Remark~\ref{rem:box-nonbinding}).

The proof follows the localized inexact projected-subgradient argument. Define
the exact projected local PSGD iterate
\[
    \widetilde x^{(t+1)}
    :=
    \Pi_{\mathbb R_+^{|\V|}}
    \left(
        x^{(t)}-\eta_{t+1}D^{-1}g^{(t)}
    \right),
\]
and define
\[
    d^{(t)}:=x^{(t)}-x^\star,
    \qquad
    \widetilde d^{(t+1)}:=\widetilde x^{(t+1)}-x^\star,
    \qquad
    d^{(t+1)}:=x^{(t+1)}-x^\star.
\]

From the localized one-step inequality used in the proof of
Theorem~\ref{thm:eps-optimality}, we have
\[
\begin{aligned}
    2\eta_{t+1}
    \bigl(F(x^{(t)})-F(x^\star)\bigr)
    &\le
    (1-\sigma\eta_{t+1})\|d^{(t)}\|_D^2
    -
    \|\widetilde d^{(t+1)}\|_D^2  \\
    &\qquad
    +
    \eta_{t+1}^2
    \|g^{(t)}\|_{D^{-1},L_t}^2.
\end{aligned}
\]
This is the exact-update progress inequality, with the global subgradient norm
replaced by the restricted local norm on \(L_t\).

We now compare the exact iterate \(\widetilde x^{(t+1)}\) with the thresholded
iterate \(x^{(t+1)}\). By Lemma~\ref{lem:threshold-inexact}, the two iterates
can differ only on the skipped positively pushed set \(B^{(t)}\). For
\(u\in B^{(t)}\), the exact update gives
\[
    \widetilde x_u^{(t+1)}
    =
    e_u^{(t)}
    =
    \eta_{t+1}\kappa^{(t)}(u),
\]
while the thresholded update leaves
\[
    x_u^{(t+1)}=0.
\]
Therefore
\[
\begin{aligned}
    \|d^{(t+1)}\|_D^2-\|\widetilde d^{(t+1)}\|_D^2
    &=
    \sum_{u\in B^{(t)}}d_u
    \left[
        (x_u^{(t+1)}-x_u^\star)^2
        -
        (\widetilde x_u^{(t+1)}-x_u^\star)^2
    \right] \\
    &=
    \sum_{u\in B^{(t)}}d_u
    \left[
        (x_u^\star)^2
        -
        (e_u^{(t)}-x_u^\star)^2
    \right].
\end{aligned}
\]
Expanding the square,
\[
    (x_u^\star)^2-(e_u^{(t)}-x_u^\star)^2
    =
    2e_u^{(t)}x_u^\star-(e_u^{(t)})^2
    \le
    2\bar x e_u^{(t)},
\]
because \(x^\star\ge0\), \(e_u^{(t)}\ge0\), and
\(\|x^\star\|_\infty\le \bar x\). Hence
\[
    \|d^{(t+1)}\|_D^2
    \le
    \|\widetilde d^{(t+1)}\|_D^2
    +
    2\bar x\|e^{(t)}\|_{1,D}.
\]
Substituting this into the exact-update one-step inequality gives
\[
\begin{aligned}
    2\eta_{t+1}
    \bigl(F(x^{(t)})-F(x^\star)\bigr)
    &\le
    (1-\sigma\eta_{t+1})\|d^{(t)}\|_D^2
    -
    \|d^{(t+1)}\|_D^2  \\
    &\qquad
    +
    \eta_{t+1}^2
    \|g^{(t)}\|_{D^{-1},L_t}^2
    +
    2\bar x\|e^{(t)}\|_{1,D}.
\end{aligned}
\]

Now use
\[
    \eta_{t+1}=\frac{1}{\sigma(t+1)}.
\]
Then
\[
    1-\sigma\eta_{t+1}
    =
    \frac{t}{t+1}.
\]
Multiplying the preceding inequality by \(t+1\), we obtain
\[
\begin{aligned}
    \frac{2}{\sigma}
    \bigl(F(x^{(t)})-F(x^\star)\bigr)
    &\le
    t\|d^{(t)}\|_D^2
    -
    (t+1)\|d^{(t+1)}\|_D^2  \\
    &\qquad
    +
    \frac{\|g^{(t)}\|_{D^{-1},L_t}^2}{\sigma^2(t+1)}
    +
    2\bar x(t+1)\|e^{(t)}\|_{1,D}.
\end{aligned}
\]
Summing over \(t=0,\dots,T-1\), the first two terms telescope:
\[
    \sum_{t=0}^{T-1}
    \left[
        t\|d^{(t)}\|_D^2
        -
        (t+1)\|d^{(t+1)}\|_D^2
    \right]
    =
    -T\|d^{(T)}\|_D^2
    \le0.
\]
Therefore
\[
\begin{aligned}
    \frac{2}{\sigma}
    \sum_{t=0}^{T-1}
    \bigl(F(x^{(t)})-F(x^\star)\bigr)
    &\le
    \frac{1}{\sigma^2}
    \sum_{t=0}^{T-1}
    \frac{\|g^{(t)}\|_{D^{-1},L_t}^2}{t+1}  \\
    &\qquad
    +
    2\bar x
    \sum_{t=0}^{T-1}
    (t+1)\|e^{(t)}\|_{1,D}.
\end{aligned}
\]
Multiplying by \(\sigma/2\),
\[
\begin{aligned}
    \sum_{t=0}^{T-1}
    \bigl(F(x^{(t)})-F(x^\star)\bigr)
    &\le
    \frac{1}{2\sigma}
    \sum_{t=0}^{T-1}
    \frac{\|g^{(t)}\|_{D^{-1},L_t}^2}{t+1} \\
    &\qquad
    +
    \sigma\bar x
    \sum_{t=0}^{T-1}
    (t+1)\|e^{(t)}\|_{1,D}.
\end{aligned}
\]

By Lemma~\ref{lem:threshold-inexact},
\[
    \|e^{(t)}\|_{1,D}
    =
    \eta_{t+1}
    \sum_{u\in B^{(t)}}d_u\kappa^{(t)}(u).
\]
Since
\[
    (t+1)\eta_{t+1}
    =
    \frac{1}{\sigma},
\]
we get
\[
    \sigma(t+1)\|e^{(t)}\|_{1,D}
    =
    \sum_{u\in B^{(t)}}d_u\kappa^{(t)}(u).
\]
Hence
\[
    \sigma\bar x
    \sum_{t=0}^{T-1}
    (t+1)\|e^{(t)}\|_{1,D}
    =
    \bar x
    \sum_{t=0}^{T-1}
    \sum_{u\in B^{(t)}}d_u\kappa^{(t)}(u).
\]
Therefore
\[
\begin{aligned}
    \sum_{t=0}^{T-1}
    \bigl(F(x^{(t)})-F(x^\star)\bigr)
    &\le
    \frac{1}{2\sigma}
    \sum_{t=0}^{T-1}
    \frac{\|g^{(t)}\|_{D^{-1},L_t}^2}{t+1} \\
    &\qquad
    +
    \bar x
    \sum_{t=0}^{T-1}
    \sum_{u\in B^{(t)}}d_u\kappa^{(t)}(u).
\end{aligned}
\]
Dividing by \(T\) and using that the best iterate is no worse than the average,
\[
\begin{aligned}
    F(\hat x^{(T)})-F(x^\star)
    &\le
    \frac{1}{2\sigma T}
    \sum_{t=0}^{T-1}
    \frac{\|g^{(t)}\|_{D^{-1},L_t}^2}{t+1} \\
    &\qquad
    +
    \frac{\bar x}{T}
    \sum_{t=0}^{T-1}
    \sum_{u\in B^{(t)}}d_u\kappa^{(t)}(u) \\
    &=
    \varepsilon_{\rm loc}(T)
    +
    \frac{\bar x}{T}
    \sum_{t=0}^{T-1}
    \sum_{u\in B^{(t)}}d_u\kappa^{(t)}(u).
\end{aligned}
\]
This proves the bound.

If the coordinatewise local bound
\[
    |g_u^{(t)}/d_u|\le B_{\rm loc}
\]
holds on \(A^{(t)}\cup\partial A^{(t)}\), then for every
\(u\in B^{(t)}\),
\[
    \kappa^{(t)}(u)\le B_{\rm loc}.
\]
Therefore
\[
    \sum_{u\in B^{(t)}}d_u\kappa^{(t)}(u)
    \le
    B_{\rm loc}\sum_{u\in B^{(t)}}d_u
    =
    B_{\rm loc}\vol(B^{(t)}).
\]
Substituting gives
\[
    F(\hat x^{(T)})-F(x^\star)
    \le
    \varepsilon_{\rm loc}(T)
    +
    \frac{\bar x B_{\rm loc}}{T}
    \sum_{t=0}^{T-1}\vol(B^{(t)})
    =
    \varepsilon_{\rm vol}(T).
\]

Finally, by Theorem~\ref{thm:eps-optimality},
\[
    \varepsilon_{\rm loc}(T)
    =
    \frac{1}{2\sigma T}
    \sum_{t=0}^{T-1}
    \frac{\|g^{(t)}\|_{D^{-1},L_t}^2}{t+1}
    \le
    \frac{G_{\rm loc}^2(T)(1+\log T)}{2\sigma T},
\]
where
\[
    G_{\rm loc}^2(T)
    =
    \max_{0\le t<T}
    \|g^{(t)}\|_{D^{-1},L_t}^2.
\]
Under the hypotheses of Proposition~\ref{prop:Gloc-bound} for the
box-constrained variant, one may take
\[
    \bar x=\frac{\delta-1}{\sigma}
    \qquad\text{and}\qquad
    B_{\rm loc}=(\delta-1)(1/\sigma+2),
\]
to get $\varepsilon_{\rm vol}(T)=\varepsilon_{\rm loc}(T)
+\frac{(\delta-1)B_{\rm loc}}{\sigma T}
\sum_{t=0}^{T-1}\vol(B^{(t)})$.
\end{proof}

\newpage

\newpage
\subsection{Proof of Theorem \ref{thm:robust-sweep}: Sweep Bound under Dual Suboptimality}
\label{app:robust-sweep}

\begin{Auxiliary Lemma} [HFD lower bound on the optimal dual value]
\label{lem:hfd-dual-lower}
Assume Assumptions~\ref{ass:overlap}, \ref{ass:volume}, and
\ref{ass:parameter}. In particular, assume $ \delta=\frac{3}{\alpha}, \qquad  0\le w_e(A)\le 1 \quad \forall e\in \E,\ A\subseteq e, $ and $\sigma\le \frac{\beta\Phi(C)}{3}$. Let $D(x):=-F(x), \qquad x^\star\in\arg\min_{x\ge0}F(x)$. Equivalently, \(x^\star\in\arg\max_{x\ge0}D(x)\). Define $D_C:=\frac{\vol(C)^2}{2\vol(\partial C)}$, and assume \(\vol(\partial C)>0\). Then $D(x^\star)\ge D_C$.
\end{Auxiliary Lemma}

\begin{proof}
Since \(x^\star\in\arg\max_{x\ge0}D(x)\), it suffices to exhibit one feasible
vector \(x\ge0\) such that \(D(x)\ge D_C\). We use the test vector
\[
    x=\lambda \mathbf 1_C,
    \qquad
    \lambda\ge0.
\]
Recall that
\[
    D(x)
    =
    \langle \Delta-d,x\rangle
    -
    \frac12\sum_{e\in E}\vartheta_e f_e(x)^2
    -
    \frac{\sigma}{2}x^\top D x.
\]
By positive homogeneity of the Lov\'asz extension,
\[
    f_e(\lambda\mathbf 1_C)=\lambda f_e(\mathbf 1_C).
\]
Moreover,
\[
    f_e(\mathbf 1_C)=w_e(C\cap e).
\]

We first lower-bound the linear term. Since
\[
    \Delta(v)=\delta d_v
    \quad\text{for }v\in S
\]
and \(\Delta(v)=0\) otherwise,
\[
\begin{aligned}
    \langle \Delta-d,\mathbf 1_C\rangle
    &=
    \Delta(C)-\vol(C)  \\
    &=
    \delta \vol(S\cap C)-\vol(C).
\end{aligned}
\]
By Assumption~\ref{ass:overlap},
\[
    \vol(S\cap C)\ge \alpha\vol(C).
\]
Using \(\delta=3/\alpha\), we get
\[
    \langle \Delta-d,\mathbf 1_C\rangle
    \ge
    \frac{3}{\alpha}\alpha\vol(C)-\vol(C)
    =
    2\vol(C).
\]
Therefore
\[
    \langle \Delta-d,\lambda\mathbf 1_C\rangle
    \ge
    2\lambda\vol(C).
\]

Next, since \(0\le w_e(A)\le1\), we have
\[
    w_e(C\cap e)^2\le w_e(C\cap e).
\]
Thus
\[
\begin{aligned}
    \sum_{e\in E}\vartheta_e f_e(\mathbf 1_C)^2
    &=
    \sum_{e\in E}\vartheta_e w_e(C\cap e)^2  \\
    &\le
    \sum_{e\in E}\vartheta_e w_e(C\cap e) \\
    &=
    \vol(\partial C).
\end{aligned}
\]
Also,
\[
    \mathbf 1_C^\top D\mathbf 1_C=\vol(C).
\]
Combining these estimates gives
\[
\begin{aligned}
    D(\lambda\mathbf 1_C)
    &\ge
    2\lambda\vol(C)
    -
    \frac{\lambda^2}{2}\vol(\partial C)
    -
    \frac{\sigma\lambda^2}{2}\vol(C).
\end{aligned}
\]

Choose
\[
    \lambda=\frac{\vol(C)}{\vol(\partial C)}.
\]
Then
\[
\begin{aligned}
    D(\lambda\mathbf 1_C)
    &\ge
    2\frac{\vol(C)^2}{\vol(\partial C)}
    -
    \frac12\frac{\vol(C)^2}{\vol(\partial C)}
    -
    \frac{\sigma}{2}
    \frac{\vol(C)^3}{\vol(\partial C)^2}  \\
    &=
    \left(
        \frac32
        -
        \frac{\sigma}{2}
        \frac{\vol(C)}{\vol(\partial C)}
    \right)
    \frac{\vol(C)^2}{\vol(\partial C)}.
\end{aligned}
\]
By Assumption~\ref{ass:volume},
\[
    \vol(C)\le \vol(V\setminus C),
\]
so
\[
    \Phi(C)=\frac{\vol(\partial C)}{\vol(C)}.
\]
Since \(\beta\le1\) and
\[
    \sigma\le \frac{\beta\Phi(C)}{3},
\]
we have
\[
    \sigma
    \le
    \frac{\Phi(C)}{3}
    =
    \frac{\vol(\partial C)}{3\vol(C)}.
\]
Therefore
\[
    \sigma\frac{\vol(C)}{\vol(\partial C)}
    \le
    \frac13.
\]
Substituting this into the previous lower bound yields
\[
\begin{aligned}
    D(\lambda\mathbf 1_C)
    &\ge
    \left(
        \frac32-\frac16
    \right)
    \frac{\vol(C)^2}{\vol(\partial C)}  \\
    &=
    \frac43
    \frac{\vol(C)^2}{\vol(\partial C)}.
\end{aligned}
\]
In particular,
\[
    D(\lambda\mathbf 1_C)
    \ge
    \frac{\vol(C)^2}{2\vol(\partial C)}
    =
    D_C.
\]
Since \(x^\star\) maximizes \(D\) over \(x\ge0\), we conclude
\[
    D(x^\star)\ge D(\lambda\mathbf 1_C)\ge D_C.
\]
\end{proof}

\Robust*

\begin{proof}
Let
\[
    D(x):=-F(x)
\]
denote the HFD dual objective. Since
\[
    F(x)-F(x^\star)\le \varepsilon,
\]
we equivalently have
\[
    D(x^\star)-D(x)\le \varepsilon.
\]

By the auxiliary Lemma \ref{lem:hfd-dual-lower}, 
\[
    D(x^\star)\ge D_C.
\]
Therefore
\[
    D(x)\ge D(x^\star)-\varepsilon\ge D_C-\varepsilon.
\]
Since \(\varepsilon<D_C\), this implies
\[
    D(x)>0.
\]

Define
\[
    a:=\langle \Delta-d,x\rangle,
    \qquad
    \nu:=\sum_{e\in E}\vartheta_e f_e(x)^2.
\]
By definition of \(D\),
\[
    D(x)
    =
    a-\frac12\nu-\frac{\sigma}{2}x^\top D x.
\]
Since \(x^\top D x\ge0\), we have
\[
    D(x)\le a-\frac12\nu.
\]
Hence
\[
    a-\frac12\nu\ge D(x)>0.
\]
In particular, \(a>0\).

By \(K\)-sweep locality,
\[
    \sum_{e\in E}\vartheta_e f_e(x)\ell_b(e)
    \le
    K\sqrt{\nu}.
\]
Dividing by \(a>0\), we get
\[
    \frac{\sum_{e\in E}\vartheta_e f_e(x)\ell_b(e)}{a}
    \le
    K\frac{\sqrt{\nu}}{a}.
\]
We now bound \(\sqrt{\nu}/a\). Since \(a-\nu/2>0\),
\[
    \frac{\sqrt{\nu}}{a}
    \le
    \frac{1}{\sqrt{2(a-\nu/2)}}.
\]
Indeed, after squaring both sides, this inequality is equivalent to
\[
    2\nu(a-\nu/2)\le a^2,
\]
which is equivalent to
\[
    (a-\nu)^2\ge0.
\]
Using \(a-\nu/2\ge D(x)\), we obtain
\[
    \frac{\sqrt{\nu}}{a}
    \le
    \frac{1}{\sqrt{2D(x)}}.
\]
Therefore
\begin{equation}\label{eq:one}
    \frac{\sum_{e\in E}\vartheta_e f_e(x)\ell_b(e)}{a}
    \le
    \frac{K}{\sqrt{2D(x)}}.
\end{equation}

For \(h>0\), define the sweep set
\[
    S_h:=\{v\in V:x_v\ge h\}.
\]
Also define
\[
    q(h):=\Delta(S_h)-\vol(S_h),
\]
and
\[
    r(h):=
    \sum_{e\in E}\vartheta_e w_e(S_h\cap e)\ell_b(e).
\]
For nonnegative \(x\), the standard layer-cake identities give
\[
    a
    =
    \int_0^\infty q(h)\,dh,
\]
and the Lovász coarea formula gives
\[
    \sum_{e\in E}\vartheta_e f_e(x)\ell_b(e)
    =
    \int_0^\infty r(h)\,dh.
\]

Let
\[
    H_+:=\{h>0:q(h)>0\}.
\]
Since
\[
    a=\int_0^\infty q(h)\,dh>0,
\]
the set \(H_+\) is nonempty. Moreover, because \(q(h)\le0\) outside
\(H_+\),
\[
    \int_{H_+}q(h)\,dh\ge a.
\]
Also \(r(h)\ge0\), so
\[
    \int_{H_+}r(h)\,dh
    \le
    \int_0^\infty r(h)\,dh.
\]
Thus, 
\[
    \frac{\int_{H_+}r(h)\,dh}{\int_{H_+}q(h)\,dh}
    \le
    \frac{\int_0^\infty r(h)\,dh}{a}.
\]
By averaging, there exists some \(h\in H_+\) such that
\[
    \frac{r(h)}{q(h)}
    \le
    \frac{\int_{H_+}r(h)\,dh}{\int_{H_+}q(h)\,dh}
    \le
    \frac{\int_0^\infty r(h)\,dh}{a}.
\]
Using \eqref{eq:one}, 
\begin{equation}\label{eq:two}
    \frac{r(h)}{q(h)}
    \le
    \frac{K}{\sqrt{2D(x)}}.
\end{equation}

Since \(h\in H_+\), we have
\[
    q(h)>0.
\]
By definition of \(\ell_b(e)\),
\[
    \ell_b(e)\ge \frac{1}{\sqrt{\vol(C)}}.
\]
Therefore
\[
\begin{aligned}
    r(h)
    &=
    \sum_{e\in E}\vartheta_e w_e(S_h\cap e)\ell_b(e)  \\
    &\ge
    \frac{1}{\sqrt{\vol(C)}}
    \sum_{e\in E}\vartheta_e w_e(S_h\cap e)  \\
    &=
    \frac{\vol(\partial S_h)}{\sqrt{\vol(C)}}.
\end{aligned}
\]
On the other hand,
\[
    q(h)
    =
    \Delta(S_h)-\vol(S_h)
    \le
    \Delta(S_h).
\]
Since \(\Delta(v)=\delta d_v\) for \(v\in S\) and \(\Delta(v)=0\) otherwise,
\[
    \Delta(S_h)
    =
    \delta\vol(S_h\cap S)
    \le
    \delta\vol(S_h).
\]
Using \(\delta=3/\alpha\), we get
\[
    q(h)\le \frac3\alpha\vol(S_h).
\]

Combining the previous inequalities with equation \ref{eq:two} gives,
\[
    \frac{\vol(\partial S_h)}{\sqrt{\vol(C)}}
    \le
    r(h)
    \le
    \frac{K}{\sqrt{2D(x)}}q(h)
    \le
    \frac{K}{\sqrt{2D(x)}}\cdot \frac3\alpha\vol(S_h).
\]
Hence
\begin{equation}\label{eq:three}
    \frac{\vol(\partial S_h)}{\vol(S_h)}
    \le
    \frac{3K}{\alpha}
    \sqrt{\frac{\vol(C)}{2D(x)}}.
\end{equation}

By Assumption~\ref{ass:volume},
\[
    \vol(C)\le \vol(V\setminus C),
\]
so
\[
    \Phi(C)
    =
    \frac{\vol(\partial C)}{\vol(C)}.
\]
Therefore
\[
    D_C
    =
    \frac{\vol(C)^2}{2\vol(\partial C)}
    =
    \frac{\vol(C)}{2\Phi(C)}.
\]
Equivalently,
\[
    \vol(C)=2\Phi(C)D_C.
\]
Substituting this identity into equation \ref{eq:three} gives
\[
    \frac{\vol(\partial S_h)}{\vol(S_h)}
    \le
    \frac{3K}{\alpha}
    \sqrt{\Phi(C)}
    \sqrt{\frac{D_C}{D(x)}}.
\]

Finally,
\[
    D(x)
    \ge
    D(x^\star)-\varepsilon
    \ge
    D_C-\varepsilon
    =
    D_C\left(1-\frac{\varepsilon}{D_C}\right).
\]
Thus
\[
    \sqrt{\frac{D_C}{D(x)}}
    \le
    \frac{1}{\sqrt{1-\varepsilon/D_C}}.
\]
Consequently,
\[
    \frac{\vol(\partial S_h)}{\vol(S_h)}
    \le
    \frac{3K}{\alpha}\sqrt{\Phi(C)}
    \cdot
    \frac{1}{\sqrt{1-\varepsilon/D_C}}.
\]

If additionally
\[
    \vol(S_h)\le \vol(V\setminus S_h),
\]
then
\[
    \Phi(S_h)
    =
    \frac{\vol(\partial S_h)}
    {\min\{\vol(S_h),\vol(V\setminus S_h)\}}
    =
    \frac{\vol(\partial S_h)}{\vol(S_h)}.
\]
Hence the same bound holds for \(\Phi(S_h)\).
\end{proof}

\newpage
\subsection{Proof of Corollary \ref{cor:early-stopping}: Early stopping clustering guarantee for TL-HFD}\label{app:early_stopping}

\begin{Auxiliary Lemma}[{Support-volume sufficient condition for sweep locality}]\label{lem:support-volume-sweep-locality}
    Let \(x\ge0\), and let $A_x:=\supp(x)=\{v\in \V:x_v>0\}$. Then \(x\) satisfies \(\mathbb{K}\)-sweep locality with $\mathbb{K} = 1+\sqrt{\frac{\vol(A_x)}{\vol(C)}}$. In particular, if $\vol(A_x) \le  \left(\frac4\beta-1\right)^2\vol(C)$, then \(x\) satisfies the original HFD sweep-locality condition $\sum_{e\in \E}\vartheta_e f_e(x)\ell_b(e) \le \frac4\beta\sqrt{\nu(x)}$.
\end{Auxiliary Lemma}

\begin{proof}
If \(\nu(x)=0\), then \(f_e(x)=0\) for every \(e\) with \(\vartheta_e>0\).
Therefore
\[
    \sum_{e\in E}\vartheta_e f_e(x)\ell_b(e)=0
    =
    \left(
        1+\sqrt{\frac{\vol(A_x)}{\vol(C)}}
    \right)\sqrt{\nu(x)}.
\]
Thus the claim is trivial in this case. We may henceforth assume
\(\nu(x)>0\).

Let
\[
    E_x:=\{e\in \E:f_e(x)>0\}.
\]
If \(e\cap A_x=\emptyset\), then \(x\) is identically zero on \(e\), so
\(x|_e\equiv 0\). Since the cut-costs are normalized, \(0\in B_e\), and hence
\(f_e(0)=0\). Therefore \(f_e(x)=0\) whenever \(e\cap A_x=\emptyset\). Thus
every hyperedge in \(E_x\) intersects \(A_x\).

Consequently,
\[
    \sum_{e\in E_x}\vartheta_e
    \le
    \sum_{e:e\cap A_x\neq\emptyset}\vartheta_e.
\]
Each hyperedge \(e\) with \(e\cap A_x\neq\emptyset\) is counted at least once
in
\[
    \sum_{v\in A_x}\sum_{e\ni v}\vartheta_e.
\]
Therefore
\[
    \sum_{e\in E_x}\vartheta_e
    \le
    \sum_{v\in A_x}\sum_{e\ni v}\vartheta_e
    =
    \sum_{v\in A_x}d_v
    =
    \vol(A_x).
\]

Next, since
\[
    \ell_b(e)
    =
    \max\left\{
        \frac{1}{\sqrt{\vol(C)}},
        \frac{f_e(x)}{\sqrt{\nu(x)}}
    \right\},
\]
we have the elementary upper bound
\[
    \ell_b(e)
    \le
    \frac{1}{\sqrt{\vol(C)}}
    +
    \frac{f_e(x)}{\sqrt{\nu(x)}}.
\]
Thus
\[
\begin{aligned}
    \sum_{e\in E}\vartheta_e f_e(x)\ell_b(e)
    &\le
    \frac{1}{\sqrt{\vol(C)}}
    \sum_{e\in E}\vartheta_e f_e(x)
    +
    \frac{1}{\sqrt{\nu(x)}}
    \sum_{e\in E}\vartheta_e f_e(x)^2.
\end{aligned}
\]
The second term is exactly
\[
    \frac{1}{\sqrt{\nu(x)}}
    \sum_{e\in E}\vartheta_e f_e(x)^2
    =
    \sqrt{\nu(x)}.
\]
For the first term, since \(f_e(x)=0\) outside \(E_x\),
\[
    \sum_{e\in E}\vartheta_e f_e(x)
    =
    \sum_{e\in E_x}\vartheta_e f_e(x).
\]
By Cauchy--Schwarz,
\[
\begin{aligned}
    \sum_{e\in E_x}\vartheta_e f_e(x)
    &\le
    \left(\sum_{e\in E_x}\vartheta_e\right)^{1/2}
    \left(\sum_{e\in E_x}\vartheta_e f_e(x)^2\right)^{1/2}  \\
    &\le
    \sqrt{\vol(A_x)}\sqrt{\nu(x)}.
\end{aligned}
\]
Substituting this bound gives
\[
\begin{aligned}
    \sum_{e\in E}\vartheta_e f_e(x)\ell_b(e)
    &\le
    \frac{\sqrt{\vol(A_x)}}{\sqrt{\vol(C)}}\sqrt{\nu(x)}
    +
    \sqrt{\nu(x)} \\
    &=
    \left(
        1+\sqrt{\frac{\vol(A_x)}{\vol(C)}}
    \right)\sqrt{\nu(x)}.
\end{aligned}
\]
Hence \(x\) satisfies \(\mathbb{K}\)-sweep locality with
\[
    \mathbb{K}
    =
    1+\sqrt{\frac{\vol(A_x)}{\vol(C)}}.
\]

Finally, if
\[
    \vol(A_x)
    \le
    \left(\frac4\beta-1\right)^2\vol(C),
\]
then
\[
    1+\sqrt{\frac{\vol(A_x)}{\vol(C)}}
    \le
    \frac4\beta.
\]
Therefore
\[
    \sum_{e\in E}\vartheta_e f_e(x)\ell_b(e)
    \le
    \frac4\beta\sqrt{\nu(x)},
\]
which is the original HFD sweep-locality condition.
\end{proof}

\earlystopping*

\begin{proof}

By Corollary~\ref{cor:dual_gap_threshold}, given in Appendix ~\ref{app:dual_gap_threshold}, with the algorithm-scope
clarification at the start of its proof, the
$\varepsilon_{\mathrm{vol}}(T)$ bound is a formal guarantee for the
box-constrained Algorithm~\ref{alg:app-pq-psgd-theory}. The same bound also
applies to any run of the unconstrained Algorithm~\ref{alg:pq-psgd-theory} whose
iterates remain below the box ceiling, since then the two algorithms
coincide; see Remark~\ref{rem:box-nonbinding}.

By Corollary~\ref{cor:dual_gap_threshold}, under the box-constrained
Proposition~\ref{prop:Gloc-bound}, the best iterate satisfies
\[
    F(\hat x^{(T)})-F(x^\star)
    \le
    \varepsilon_{\rm vol}(T).
\]
Since \(\varepsilon_{\rm vol}(T)<D_C\) by assumption, the suboptimality
requirement in Theorem~\ref{thm:robust-sweep} is satisfied with
\(x=\hat x^{(T)}\) and \(\varepsilon=\varepsilon_{\rm vol}(T)\).

It remains to verify the \(K\)-sweep-locality condition. Let
\(A_{\hat x}:=\supp(\hat x^{(T)})\). By definition,
\[
    \zeta_T
    =
    \frac{\vol(A_{\hat x})}{\vol(C)}.
\]
Therefore, by Auxiliary Lemma~\ref{lem:support-volume-sweep-locality},
\(\hat x^{(T)}\) satisfies \(K_T\)-sweep locality with
\[
    K_T
    =
    1+\sqrt{\frac{\vol(A_{\hat x})}{\vol(C)}}
    =
    1+\sqrt{\zeta_T}.
\]
Applying Theorem~\ref{thm:robust-sweep} with
\(x=\hat x^{(T)}\), \(\varepsilon=\varepsilon_{\rm vol}(T)\), and
\(K=K_T\), there exists a threshold \(h>0\) such that
\(S_h:=\{v:\hat x^{(T)}_v\ge h\}\) satisfies
\[
    \frac{\vol(\partial S_h)}{\vol(S_h)}
    \le
    \frac{3K_T}{\alpha}\sqrt{\Phi(C)}
    \cdot
    \frac{1}{\sqrt{1-\varepsilon_{\rm vol}(T)/D_C}}.
\]
Substituting \(K_T=1+\sqrt{\zeta_T}\) gives
\[
    \frac{\vol(\partial S_h)}{\vol(S_h)}
    \le
    \frac{3}{\alpha}(1+\sqrt{\zeta_T})\sqrt{\Phi(C)}
    \cdot
    \frac{1}{\sqrt{1-\varepsilon_{\rm vol}(T)/D_C}}.
\]

If additionally \(\vol(S_h)\le \vol(V\setminus S_h)\), then
\[
    \Phi(S_h)
    =
    \frac{\vol(\partial S_h)}
    {\min\{\vol(S_h),\vol(V\setminus S_h)\}}
    =
    \frac{\vol(\partial S_h)}{\vol(S_h)}.
\]
Hence the same bound holds for \(\Phi(S_h)\).
\end{proof}

\newpage


\section{Explored-Volume Bound (Theorem \ref{thm:explored_volume_topk})}
\label{app:stronglocal}

This section provides the explicit local ingredients needed for the
\(B\)-based truncation bound and for the activated-volume certificate. In
particular, Proposition~\ref{prop:Gloc-bound} gives an explicit bound on the
restricted local subgradient norm under the box-constrained variant, and
Theorem~\ref{thm:explored_volume_topk} certifies the activated volume during the run.

\subsection{Bounding the optimal solution}\label{app:localized}
 
\begin{Auxiliary Lemma}[Optimal solution bound]\label{lem:xstar-bound}
Let \(x^\star\in\arg\min_{x\ge 0}F(x)\) with \(\delta>1\) and \(\sigma>0\).
Then
\[
    \|x^\star\|_\infty \le \frac{\delta-1}{\sigma}.
\]

\end{Auxiliary Lemma}

\begin{proof}
Let
\[
    m:=\|x^\star\|_\infty,
    \qquad
    U:=\{v\in V:x_v^\star=m\}.
\]
If \(m=0\), then the claim is immediate. Assume henceforth that \(m>0\).

Since \(x^\star\) minimizes \(F\) over \(\mathbb R_+^{|\V|}\), the KKT conditions
imply that there exists a subgradient \(g^\star\in\partial F(x^\star)\) such
that
\[
    g_v^\star=0
    \qquad
    \forall v\in U,
\]
because every \(v\in U\) satisfies \(x_v^\star=m>0\).

By the subgradient form of Lemma~\ref{lem:subgrad_in_F},
\[
    g^\star
    =
    \sum_{e\in E}\vartheta_e f_e(x^\star)\rho_e(x^\star)
    +\sigma D x^\star
    -(\Delta-d),
\]
for some choice \(\rho_e(x^\star)\in \partial f_e(x^\star)\) for each
\(e\in E\). Summing over \(U\), we obtain
\[
\begin{aligned}
    0
    &=
    \sum_{v\in U} g_v^\star \\
    &=
    \sum_{e\in E}\vartheta_e f_e(x^\star)
    \sum_{v\in U\cap e}[\rho_e(x^\star)]_v
    +\sigma \sum_{v\in U} d_v x_v^\star
    -\Delta(U)+\vol(U).
\end{aligned}
\]
Since \(x_v^\star=m\) for all \(v\in U\),
\[
    \sum_{v\in U} d_v x_v^\star
    =
    m\,\vol(U).
\]
Hence
\[
    0
    =
    \sum_{e\in E}\vartheta_e f_e(x^\star)\rho_e(x^\star)(U\cap e)
    +\sigma m\,\vol(U)
    -\Delta(U)+\vol(U),
\]
where
\[
    \rho_e(x^\star)(U\cap e)
    :=
    \sum_{v\in U\cap e}[\rho_e(x^\star)]_v.
\]

We now show that each edge contribution is nonnegative. Fix \(e\in E\), and
let \(U_e:=U\cap e\). Since every vertex in \(U_e\) attains the maximum value
of \(x^\star\) on \(e\), \(U_e\) is precisely the top level-set block of
\(x^\star|_e\). By the standard characterization of Lov\'asz subgradients,
every \(\rho_e(x^\star)\in \partial f_e(x^\star)\) is a convex combination of
greedy extreme subgradients associated with permutations consistent with the
ordering of \(x^\star|_e\). For each such greedy extreme subgradient, the total
mass on the first block \(U_e\) is
\[
    w_e(U_e)-w_e(\varnothing)=w_e(U_e).
\]
Therefore, for every \(\rho_e(x^\star)\in \partial f_e(x^\star)\),
\[
    \rho_e(x^\star)(U_e)=w_e(U_e)\ge 0,
\]
using the normalization \(w_e(\varnothing)=0\) and nonnegativity of cut costs.
Since also \(f_e(x^\star)\ge 0\), every summand
\[
    \vartheta_e f_e(x^\star)\rho_e(x^\star)(U_e)
\]
is nonnegative. Thus
\[
    \sum_{e\in E}\vartheta_e f_e(x^\star)\rho_e(x^\star)(U\cap e)\ge 0.
\]

It follows that
\[
    0
    \ge
    \sigma m\,\vol(U)-\Delta(U)+\vol(U),
\]
or equivalently
\[
    \sigma m\,\vol(U)
    \le
    \Delta(U)-\vol(U).
\]
Now
\[
    \Delta(U)=\delta\,\vol(U\cap S),
\]
so
\[
    \Delta(U)-\vol(U)
    =
    \delta\,\vol(U\cap S)-\vol(U)
    \le
    (\delta-1)\vol(U),
\]
because \(\vol(U\cap S)\le \vol(U)\). Therefore
\[
    \sigma m\,\vol(U)\le (\delta-1)\vol(U).
\]
Since \(U\neq\varnothing\), we have \(\vol(U)>0\), and hence
\[
    m\le \frac{\delta-1}{\sigma}.
\]
As \(m=\|x^\star\|_\infty\), the claim follows.
\end{proof}

\subsection{Box-constrained variant for explicit constants}\label{app:box-constraints}
The localized convergence theorem does not require bounded iterates. However,
to obtain the explicit coordinatewise bound used in
Proposition~\ref{prop:Gloc-bound} and the \(B\)-based forms of
Corollaries~\ref{cor:dual_gap_threshold} and~\ref{cor:early-stopping}, we
introduce a box-constrained theoretical variant. By
Auxiliary Lemma~\ref{lem:xstar-bound}, the optimizer already satisfies
\(\|x^\star\|_\infty\le (\delta-1)/\sigma\), so this clipping does not change
the optimum; it is used only to keep intermediate iterates uniformly bounded.

Define $\bar{x} := (\delta - 1)/\sigma$ and the box
\[
  \mathcal{X} := \bigl\{x \in \mathbb{R}^\V : 0 \le x_v \le \bar{x}\;\;\forall\, v \in V\bigr\}.
\]
By Auxiliary Lemma~\ref{lem:xstar-bound}, $x^\star \in \mathcal{X}$, so $\arg\min_{x \in \mathcal{X}} F(x) = x^\star$.
We modify the algorithm \ref{alg:pq-psgd-theory} by replacing the projection $\Pi_{\mathbb{R}^V_+}$ with the coordinatewise box projection $[\Pi_{\mathcal{X}}(z)]_v = \min(\max(z_v, 0),\, \bar{x})$ as given in Algorithm \ref{alg:app-pq-psgd-theory}.
 
\begin{Auxiliary Lemma}[Box projection preserves locality]\label{lem:box-locality}
  The box-constrained variant preserves all properties stated in Lemma~\ref{lem:locality}.
  In particular, for every $u \notin A^{(t)} \cup \partial A^{(t)}$,
  \[
    [\Pi_{\mathcal{X}}(x^{(t)} - \eta\, D^{-1} g^{(t)})]_u = 0.
  \]
  Moreover, $\|x^{(t)}\|_\infty \le \bar{x}$ for all $t \ge 0$.
\end{Auxiliary Lemma}
 
\begin{proof}
  For $u \notin A^{(t)} \cup \partial A^{(t)}$: $x_u^{(t)} = 0$ and $g_u = d_u > 0$ (Lemma~\ref{lem:locality}), so $y_u = -\eta < 0$ and $\Pi_{\mathcal{X}}(y_u) = 0$.
  The iterate bound $\|x^{(t)}\|_\infty \le \bar{x}$ follows by construction.
  Nonexpansiveness of $\Pi_{\mathcal{X}}$ in $\|\cdot\|_D$ holds because $\mathcal{X}$ is a convex set and the projection is coordinatewise with $D$ diagonal.
\end{proof}

\begin{remark}
For the explicit local subgradient bound, we consider the box-constrained variant obtained by replacing projection onto \(\mathbb R_+^{|\V|}\) with projection onto \(\mathcal X=\{x:0\le x_v\le (\delta-1)/\sigma\}\). This does not change the optimizer, since \(\|x^\star\|_\infty\le(\delta-1)/\sigma\).
\end{remark}

\begin{algorithm}[H]
\caption{Thresholded Local Hyper-Flow Diffusion (TL-HFD, box-constrained theoretical variant)}
\label{alg:app-pq-psgd-theory}
\begin{algorithmic}[1]
\Require Current iterate \(x^{(t)}\), seed set \(S\),
         step size \(\eta_{t+1}\), activation count \(k \ge 0\),
         box radius \(\bar x=(\delta-1)/\sigma\)
\Ensure Next iterate \(x^{(t+1)}\)

\Statex \textit{\# Active region and its one-hop boundary.}
\State \(A^{(t)} \gets \operatorname{supp}(x^{(t)}) \cup S\)
\State \(\partial A^{(t)} \gets
       \{u \notin A^{(t)} : \exists\, e \in E,\;
       u \in e,\; e \cap A^{(t)} \neq \emptyset\}\)

\Statex \textit{\# Compute subgradient only on the local region (Lemma~\ref{lem:locality}).}
\State \(g^{(t)} \gets g(x^{(t)})\)
       restricted to \(A^{(t)} \cup \partial A^{(t)}\)
\State \(x^{(t+1)} \gets x^{(t)}\)

\Statex \textit{\# Box-projected PSGD update on active vertices.}
\ForAll{\(u \in A^{(t)}\)}
  \State \(x^{(t+1)}_u \gets
       \min\left\{\bar x,\;
       \max\left\{0,\;
       x^{(t)}_u-\eta_{t+1}[g^{(t)}]_u/d_u
       \right\}
       \right\}\)
\EndFor

\Statex \textit{\# Scoring for boundary vertices.}
\ForAll{\(u \in \partial A^{(t)}\)}
  \State \(\kappa^{(t)}(u) \gets \max\left\{0,\; -[g^{(t)}]_u / d_u\right\}\)
  \State \(c^{(t)}(u) \gets \left(d_{\mathrm{in}}^{(t)}(u)/d(u)\right)^\gamma\)
  \State \(s^{(t)}(u) \gets \kappa^{(t)}(u)\, c^{(t)}(u)\)
\EndFor

\Statex \textit{\# Activate the top-\(k\) boundary vertices.}
\State \(K_t \gets \textsc{TopK}(s^{(t)},\, k)\)
\ForAll{\(u \in K_t\)}
  \State \(x_u^{(t+1)} \gets
       \min\left\{\bar x,\;
       \eta_{t+1}\kappa^{(t)}(u)
       \right\}\)
\EndFor

\State \Return \(x^{(t+1)}\)
\end{algorithmic}
\end{algorithm}

\begin{remark}[Box constraint is non-binding in practice]
\label{rem:box-nonbinding}
Although Proposition~\ref{prop:Gloc-bound} is stated for the
box-constrained variant, Algorithm~\ref{alg:app-pq-psgd-theory}, our
experiments use the unconstrained Algorithm~\ref{alg:pq-psgd-theory}.
The two algorithms produce identical iterates whenever no coordinate of
the unconstrained iterate reaches the ceiling
$\bar{x}=(\delta-1)/\sigma$. This is the case in our experiments.For example, for Amazon runs with $\delta_{\exp}=200$ and
$\sigma=10^{-4}$, the seed-specific theoretical coefficient is
$\delta(s^\star)
=\delta_{\exp}\operatorname{vol}(T)/d_{s^\star}$, so
$\bar{x}=((\delta(s^\star)-1)/\sigma)
\ge 1.99\times 10^6$. The observed coordinates of $x^{(t)}$
remain below $10^2$ throughout all runs. Hence the box
projection is never active on these instances, and the iterates of
Algorithm~\ref{alg:pq-psgd-theory} and
Algorithm~\ref{alg:app-pq-psgd-theory} coincide.
The box constraint therefore plays a theoretical role: it yields the
uniform coordinatewise bound $|g_v^{(t)}/d_v|\le B_{\mathrm{loc}}$
established in Proposition~\ref{prop:Gloc-bound} and used in the
convergence corollaries, without affecting the empirical behavior of
the algorithm. Moreover, the box projection $\Pi_{\mathcal X}$ is coordinatewise and 1-Lipschitz in the $D$-norm (Auxiliary Lemma~\ref{lem:box-locality}). For skipped boundary coordinates, the box-projected truncation error is at most $\eta_{t+1}\kappa^{(t)}(u)$, so the same truncation-error upper bound used in Lemma~\ref{lem:threshold-inexact} applies to Algorithm~\ref{alg:app-pq-psgd-theory}.
\end{remark}

\subsection{Explicit bound on the local subgradient norm}
 
\begin{proposition}[Explicit local subgradient bound]\label{prop:Gloc-bound}
  Under the box-constrained algorithm with $\|x^{(t)}\|_\infty \le \bar{x} = (\delta - 1)/\sigma$, $\delta \geq 2$,  and the cut-cost normalization $0 \le w_e(S) \le 1$, for every iteration~$t$:
  \[
    \|g^{(t)}\|^2_{D^{-1},\, L_t}
    \;\le\;
    B^2 \cdot \operatorname{vol}(L_t),
    \qquad
    B \;:=\; (\delta - 1)\!\left(\frac{1}{\sigma} + 2\right).
  \]
\end{proposition}
 
\begin{proof}
  It suffices to show $|g_v / d_v| \le B$ for all $v \in L_t$, since this implies $g_v^2/d_v \le B^2\, d_v$.
 
  \paragraph{Bounding the Lov\'{a}sz contribution.}
  We first establish that for every $v \in L_t$,
  \begin{equation}\label{eq:lov-bound}
    \frac{1}{d_v}\Biggl|\sum_{e \ni v} \vartheta_e\, f_e(x^{(t)})\, [\rho_e(x^{(t)})]_v\Biggr|
    \;\le\; \bar{x}.
  \end{equation}
  Since $\rho_e \in B_e$ and $0 \le w_e(S) \le 1$, for any single vertex $v \in e$ we have
  \[
    [\rho_e]_v = \langle \rho_e, \mathbf{1}_{\{v\}} \rangle \le w_e(\{v\}) \le 1,
  \]
  and
  \[
    [\rho_e]_v = \langle \rho_e, \mathbf{1}_e \rangle - \langle \rho_e, \mathbf{1}_{e \setminus \{v\}} \rangle = 0 - \langle \rho_e, \mathbf{1}_{e \setminus \{v\}} \rangle \ge -w_e(e \setminus \{v\}) \ge -1,
  \]
  so $|[\rho_e]_v| \le 1$.  Since $f_e(x) \ge 0$, this gives
  $\bigl|\sum_{e \ni v} \vartheta_e\, f_e(x)\, [\rho_e]_v\bigr| \le \sum_{e \ni v} \vartheta_e\, f_e(x)$.
 
  For $x \ge 0$ with $\|x\|_\infty \le \bar{x}$: since $\langle \rho, \mathbf{1}_S \rangle \le 1$ for all $S \subseteq e$, the positive components of any $\rho \in B_e$ sum to at most~$1$.
  Therefore $f_e(x) = \max_{\rho \in B_e} \langle \rho, x \rangle \le \|x\|_\infty \le \bar{x}$, and
  \[
    \sum_{e \ni v} \vartheta_e\, f_e(x) \;\le\; \bar{x} \sum_{e \ni v} \vartheta_e \;=\; \bar{x}\, d_v,
  \]
  establishing~\eqref{eq:lov-bound}.
 
  \paragraph{Bounding $|g_v/d_v|$ by cases.}
  Recall that
  $g_v/d_v = [\text{Lov}]_v / d_v + \sigma\, x_v - (\Delta(v) - d_v)/d_v$,
  where $[\text{Lov}]_v := \sum_{e \ni v} \vartheta_e\, f_e(x)\, [\rho_e]_v$.
 
  \medskip
  \noindent\emph{Case (a): $v \in S \cap A^{(t)}$ \textup{(}seed vertex\textup{)}.}
  Here $\Delta(v) = \delta\, d_v$, so $(\Delta(v) - d_v)/d_v = \delta - 1$.  Using~\eqref{eq:lov-bound} and $x_v \le \bar{x}$:
  \[
    \left|\frac{g_v}{d_v}\right|
    \;\le\;
    \bar{x} + \sigma\, \bar{x} + (\delta - 1)
    \;=\;
    (1 + \sigma)\,\frac{\delta - 1}{\sigma} + (\delta - 1)
    \;=\;
    (\delta - 1)\!\left(\frac{1}{\sigma} + 2\right)
    = B.
  \]
 
  \medskip
  \noindent\emph{Case (b): $v \in A^{(t)} \setminus S$ \textup{(}active, non-seed\textup{)}.}
  Here $\Delta(v) = 0$, so $(\Delta(v) - d_v)/d_v = -1$.
  \[
    \left|\frac{g_v}{d_v}\right|
    \;\le\;
    \bar{x} + \sigma\, \bar{x} + 1
    \;=\;
    \frac{\delta - 1}{\sigma} + (\delta - 1) + 1
    \;=\;
    \frac{\delta - 1}{\sigma} + \delta.
  \]
 
  \medskip
  \noindent\emph{Case (c): $v \in \partial A^{(t)}$ \textup{(}boundary\textup{)}.}
  Here $x_v = 0$ and $\Delta(v) = 0$.
  \[
    \left|\frac{g_v}{d_v}\right|
    \;\le\;
    \bar{x} + 0 + 1
    \;=\;
    \frac{\delta - 1}{\sigma} + 1.
  \]
 
  For $\delta \ge 2$ (guaranteed since $\delta = 3/\alpha \ge 3$), Case~(a) dominates:
  \begin{itemize}
    \item Case~(a) $\ge$ Case~(b): $\;(\delta-1)(1/\sigma + 2) \ge (\delta-1)/\sigma + \delta\;$ iff $\;2(\delta-1) \ge \delta\;$ iff $\;\delta \ge 2$. 
    \item Case~(a) $\ge$ Case~(c): $\;(\delta-1)(1/\sigma + 2) \ge (\delta-1)/\sigma + 1\;$ iff $\;2(\delta-1) \ge 1$. 
  \end{itemize}
   Thus $|g_v/d_v| \le B$ for all $v \in L_t$, and therefore
  \[
    \|g^{(t)}\|^2_{D^{-1},\, L_t}
    = \sum_{v \in L_t} \frac{g_v^2}{d_v}
    \le B^2 \sum_{v \in L_t} d_v
    = B^2 \cdot \operatorname{vol}(L_t).\qedhere
  \]
\end{proof}

\subsection{Localized explored-volume bound}\label{strong_local_app}

\exploredVolume*

\begin{proof}
Fix an iteration \(t\). By construction of
Algorithm~\ref{alg:pq-psgd-theory}, \(x^{(0)}=0\) and let  
$ J_t:=\{u\in \partial A^{(t)}:u\notin M_t,\ x_u^{(t+1)}>0\}
$ be the set of vertices activated for the first time at iteration \(t\), where \(M_{t+1}:=M_t\cup J_t\). As every first-time activated vertex is
selected from the current boundary, we get
\[
    J_t\subseteq \partial A^{(t)}\subseteq L_t.
\]
By hypothesis,
\[
    \kappa^{(t)}(u)\ge \rho_t
    \qquad \forall u\in J_t.
\]
Therefore
\[
    \rho_t^2\vol(J_t)
    =
    \rho_t^2\sum_{u\in J_t}d_u
    \le
    \sum_{u\in J_t}d_u\bigl(\kappa^{(t)}(u)\bigr)^2.
\]
Since
\[
    \kappa^{(t)}(u)
    =
    \max\left\{0,-\frac{g_u^{(t)}}{d_u}\right\},
\]
we have, for every \(u\),
\[
    d_u\bigl(\kappa^{(t)}(u)\bigr)^2
    \le
    \frac{(g_u^{(t)})^2}{d_u}.
\]
Thus, using \(J_t\subseteq L_t\),
\[
    \rho_t^2\vol(J_t)
    \le
    \sum_{u\in J_t}\frac{(g_u^{(t)})^2}{d_u}
    \le
    \sum_{v\in L_t}\frac{(g_v^{(t)})^2}{d_v}
    =
    \|g^{(t)}\|_{D^{-1},L_t}^2.
\]
Hence
\[
    \vol(J_t)
    \le
    \frac{\|g^{(t)}\|_{D^{-1},L_t}^2}{\rho_t^2}.
\]

The sets \(J_t\) are pairwise disjoint because each vertex can be
activated for the first time in at most one iteration. Moreover,
\(S\cap J_t=\varnothing\) for all \(t\), since \(S\subseteq A^{(0)}\). Since $\vol(J_t)=0$ for $t \notin I_t$ we restrict the sum to $t \in I_T$, where  $\rho_t > 0$. Therefore
\[
    \vol(M_T)
    =
    \vol(S)
    +
    \sum_{t\in I_T}\vol(J_t).
\]
Substituting the preceding bound on \(\vol(J_t)\) gives
\[
    \vol(M_T)
    \le
    \vol(S)
    +
    \sum_{t\in I_T}
    \frac{\|g^{(t)}\|_{D^{-1},L_t}^2}{\rho_t^2}.
\]
If additionally
\[
    \|g^{(t)}\|_{D^{-1},L_t}\le G_{\mathrm{loc}}
    \qquad \forall t,
\]
then
\[
    \vol(M_T)
    \le
    \vol(S)
    +
    G_{\mathrm{loc}}^2
    \sum_{t\in I_T}\rho_t^{-2}.
\]
\end{proof}

\subsection{Known Facts: A technical fact about projections in the $\|\cdot\|_D$ norm }
\label{app:proj}

Equation \ref{eq:local-update} uses the coordinatewise Euclidean projection onto the nonnegative orthant,
$\Pi_{\mathbb{R}^V_+}(z)=\max\{z,0\}$ applied coordinatewise.
Because the orthant constraint is separable across coordinates, this coincides with the
projection in any diagonal weighted Euclidean norm (in particular, the $\|\cdot\|_D$ norm).

\begin{lemma}[Projection in $\|\cdot\|_D$ is coordinatewise clipping]
\label{lem:Dproj}
Let $D=\mathrm{diag}(d)$ with $d_u>0$ for all $u$. Define the $D$-projection onto $\mathbb{R}^V_+$ by
\[
\Pi^{D}_{\mathbb{R}^V_+}(z) \in \arg\min\{\|x-z\|_D^2:\ x\ge 0\}.
\]
Then $\Pi^{D}_{\mathbb{R}^V_+}(z) = \Pi_{\mathbb{R}^V_+}(z)$ coordinatewise, i.e.
$(\Pi^{D}_{\mathbb{R}^V_+}(z))_u = \max\{z_u,0\}$. Moreover, $\Pi_{\mathbb{R}^V_+}$ is nonexpansive in
$\|\cdot\|_D$:
\[
\|\Pi_{\mathbb{R}^V_+}(a) - \Pi_{\mathbb{R}^V_+}(b)\|_D \le \|a-b\|_D \qquad \forall a,b\in \mathbb{R}^V.
\]
\end{lemma}

\begin{proof}
Because $D$ is diagonal, the projection problem decouples coordinatewise:
\[
\min_{x\ge 0}\ \sum_{u\in V} d_u(x_u-z_u)^2 \;=\; \sum_{u\in V}\ \min_{x_u\ge 0}\ d_u(x_u-z_u)^2.
\]
Each scalar problem has the unique minimizer $x_u^\star=\max\{z_u,0\}$, hence
$\Pi^D_{\mathbb{R}^V_+}(z)=\Pi_{\mathbb{R}^V_+}(z)$.

For nonexpansiveness, fix $a,b$ and set $p=\Pi_{\mathbb{R}^V_+}(a)$, $q=\Pi_{\mathbb{R}^V_+}(b)$.
The scalar map $t\mapsto \max\{t,0\}$ is $1$-Lipschitz, so $|p_u-q_u|\le |a_u-b_u|$ for each $u$.
Multiplying by $d_u$ and summing yields
\[
\|p-q\|_D^2=\sum_{u} d_u(p_u-q_u)^2 \;\le\; \sum_{u} d_u(a_u-b_u)^2 = \|a-b\|_D^2.
\]
Taking square roots gives the claim.
\end{proof}

\subsection{Directional derivatives and Lov\'asz support functions (Known result - provided here for completeness) }
\label{app:dirder}

Recall $f_e(x)=\max_{\rho\in B_e}\langle \rho,x\rangle$ and $M_e(x)=\arg\max_{\rho\in B_e}\langle \rho,x\rangle$.

\begin{proposition}[Directional derivative formulas]
\label{prop:dirder}
Fix $x\in \mathbb{R}^V$ and a hyperedge $e$.
\begin{enumerate}
\item For any direction $h$,
\[
f'_e(x;h)=\max_{\rho\in M_e(x)} \langle \rho,h\rangle.
\]
\item Let $g_e(x)=\tfrac12 f_e(x)^2$. Then $g'_e(x;h)=f_e(x)\, f'_e(x;h)$.
\item Consequently, the directional derivative of $F$ in equation \eqref{eq:cut_HFD} satisfies
\[
F'(x;h)=\sum_{e\in E} \vartheta_e f_e(x)\max_{\rho\in M_e(x)}\langle \rho,h\rangle
+\sigma x^\top D h - \langle \Delta-d,h\rangle.
\]
\end{enumerate}
\end{proposition}

\begin{proof}
(1) Define $\phi(t):=f_e(x+th)=\max_{\rho\in B_e}\langle \rho,x+th\rangle$.
Because $B_e$ is compact and the objective is continuous linear in $\rho$, the max is attained.
Moreover $\phi$ is convex as a pointwise maximum of linear functions in $t$.

Lower bound: for any $\bar\rho\in M_e(x)$ and any $t>0$,
\[
f_e(x+th)\ge \langle \bar\rho,x+th\rangle = f_e(x) + t\langle \bar\rho,h\rangle,
\]
so $\frac{f_e(x+th)-f_e(x)}{t}\ge \langle \bar\rho,h\rangle$.
Taking $\max_{\bar\rho\in M_e(x)}$ and then $t\downarrow 0$ gives
$f'_e(x;h)\ge \max_{\rho\in M_e(x)}\langle \rho,h\rangle$.

Upper bound: for $t>0$ let $\rho_t\in M_e(x+th)$.
Then
\[
f_e(x+th)=\langle \rho_t,x+th\rangle = \langle \rho_t,x\rangle + t\langle \rho_t,h\rangle
\le f_e(x)+t\langle \rho_t,h\rangle,
\]
since $\langle \rho_t,x\rangle\le \max_{\rho\in B_e}\langle \rho,x\rangle=f_e(x)$.
Thus $\frac{f_e(x+th)-f_e(x)}{t}\le \langle \rho_t,h\rangle$.
As $t\downarrow 0$, compactness implies $\{\rho_t\}$ has limit points, and any limit point must lie in
$M_e(x)$ by continuity of $\rho\mapsto \langle \rho,x\rangle$.
Taking $\limsup_{t\downarrow 0}$ yields
$f'_e(x;h)\le \max_{\rho\in M_e(x)}\langle \rho,h\rangle$.
Combining proves (1).

(2) Let $\psi(t):=f_e(x+th)$. Since $f_e$ is convex and finite everywhere, $\psi$ is convex and hence
directionally differentiable at $0$.
The scalar map $s\mapsto \tfrac12 s^2$ is differentiable with derivative $s$.
Therefore, by the one-sided chain rule for directional derivatives,
\[
g'_e(x;h)=\left(\tfrac12 \psi(t)^2\right)'_{t=0+}=\psi(0)\psi'(0+)=f_e(x)\, f'_e(x;h).
\]

(3) $F$ is a finite sum of directionally differentiable functions, so the directional derivative is the sum of
directional derivatives. Apply (2) to each $e$, and use the standard derivative of the quadratic and linear terms.
\end{proof}
\newpage

\section{Additional Experimental Details}\label{app:exp-details}
 
\subsection{Real World Dataset Statistics}\label{app:data}
 Table \ref{tab:datasets} provides the statistics of the real-world datasets used for our empirical study. We now describe the datasets used together with the evaluation criteria.
\begin{table}[h]
\centering
\caption{Summary statistics of datasets used in the experiments.}
\label{tab:datasets}
\begin{tabular}{l r r r r l}
\toprule
Dataset & $|V|$ & $|E|$ & Mean/Median & Clusters & Cut-costs \\
&&&& (\#nodes)&\\
\midrule
Trivago-clicks \ref{sec:exp-trivago}, \ref{app:trivago-sensitivity}  & 172,738   & 233,202   & 5.6 / 2.0  & 10 (144--945)    & Unit, Card \\
Amazon-reviews \ref{sec:exp-amazon}  & 2,268,231 & 4,285,363 & 32.2 / 11.0 & 9 (31--5334)     & Unit \\
Florida Bay \ref{app:foodweb}  & 126 & 141,233 & 4,483.6/  3,770.5  & 3 (17--70 ) & Sub \\
High-school-contact \ref{app:highschool} & 327 & 7,818 & 55.6/ 53 & 9 (29--44 ) & Unit, Card \\
Oil trade (2017) \ref{app:oil-trade} & 229 & 100,639 & 1,757.9/175.0 & --- (ranking) & Sub \\
\bottomrule
\end{tabular}
\end{table}

\paragraph{Trivago-Clicks \citep{chodrow2021generative}.} The nodes in this hypergraph represent accommodations and hotels. A hyperedge connects a set of nodes when a user executes a ``click-out'' action during the same browsing session, indicating that the user was redirected to a partner website. The graph is constructed by employing geographical locations as ground truth cluster identities. The dataset contains $160$ such clusters in total. Similar to \citep{fountoulakis2021local}, we focus on all clusters in this dataset that contain fewer than $1,000$ nodes and exhibit conductance below $0.25$.

\paragraph{Amazon-reviews \citep{ni2019justifying, veldt2020minimizing}.} This hypergraph is constructed from Amazon product review data, with each node representing a product. Products are connected by a hyperedge when they are reviewed by the same individual. Similar to \cite{fountoulakis2021local}, the product category labels are employed as ground truth cluster identities. The dataset encompasses $29$ product categories in total. Given our primary focus on local clustering, we examine all clusters containing fewer than $10,000$ nodes.

\paragraph{High-school-contact \citep{chodrow2021generative,mastrandrea2015contact}} Nodes in this hypergraph correspond to high school students. Students form a hyperedge when they are all detected in close proximity simultaneously, as measured by wearable sensors. The ground truth clusters are defined based on classroom assignments, with $9$ classrooms represented in the dataset. 

\paragraph{Florida Bay food network \citep{li2017inhomogeneous}.} Each node in this hypergraph represents a species or organism inhabiting the Bay, while hyperedges encode transformed network motifs. Nodes are labeled according to their role in the food chain.  Submodular cut costs are shown to be more appropriate for this dataset \citep{fountoulakis2021local}. 

\paragraph{Oil Trade $2017$ \citep{fountoulakis2021local}.} This hypergraph is constructed from $2017$ international oil-trade records in the
UN Comtrade dataset. Each node represents a country, and a $4$-tuple
\(\{v_1,v_2,v_3,v_4\}\) forms a hyperedge if the trade surplus from each of
\(v_1\) and \(v_2\) to each of \(v_3\) and \(v_4\) exceeds $10$ million USD. This
threshold is approximately the $80$th percentile of country-level oil export
values. Two countries are grouped in the same hyperedge when they share at least
\emph{two} major trading partners. Similar to \citet{fountoulakis2021local}, we use this network for a node-ranking task rather than a clustering task. As with the Florida Bay food network, submodular cut-costs are more appropriate for this dataset.

\subsection{Hypergraph cut-cost models} \label{app:data_cut}
Our experiments use three cut-cost classes. Under \emph{unit} cut-cost, every non-trivial split of a hyperedge incurs cost~$1$. Under \emph{cardinality}
cut-cost, the penalty depends on the number of nodes on each side:
$w_e(S) = \min(|S\cap e|,\, |e\setminus S|)/\lfloor|e|/2\rfloor$, giving
finer-grained credit for edges that overlap heavily with the active set.
Under \emph{general submodular} cut-cost, an arbitrary submodular function
$w_e$ assigns a different penalty to each split, capturing directional or
role-based structure (e.g., prey--predator interactions in food webs). We
prefix methods by U- (unit), C- (cardinality), and S- (submodular) to
indicate the cut-cost used.

\subsection{Methods and Parameter Settings}\label{app:params}

\subsubsection{HFD baseline \citep{fountoulakis2021local}.}
We set the same parameters for HFD as described in  \citet{fountoulakis2021local}. We use \(\sigma=10^{-4}\) in all clustering experiments. The total seed-injection mass is set to a constant multiple of the target-cluster volume: $\|\Delta\|_1 = t\,\vol(\mathcal  T)$, where \(\mathcal  T\) denotes the target cluster and \(t\) is a dataset-dependent injection factor. For Amazon-reviews, \(t=200\) for the smaller categories $\{1,2,3,12,18\}$, and \(t=50\) for the larger categories $\{15,17,24,25\}$. For Trivago-clicks and High-school-contact, \(t=3\). For the Florida Bay food network, $t$ is set to  cluster-specific values \(t=20,10,5\) for clusters $\{1,2,3\}$, respectively. The injection factor \(t\) controls how far the diffusion spreads from the seed set. Following prior HFD evaluations, \(t\) is chosen  large enough so that the diffusion reaches a nontrivial portion of the target cluster while still penalizing excessive spread through the objective. All the mass is placed on the single seed: $\Delta(s)=\|\Delta\|_1$. For the Oil trade (2017) ranking experiment, we follow the HFD protocol of \citet{fountoulakis2021local}. The seed node is Iran, and the seed-injection mass is set to $\Delta[\mathrm{seed}] = 500{,}000$.  Unlike the clustering experiments, which use \(\sigma=10^{-4}\), the Oil trade ranking experiment uses $\sigma = 0.1$. We run HFD with the iteration budget of $T=20$ for submodular cut-costs. Since this dataset is used for node ranking rather than clustering, we do not report F1 or conductance. Instead, countries are ranked by the HFD diffusion score $\sigma x_v$, where \(x\) is the returned HFD diffusion vector. 

\paragraph{LH \citep{liu2021strongly} and ACL \citep{andersen2006local} baselines.}
For the LH baselines, we use the parameter choices recommended by the authors
of ~\citet{liu2021strongly}. For both \(*\)-LH-2.0 and \(*\)-LH-1.4, we set
$\gamma=0.1, \; \rho=0.5, \; \kappa=c\,r,$ where \(r\) is the ratio between the number of seed nodes and the target-cluster size, and \(c\) is a tuning constant. For Amazon-reviews, we use \(c=0.025\), as recommended in~\citet{liu2021strongly}. For  Trivago-clicks, and Florida Bay, we also use \(c=0.025\). For High-school-contact, \(c=0.25\) is chosen by tuning so that both \(*\)-LH-2.0 and \(*\)-LH-1.4 obtain competitive performance. For ACL, we use the same parameter choices as specified in~\citet{fountoulakis2021local}. For the LH baselines, we set \(\delta=1\) for U-LH-\(*\), and $\delta=\max_{e\in E}|e|$ for C-LH-\(*\).

\paragraph{TL-HFD.} TL-HFD (Algorithm \ref{alg:pq-psgd-theory}) inherits the seed-injection protocol from the HFD baseline. The regularization parameter is fixed at $\sigma = 10^{-4}$ across all clustering experiments. Each trial uses a single-seed diffusion: for a chosen seed vertex $s \in T$, the injection vector $\Delta$ has exactly one nonzero entry,
\[
\Delta(s) \;=\; \delta_{\mathrm{exp}} \cdot \mathrm{vol}(T),
\qquad
\Delta(v) \;=\; 0 \quad \text{for } v \neq s,
\]
where $T$ is the target cluster and $\delta_{\mathrm{exp}}$ is a dataset-dependent injection factor.

For Amazon-reviews and Trivago-clicks, we run multiple independent trials per cluster, each using a different vertex as the single seed. For contact-high-school, Florida Bay, and the synthetic HSBM, we use every target-cluster vertex once as the seed (yielding $|\mathcal T|$ trials per cluster). For oil trade, the seed is fixed to Iran. We report the median F1 across trials in the clustering experiments. 

\underline{Injection factors.} The injection factors $\delta_{\mathrm{exp}}$ follow the standard HFD protocol of \citet{fountoulakis2021local}: $\delta_{\mathrm{exp}} = 3$ for Trivago-clicks, contact-high-school, and synthetic HSBM; $\delta_{\mathrm{exp}} = 200$ for the smaller Amazon categories $\{1, 2, 3, 12, 18\}$ and $\delta_{\mathrm{exp}} = 50$ for the larger categories $\{15, 17, 24, 25\}$; cluster-specific values $\{20, 10, 5\}$ for Florida Bay (Pelagic, Detritus, Benthic respectively); and $\Delta(\text{Iran}) = 500{,}000$ for the oil trade ranking experiment.

\underline{Boundary activation.} TL-HFD (Algorithm \ref{alg:pq-psgd-theory}) additionally requires a score exponent $\gamma$ and a boundary-activation rule. Boundary vertices are ranked by the  score
\[
s^{(t)}(u) \;=\; \kappa^{(t)}(u) \left( \frac{d_{\mathrm{in}}^{(t)}(u)}{d(u)} \right)^{\!\gamma},
\]
where $\gamma$ controls the strength of the structural commitment factor.

\underline{Two activation regimes: Activation-scale selection (auto-$k$) and grid-over-$k$.}
TL-HFD selects $k$, the number of boundary vertices activated per iteration,
in one of two ways depending on the dataset. Under \emph{Activation-scale
selection} (auto-$k$), the user specifies a grid of fractions $f$, and at
runtime the integer activation count $k$ is computed as a function of $f$
and the seed-injection magnitude. This rule lets a single fraction grid
generalize across clusters of different sizes without retuning $k$ directly.
Under \emph{grid-over-$k$}, the user instead supplies $k$ from a small set
of integer values. In both regimes, the configuration achieving the lowest
output conductance is reported; ties are broken by selecting the smallest
$f$ (auto-$k$) or smallest $k$ (grid-over-$k$). For the oil-trade ranking
task, where there is no labeled target cluster and conductance is undefined,
the configuration is selected by qualitative inspection of the returned
ranking.

The auto-$k$ rule activates the top-scoring boundary vertices according to
\[
k \;=\; \max\!\left\{1, \left\lfloor f \cdot \mathrm{vol}_{\mathrm{est}} \right\rceil\right\}
\quad \text{(unit cut-cost),}
\qquad
k \;=\; \max\!\left\{1, \left\lfloor f \cdot n_{\mathrm{est}} \right\rceil\right\}
\quad \text{(cardinality cut-cost),}
\]
where
\[
\mathrm{vol}_{\mathrm{est}} \;=\; \frac{\max_v \Delta(v)}{\delta_{\mathrm{exp}}},
\qquad
n_{\mathrm{est}} \;=\; \frac{\mathrm{vol}_{\mathrm{est}}}{\bar{d}},
\]
and $\bar{d}$ is the mean vertex degree. Because each trial uses a single seed with $\max_v \Delta(v) = \delta_{\mathrm{exp}} \cdot \mathrm{vol}(\mathcal T)$, the estimator simplifies to $\mathrm{vol}_{\mathrm{est}} = \mathrm{vol}(\mathcal T)$, so $f$ has a clear interpretation as a per-iteration boundary expansion rate relative to the target volume. Smaller $f$ yields a more conservative, locality-respecting diffusion; larger $f$ expands faster but increases the risk of activating vertices outside the target cluster. The fraction is selected from a small per-cluster grid (Section~\ref{app:fraction-sweep}). This rule is well-behaved on graphs with moderate density and a clearly defined target cluster: the range $f \in [10^{-3}, 10^{-1}]$ then produces $k$ values in $[1, |\mathcal T|]$, allowing a single per-dataset sweep to identify a near-optimal $f$. We use auto-$k$ on Trivago-clicks, Amazon-reviews, contact-high-school, and synthetic HSBM.

Two regimes break the volume-based scaling and require grid-over-$k$ instead. The first is \emph{very dense graphs}, where the target volume is large relative to $|V|$: on Florida Bay, $\bar{d} \approx 4{,}484$ and a target cluster of $35$ vertices has $\mathrm{vol}(T) \approx 173{,}000$, so even $f = 10^{-3}$ produces $k > |V|$ and auto-$k$ saturates instantly. The second is \emph{node-ranking tasks without a target cluster}, where $\mathrm{vol}(T)$ is undefined: the oil trade diffusion is run from a single fixed seed (Iran) for the purpose of ranking trading partners, with no labeled target community to scale against. In both regimes we sweep $k$ over a small grid and report the configuration that produces the lowest output conductance (Florida Bay) or the most semantically coherent ranking (oil trade).

We use \(\gamma=0.5\) on Amazon-reviews and \(\gamma=1.0\) on
Trivago-clicks and contact-high-school. The smaller exponent on Amazon makes
the structural commitment factor less aggressive, so the ranking remains more
responsive to the gradient push \(\kappa^{(t)}(u)\). This was empirically more
stable on Amazon, where product co-review hyperedges often contain bridge
products spanning multiple categories. On Trivago-clicks and contact-high-school, the cluster boundaries are cleaner --- co-clicked accommodations and within-classroom contact events are associated with a single cluster --- so a sharper $\gamma = 1.0$ produces better recovery.

\begin{table}[h]
\centering
\caption{Dataset-specific TL-HFD parameters for real-world experiments. $T_{\mathrm{iter}}$ denotes the number of TL-HFD iterations; two values are given when unit and cardinality cut-costs use different counts.}
\label{tab:tlhfd-real-params}
\begin{tabular}{l c c  c l}
\toprule
Dataset & Cut-costs  & $T_{\mathrm{iter}}$ (unit / card) & $\gamma$ & Activation  \\
\midrule
Trivago-clicks      & Unit, Card      & 500 / 1000 & 1.0  & auto-$k$ via $f$ \\
Amazon-reviews      & Unit              & 1000 / ---  & 0.5  & auto-$k$ via $f$ \\
Contact-high-school & Unit, Card     & 1000 / 1000 & 1.0  & auto-$k$ via $f$ \\
Florida Bay         & Sub             & 500 / ---  & 1.0  & grid over $k$ \\
Oil trade (2017)    & Sub             & 20  / ---  & 1.0 & grid over $k$ \\
\bottomrule
\end{tabular}
\end{table}

\underline{Fraction and \(k\)-grid summary for real-world datasets.}
\label{app:fraction-sweep}
 Here we summarize the grids used to select the boundary activation scale in the experiments. For sparse-to-moderate clustering datasets, TL-HFD uses a fraction \(f\) and converts it to an activation count \(k\) through the auto-\(k\) rule. For dense submodular clustering and ranking datasets, such as Florida Bay and Oil Trade, the volume-based scaling is less meaningful as explained above: in Florida Bay the extreme density can make the resulting \(k\) too large, while in Oil Trade there is no labeled target cluster volume because the task is ranking rather than clustering. For these datasets, we sweep \(k\) directly over a small integer grid.

 The auto-$k$ rule uses different normalizations for the two cut-cost classes: $k = \max\{1, \lfloor f \cdot \mathrm{vol}_{\text{est}} \rceil\}$ for unit cut-cost and $k = \max\{1, \lfloor f \cdot n_{\text{est}} \rceil\}$ for cardinality cut-cost, where $n_{\text{est}} = \mathrm{vol}_{\text{est}} / \bar{d}$ and $\bar{d}$ is the mean vertex degree. For the same $f$, the cardinality rule therefore activates approximately $\bar{d}$-fold fewer boundary vertices per iteration than the unit rule. To explore comparable activation counts across cut-cost classes, the cardinality grid is extended to roughly $\bar{d}$-fold larger fractions: on Trivago-clicks ($\bar{d} \approx 5.6$), this corresponds to unit fractions in $[0.01, 0.05]$ and cardinality fractions in $[0.01, 0.10]$. The same scaling motivates the corresponding grid for submodular cut-cost on datasets where it is used. This is purely a normalization difference in how $f$ maps to $k$. 
 
\begin{table}[h]
\centering
\caption{Activation-scale grids used for TL-HFD. A fraction grid means that
\(k\) is computed by the auto-\(k\) rule; a direct \(k\)-grid means that the
integer activation count is swept directly.}
\label{tab:fraction-grids}
\small
\begin{tabular}{l l l}
\toprule
Dataset & Cut-cost & Activation-scale grid \\
\midrule
Trivago-clicks
& unit
& \(f\in\{0.01,0.02,0.03,0.05\}\) \\

Trivago-clicks
& cardinality
& \(f\in\{0.01,0.02,0.03,0.05,0.07,0.10\}\) \\

Amazon-reviews
& unit
& \(f\in\{0.001,0.002,0.005,0.01,0.02,0.03, 0.05, 0.07, 0.1\}\) \\

Contact-high-school
& unit
& \(f\in\{0.01, 0.05, 0.1, 0.2, 0.5\}\) \\

Contact-high-school
& cardinality
&  \(f\in\{0.01, 0.05, 0.1, 0.2, 0.5\}\)\\

Florida Bay food web
& submodular
& direct \(k\)-grid, \(k\in\{1,2,3,5,7,10,15\}\) \\

Oil trade
& submodular
& direct \(k\)-grid, $k \in \{3, 5, 10\}$ \\

\bottomrule
\end{tabular}
\end{table}

For Trivago-clicks and Amazon-reviews, we report how sensitive F1 is to the fraction $f$ in their corresponding  sections. 

\subsection{Theory-Experiment bridge for seed injection mass}
\label{sec:theory-exp-bridge}

The theoretical analysis of Section~\ref{sec:tlhfd} uses the degree-weighted seed injection $\Delta(v)=\delta d_v$ for $v\in S$ and $\Delta(v)=0$ otherwise. The experiments in Section~\ref{sec:experiments} and Section~\ref{app:params} use a single-seed rule $\Delta[s^\star]=\delta_{\exp}\cdot\vol(\mathcal  T)$ with $s^\star\in C$ and target $T=C$ (this is also the case for \citep{fountoulakis2021local}). The following proposition makes the identification between these two rules precise.

\begin{proposition}
\label{prop:theory-experiment-bridge}
Let $S=\{s^\star\}$ with $s^\star\in C$, $T=C$, and let
$\Delta[s^\star]=\delta_{\exp}\cdot\vol(\mathcal T)$ with $\delta_{\exp}\geq 3$.
Define the seed-specific
\[
\delta \;:=\; \frac{\delta_{\exp}\cdot\vol(\mathcal  T)}{d_{s^\star}}.
\]
Then:
\begin{enumerate}
  \item[(a)] $\Delta$ agrees with the degree-weighted rule $\Delta(v)=\delta d_v$ for $v\in S$.
  \item[(b)] Assumption~\ref{ass:overlap} holds with $\beta=1$ and any
  $\alpha\leq d_{s^\star}/\vol(C)$; in particular, since
  $\delta=\delta_{\exp}\vol(C)/d_{s^\star}$ and $\delta_{\exp}\geq 3$, it
  holds with the choice $\alpha=3/\delta$.
  \item[(c)] Assumption~\ref{ass:parameter} holds whenever $\sigma\leq\Phi(C)/3$.
  \item[(d)] Theorems~\ref{thm:eps-optimality},~\ref{thm:robust-sweep}, and~\ref{thm:explored_volume_topk}; Corollaries~\ref{cor:dual_gap_threshold}, and ~\ref{cor:early-stopping}; and Proposition~\ref{prop:Gloc-bound} apply with the seed-specific \(\delta\) as above. The lower bound $D(x^\star)\geq D_C$ (Auxiliary Lemma \ref{lem:hfd-dual-lower}) and the sweep bound retains the same form \(3K/\alpha\), now evaluated with the seed-specific \(\alpha=3/\delta\). The box radius and subgradient constant become
\[
\bar x \;=\; \frac{\delta-1}{\sigma},\qquad
B \;=\; (\delta-1)\!\left(\tfrac{1}{\sigma}+2\right),
\]
with the seed-specific $\delta$.
\end{enumerate}
\end{proposition}

\begin{proof} 
\emph{Part (a).} Direct: $\delta d_{s^\star}=\delta_{\exp}\cdot\vol(\mathcal T)$ by construction.

\emph{Part (b).} Since $S=\{s^\star\}\subseteq C$, we have
$\vol(S\cap C)=d_{s^\star}=\vol(S)$, so $\beta=1$. The tight overlap ratio is
\[
\alpha_{\rm tight}
\;=\;
\frac{\vol(S\cap C)}{\vol(C)}
\;=\;
\frac{d_{s^\star}}{\vol(C)}.
\]
Since $T=C$, $\vol(\mathcal T)=\vol(C)$, and the defining relation
$\delta=\delta_{\exp}\vol(\mathcal T)/d_{s^\star}$ gives
$\delta=\delta_{\exp}\vol(C)/d_{s^\star}$. Hence
\[
\frac{3}{\delta}
\;=\;
\frac{3d_{s^\star}}{\delta_{\exp}\vol(C)}
\;\leq\;
\frac{d_{s^\star}}{\vol(C)}
\;=\;
\alpha_{\rm tight},
\]
because $\delta_{\exp}\geq 3$. Assumption~\ref{ass:overlap} therefore
holds with the looser choice $\alpha=3/\delta\leq\alpha_{\rm tight}$ and
$\beta=1$.

\emph{Part (c).} Assumption~\ref{ass:parameter} requires $\delta=3/\alpha$,
$0\leq w_e(A)\leq 1$, and $\sigma\leq\beta\Phi(C)/3$. The first holds with
the choice $\alpha=3/\delta$ from part~(b); the second is inherited from the
cut-cost normalization of Section~2; and the third reduces to
$\sigma\leq\Phi(C)/3$ since $\beta=1$.

\emph{Part (d).} We verify that Auxiliary Lemma \ref{lem:hfd-dual-lower} goes through unchanged.
Using $x=\lambda\mathbf{1}_C$ and $S=\{s^\star\}\subseteq C$,
\[
\langle\Delta-d,\,\lambda\mathbf{1}_C\rangle
\;=\;
\lambda\!\left(\delta_{\exp}\cdot\vol(\mathcal T)\cdot\mathbf{1}_{s^\star\in C}
\;-\;\vol(C)\right)
\;=\;
\lambda(\delta_{\exp}-1)\vol(C)
\;\geq\;
2\lambda\vol(C),
\]
using $T=C$ and $\delta_{\exp}\geq 3$. This matches the lower bound used in
the proof of Auxiliary Lemma~3, so $D(\lambda\mathbf{1}_C)\geq D_C$ and
hence $D(x^\star)\geq D_C$. All subsequent arguments
(Theorem~\ref{thm:robust-sweep}, Corollary~\ref{cor:early-stopping}) are
unchanged. The constants $\bar x$ and $B$ in
Proposition~\ref{prop:Gloc-bound} are stated in terms of $\delta$
and $\sigma$, so they specialize with the seed-specific $\delta$ given above.
\end{proof}

\paragraph{Remark on $\delta_{\exp}>3$.}
For datasets where $\delta_{\exp}>3$ (e.g., Amazon with
$\delta_{\exp}\in\{50,200\}$), the seed-specific $\delta$ is correspondingly
larger. Part~(b) still applies with $\alpha=3/\delta$, which is smaller than
$\alpha_{\rm tight}$; Assumption~\ref{ass:overlap} holds at this looser
$\alpha$ because it is a lower bound. The sweep constant $3K/\alpha$ grows
with $\delta_{\exp}$, reflecting the over-injection relative to the tight
seed–cluster overlap.

\paragraph{Notation convention for experimental injection.} In the experimental sections, when we report values such as \(\delta=3\), \(\delta=50\), or \(\delta=200\), this symbol denotes the experimental injection multiplier \(\delta_{\mathrm{exp}}\) used in the single-seed rule \(\Delta[s^\star]=\delta_{\mathrm{exp}}\operatorname{vol}(T)\), following the convention of \citep{fountoulakis2021local}. This experimental multiplier should not be confused with the degree-weighted theoretical coefficient \(\delta\) used in Section~\ref{sec:tlhfd}, where \(\Delta(v)=\delta d_v\) for \(v\in S\). For a single seed \(s^\star\), the two are related by $\delta_{\mathrm{theory}}(s^\star)
= \frac{\delta_{\mathrm{exp}}\operatorname{vol}(T)}{d_{s^\star}}$. Experimental hyperparameter values use the shorthand \(\delta\equiv\delta_{\mathrm{exp}}\).
 
\subsection{Trivago-Clicks}\label{app:trivago-sensitivity}

Table \ref{app:trivago-details} shows the cluster codes used in the Table \ref{tab:trivago-f1} for the Trivago dataset.  Tables ~\ref{tab:trivago-unit-frac} and~\ref{tab:trivago-card-frac} show the median F1 across the full fraction grid, demonstrating that performance is stable over a range of $f$ values. The cardinality and submodular cut-costs require larger fractions than the unit cut-cost because the weighted in-degree $d_{\mathrm{in}}^{\mathrm{card}}(u)$ is generally smaller than its unit counterpart, making the score more selective; consequently the auto-$k$ rule with a given $f$ activates fewer vertices in the cardinality and submodular settings. Accordingly, we draw $f$ from $[0.01, 0.05]$ for unit cut-cost and from $[0.01, 0.10]$ for cardinality cut-cost.

\begin{table}[!ht]
\centering
\caption{Trivago-clicks: Cluster Codes for Table~\ref{tab:trivago-f1}. $|\mathcal T|$ is the target cluster size;
$f$ is the boundary-fraction;
$k$ is the resulting per-iteration activation count. $\gamma$ is set to $1$ for both TL-HFD cut costs.}
\label{app:trivago-details}
\setlength{\tabcolsep}{6pt}
\begin{tabular}{l c c c c c}
\toprule
              &       & \multicolumn{2}{c}{TL-UHFD (unit)} & \multicolumn{2}{c}{TL-CHFD (card.)} \\
\cmidrule(lr){3-4} \cmidrule(lr){5-6}
Cluster       & $|\mathcal T|$ & $f$   & $k$ & $f$   & $k$  \\
\midrule
KOR  (South Korea)  & 945  & 0.02 & 74  & 0.07 & 47 \\
ISL  (Iceland)      & 202  & 0.03 & 25  & 0.07 & 11 \\
PRI  (Puerto Rico)  & 144  & 0.02 & 9   & 0.03 & 3  \\
CRM  (Crimea)       & 200  & 0.01 & 11  & 0.03 & 6  \\
VNM  (Vietnam)      & 832  & 0.01 & 23  & 0.05 & 21 \\
HKG  (Hong Kong)    & 536  & 0.01 & 23  & 0.02 & 17 \\
MLT  (Malta)        & 157  & 0.01 & 5   & 0.07 & 6  \\
GTM  (Guatemala)    & 199  & 0.01 & 7   & 0.10 & 12 \\
UKR  (Ukraine)      & 264  & 0.01 & 6   & 0.10 & 12 \\
EST  (Estonia)      & 158  & 0.01 & 8   & 0.05 & 8  \\
\bottomrule
\end{tabular}
\end{table}
 
\begin{table}[!ht]
\centering
\caption{Trivago-clicks: median F1 vs.\ fraction $f$ for unit cut-cost (vol-based $k$-scaling).}
\label{tab:trivago-unit-frac}
\begin{tabular}{l c c c c}
\toprule
Cluster & $f\!=\!.01$ & $.02$ & $.03$ & $.05$ \\
\midrule
KOR (South Korea)   & 0.81 & 0.86 & 0.81 & 0.62 \\
ISL (Iceland)       & 0.98 & 0.98 & 0.98 & 0.96 \\
PRI (Puerto Rico)   & 0.97 & 0.97 & 0.97 & 0.97 \\
CRM (Crimea)        & 0.85 & 0.84 & 0.84 & 0.83 \\
VNM (Vietnam)       & 0.58 & 0.55 & 0.50 & 0.34 \\
HKG (Hong Kong)    & 0.65 & 0.49 & 0.46 & 0.09 \\
MLT (Malta)        & 0.97 & 0.96 & 0.78 & 0.79 \\
GTM (Guatemala)     & 0.95 & 0.95 & 0.95 & 0.95 \\
UKR (Ukraine)       & 0.65 & 0.65 & 0.51 & 0.38 \\
EST (Estonia)      & 0.79 & 0.56 & 0.56 & 0.55 \\
\bottomrule
\end{tabular}
\vspace{2pt}
 
\emph{Observation: $f \in [0.01, 0.02]$ is broadly robust for Trivago's unit cut-cost. Performance degrades outside this range, with larger fractions causing overshoot
(activating too many boundary vertices per iteration) and smaller fractions
causing under-exploration.}
\end{table}
 
\begin{table}[!ht]
\centering
\caption{Trivago-clicks: median F1 vs.\ fraction $f$ for cardinality cut-cost ($n_{\mathrm{est}}$-based $k$-scaling).}
\label{tab:trivago-card-frac}
\begin{tabular}{l c c c c c c}
\toprule
Cluster & $f\!=\!.01$ & $.02$ & $.03$ & $.05$ & $.07$ & $.10$ \\
\midrule
KOR (South Korea)  & 0.24 & 0.75 & \textbf{0.82} & 0.81 & \textbf{0.82} & 0.77 \\
ISL (Iceland)      & 0.07 & 0.09 & 0.59 & 0.95 & \textbf{0.98} & 0.94 \\
PRI (Puerto Rico)  & 0.07 & 0.91 & \textbf{0.97} & \textbf{0.97} & \textbf{0.97} & \textbf{0.97} \\
CRM (Crimea)    & 0.30 & 0.71 & \textbf{0.85} & 0.70 & 0.66 & 0.75 \\
VNM (Vietnam)      & 0.20 & 0.22 & 0.30 & \textbf{0.32} & \textbf{0.32} & \textbf{0.32} \\
HKG (Hong Kong)    & 0.02 & \textbf{0.76} & 0.73 & 0.02 & 0.03 & 0.01 \\
MLT (Malta)         & 0.05 & 0.23 & 0.96 & \textbf{0.97} & \textbf{0.97} & \textbf{0.97} \\
GTM (Guatemala)    & 0.05 & 0.17 & 0.96 & 0.96 & \textbf{0.97} & \textbf{0.97} \\
UKR (Ukraine)     & 0.03 & 0.08 & 0.42 & 0.59 & 0.61 & \textbf{0.73} \\
EST (Estonia)      & 0.07 & 0.08 & 0.10 & \textbf{0.58} & 0.56 & 0.42 \\
\bottomrule
\end{tabular}
\vspace{2pt}
 
\emph{Observation: Cardinality requires larger fractions than unit,
reflecting the finer-grained weighted $d_{\mathrm{in}}$ which makes the scoring
more selective. Hong Kong exhibits instability at $f > 0.03$, suggesting the cluster
boundary is sharp and sensitive to over-expansion.}
\end{table}

\subsection{Results on Amazon-reviews Dataset}\label{sec:exp-amazon}
The Amazon-reviews dataset~\citep{veldt2020minimizing, ni2019justifying} is a large-scale hypergraph with $n = 2{,}268{,}231$ nodes, $m = 4{,}285{,}363$ hyperedges, and mean degree $\bar{d} \approx 32.3$. Hyperedges are sets of product reviews, and nodes are labeled by product category. Following~\citet{fountoulakis2021local}, we evaluate on 9 category clusters with varying injection coefficients ($\delta = 200$
for small clusters, $\delta = 50$ for large clusters). For each cluster $\mathcal T$, we run $\min(|\mathcal T|, 500)$ single-seed trials: every target-cluster vertex is used once as a seed when $|\mathcal T| \leq 500$, and $500$ vertices are sampled uniformly at random from $\mathcal T$ otherwise. We set $\sigma = 10^{-4}$, $\gamma = 0.5$, and run TL-HFD for $T = 1000$ iterations (unit cut-cost only, matching the baseline). Auto-$k$ fractions are swept over $f \in \{0.001, 0.002, 0.005, 0.01, 0.02, 0.03, 0.05, 0.07, 0.1\}$, with fractions producing $k > |\mathcal T|$ skipped. For each cluster, we select $f$ by minimum output conductance $\Phi(S_h)$ across the grid (ties are resolved by selecting the F1 corresponding to the smallest such $f$) --- an unsupervised criterion that uses no ground-truth labels --- and report the median F1 across trials at the selected~$f$. The sensitivity to the structural-commitment exponent $\gamma$ is provided in Appendix~\ref{app:gamma-ablation}.

\begin{table}[t]
\centering
\caption{Amazon-reviews: 
Best result per cluster (unit costs) in \textbf{bold}. Clusters are grouped by injection
coefficient $\delta$ (smaller clusters use $\delta = 200$; larger use
$\delta = 50$). }
\label{tab:amazon-f1}
\setlength{\tabcolsep}{4pt}
\begin{tabular}{l c c c c c c c c c}
\toprule
Method    & 2             & 1             & 3             & 12            & 18            & 15            & 17            & 24            & 25            \\
          & \scriptsize Fashion & \scriptsize Beauty & \scriptsize Appliances & \scriptsize Gift Cards & \scriptsize Magazine & \scriptsize Ind./Sci. & \scriptsize Lux. Beauty & \scriptsize Pantry & \scriptsize Software \\
\midrule
U-HFD     & 0.26        & 0.07        & 0.46         & 0.82         & 0.71         & 0.02 & 0.09         & \textbf{0.83} & 0.09         \\
TL-UHFD   & \textbf{0.28} & \textbf{0.09} & \textbf{0.61} & \textbf{0.86} & \textbf{0.75} & 0.01         & \textbf{0.09} & 0.79         & \textbf{0.11} \\
U-LH-2.0  & 0.21          & 0.07          & 0.23          & 0.29          & 0.21          & \textbf{0.05} & 0.06          & 0.28          & 0.05          \\
U-LH-1.4  & 0.21          & 0.09          & 0.35          & 0.40          & 0.31          & 0.01          & 0.07          & 0.35          & 0.06          \\
ACL       & 0.20          & 0.07          & 0.22          & 0.25          & 0.17          & 0.04          & 0.05          & 0.20          & 0.04          \\
\bottomrule
\end{tabular}
\end{table}

\paragraph{Results.}
Table~\ref{tab:amazon-f1} shows the results. TL-UHFD improves over U-HFD on 6 out of 9 clusters, ties on one, loses on two. For the $\delta = 200$ clusters (small,
well-separated categories), TL-UHFD wins on all five. The largest absolute gain is
$+0.15$ F1 on Appliances ($0.46 \to 0.61$, a $33\%$ relative improvement), followed
by $+0.04$ on Gift Cards ($0.82 \to 0.86$) and $+0.04$ on Magazine Subscriptions
($0.71 \to 0.75$); Fashion and Beauty show modest gains of $+0.02$ each. For the $\delta = 50$ clusters (larger categories), results are more mixed: TL-UHFD ties U-HFD on Luxury Beauty (both at $0.09$) and edges ahead by $+0.02$ on Software, while losing
$-0.04$ on Prime Pantry, where U-HFD's primal method finds the cluster effectively.
Industrial~\& Scientific fails for all diffusion-based methods (best F1~$=0.05$,
achieved by U-LH-2.0), indicating this cluster lacks well-defined local structure
amenable to local clustering (baseline conductance~$>0.5$). TL-UHFD outperforms LH
and ACL baselines, often by large margins ---
e.g., $+0.46$ over U-LH-1.4 on Magazine Subscriptions and $+0.57$ over ACL on
Gift Cards. Table~\ref{tab:amazon-frac} reports the sensitivity of F1 to the
boundary fraction $f$, confirming that the unsupervised conductance-based selection
picks fractions in regions of stable performance.

\begin{table}[h]
\centering
\caption{Amazon-reviews: median F1 vs.\ fraction $f$ for unit cut-cost (vol-based $k$-scaling). N/A indicates $k > |\mathcal T|$ (skipped). $^\star$ indicates the cluster with lowest conductance that is chosen and reported in Table~\ref{tab:amazon-f1}. The Indust.\ \& Sci. row is truncated at two decimal digits. }
\label{tab:amazon-frac}
\begin{tabular}{l c c c c c c c c c}
\toprule
Cluster & $f\!=\!.001$ & $.002$ & $.005$ & $.01$ & $.02$ & $.03$ & $.05$ & $.07$ & $.10$ \\
\midrule
Amazon Fashion    & 0.23 & 0.23 & 0.28$^\star$ & 0.27 & N/A & N/A & N/A & N/A & N/A \\
All Beauty        & 0.05 & 0.08 & 0.09$^\star$ & 0.09 & 0.08 & N/A & N/A & N/A & N/A \\
Appliances        & 0.22 & 0.22 & 0.22 & 0.28 & 0.28 & 0.62 & 0.62 & 0.61$^\star$ & 0.61 \\
Gift Cards        & 0.25 & 0.34 & 0.53 & 0.73 & 0.86$^\star$ & 0.82 & 0.72 & N/A & N/A \\
Magazine Sub.     & 0.09 & 0.20 & 0.31 & 0.48 & 0.75$^\star$ & 0.76 & 0.65 & N/A & N/A \\
Indust.\ \& Sci.  & 0.00 & 0.00 & 0.00 & 0.00 & 0.00 & 0.01$^\star$ & 0.01 & 0.01 & N/A \\
Luxury Beauty     & 0.01 & 0.03 & 0.04 & 0.08 & 0.08 & 0.08 & 0.09$^\star$ & N/A & N/A \\
Prime Pantry      & 0.09 & 0.20 & 0.31 & 0.35 & 0.79$^\star$ & 0.74 & N/A & N/A & N/A \\
Software          & 0.02 & 0.03 & 0.05 & 0.09 & 0.11$^\star$ & 0.12 & 0.12 & N/A & N/A \\
\bottomrule
\end{tabular}
\vspace{2pt}
 
\emph{Observation: $\delta=200$ clusters peak at $f \in [0.005, 0.03]$. $\delta=50$ clusters
need larger fractions ($f \in [0.02, 0.05]$) to generate sufficient boundary
activations. Industrial \& Scientific fails across all fractions for both
methods, indicating poor cluster separability.}
\end{table}

\subsection{Florida Bay Food Web}\label{app:foodweb}
 
The Florida Bay food web~\citep{li2017inhomogeneous} is a small but extremely dense hypergraph with $n = 126$ nodes, $m = 141{,}233$ hyperedges (all 4-uniform), and mean degree $\bar{d} \approx 4{,}484$. It contains 3 clusters corresponding to trophic levels: Pelagic ($|\mathcal T|=17$), Detritus ($|\mathcal T|=35$), and Benthic ($|\mathcal T|=70$). Each hyperedge encodes a motif between two prey and two predator species.
 
\paragraph{Why submodular cut-cost only.}
Following~\citet{fountoulakis2021local}, we evaluate using the \emph{submodular} cut-cost (S-HFD), which is the natural model for this dataset. The submodular weight function $w_e$ distinguishes between prey--prey splits ($w_e = 0$), predator--predator splits ($w_e = 0$), and prey--predator splits ($w_e = 1$). Unit and cardinality cut-costs treat all splits equally, losing this biological distinction. The HFD paper uses the food web specifically to demonstrate the advantage of submodular cut-costs over simpler alternatives.
 
\paragraph{Setup.} We set $\sigma = 10^{-4}$, $T = 500$ iterations, and we use every target-cluster vertex once as a seed (matching the HFD protocol). The injection coefficients are cluster-specific: $\delta_{exp} = 20$ (Pelagic), $\delta_{exp} = 10$ (Detritus), $\delta_{exp} = 5$ (Benthic). We compare S-HFD and our method (TL-HFD) with top-$k$, $\gamma \in \{0.25, 0.5, 1.0\}$ over $k \in \{1, 2, 3, 5, 7, 10, 15\}$.

Unlike the other datasets, we sweep over a grid of $k$ values rather than using the fraction-based $k$ selection. The food web's extreme density ($\bar{d} \approx 4{,}484$ with only 126 nodes) makes the volume-based scaling $k = f \cdot \text{vol}_{\mathrm{est}}$ produce unreasonably large $k$ values --- for instance, Detritus has $\text{vol}_{\mathrm{est}} = 173{,}311$, so even $f = 0.001$ yields $k = 173$, exceeding the total number of nodes. The auto-$k$ mechanism is designed for sparse-to-moderate hypergraphs where volume scales proportionally with cluster size; it is not applicable in this ultra-dense regime.

The  scoring for submodular cut-cost uses the submodular weight function to compute $d_{\mathrm{in}}$: for each edge $e$ incident to boundary vertex $u$ that touches the active set $A$, $d_{\mathrm{in}}(u) \mathrel{+}= w_e(A \cap e)$. 
 
\begin{table}[h]
\centering
\caption{Florida Bay food web: median F1 for submodular cut-cost. Best result per cluster in \textbf{bold}.}
\label{tab:foodweb}
\begin{tabular}{l c c c c}
\toprule
Cluster & $|\mathcal T|$ & S-HFD  & TL-HFD  & $(\gamma, k)$ \\
\midrule
Pelagic  &  17 & \textbf{0.69} & \textbf{0.69} ($k\!\leq\!15$) & any \\
Detritus &  35 & 0.62 &  \textbf{0.68} ($k\!=\!3$) & $\gamma\!=\!0.25$ \\
Benthic  &  70 & \textbf{0.84} &  \textbf{0.84} ($k\!=\!3$) & $\gamma\!=\!0.25$ \\
\bottomrule
\end{tabular}
\end{table}

 \paragraph{Results.}
Table~\ref{tab:foodweb} summarizes the results. The food web is a dense 
dataset: it has only \(126\) nodes but \(141{,}233\) hyperedges, giving a mean
degree of approximately \(4{,}484\). Thus, each node participates in thousands
of hyperedges. As a consequence, the boundary at each iteration is often
smaller than the tested values of \(k\), so thresholding has limited effect:
in many runs, all boundary vertices are activated.

For Pelagic, TL-SHFD matches S-HFD, achieving median F1 \(=0.69\) across
almost all tested \(k\) and \(\gamma\) values. For Detritus, TL-HFD improves
over S-HFD, increasing median F1 from \(0.62\) to \(0.68\) at
\(\gamma=0.25\) and \(k=3\). For Benthic, S-HFD and TL-SHFD attain the same result,
achieving F1 \(=0.84\), at \(\gamma=0.25\) and \(k=3\).

These results show that TL-HFD with submodular cut-cost correctly integrates
the submodular weight function in both the subgradient computation and the
 boundary scoring, while largely preserving the HFD's clustering
quality on this dense, well-structured dataset.

\subsection{Contact-High-School}
\label{app:highschool}

The contact-high-school dataset~\citep{mastrandrea2015contact} is a dense
hypergraph of face-to-face interactions among students, with \(n=327\) nodes,
\(m=7{,}818\) hyperedges, and mean degree \(\bar d\approx 55\). The nine
ground-truth clusters correspond to school classes, such as 2BIO1, MP*1, and
PC. Following~\citet{fountoulakis2021local}, we evaluate each class by using
every target-cluster vertex once as a single seed, yielding \(|\mathcal T|\)
single-seed trials per class.

\paragraph{Setup.}
We set \(\sigma=10^{-4}\), injection factor \(\delta_{\mathrm{exp}}=3\), and  exponent \(\gamma=1.0\). TL-HFD is run for \(T=1000\)
iterations for both unit and cardinality cut-costs. For each cluster, $f$
is selected from the grid by minimum output conductance, and the median F1 at the selected $f$ is reported.

\begin{table}[h]
\centering
\caption{Contact-high-school: median F1 for unit cut-cost. $f$ is
 selected per cluster by minimum output conductance for TL-UHFD. Best result per cluster in \textbf{bold}.}
\label{tab:school-unit}
\begin{tabular}{l c c c c }
\toprule
Cluster & $|\mathcal T|$ & U-HFD & TL-UHFD & $f$\\
\midrule
2BIO1  & 36 & \textbf{0.986} & \textbf{0.986} & 0.01 \\
2BIO2  & 34 & \textbf{1.000} & 0.964          & 0.05  \\
2BIO3  & 40 & 0.593          & \textbf{0.620} & 0.01   \\
MP*1   & 29 & \textbf{0.964} & 0.893          & 0.01   \\
MP*2   & 38 & 0.731          & \textbf{0.812} & 0.50  \\
PSI*   & 34 & \textbf{1.000} & \textbf{1.000} & 0.01   \\
PC     & 44 & 0.879          & \textbf{1.000} & 0.01 \\
PC*    & 39 & \textbf{1.000} & \textbf{1.000} & 0.01  \\
MP     & 33 & \textbf{0.985} & \textbf{0.985} & 0.01  \\
\bottomrule
\end{tabular}
\end{table}
 
\begin{table}[h]
\centering
\caption{Contact-high-school: median F1 for cardinality cut-cost. Best result per cluster in \textbf{bold}.}
\label{tab:school-card}
\begin{tabular}{l c c c c }
\toprule
Cluster & $|\mathcal T|$ & C-HFD & TL-CHFD & $f$  \\
\midrule
2BIO1  & 37 & \textbf{0.986} & \textbf{0.986} & 0.01  \\
2BIO2  & 34 & \textbf{1.000} & 0.946          & 0.01   \\
2BIO3  & 40 & \textbf{1.000} & \textbf{1.000} & 0.01  \\
MP*1   & 29 & 0.964          & \textbf{1.000} & 0.01  \\
MP*2   & 38 & 0.796          & \textbf{0.833} & 0.01   \\
PSI*   & 34 & \textbf{1.000} & \textbf{1.000} & 0.01   \\
PC     & 44 & \textbf{1.000} & \textbf{1.000} & 0.01   \\
PC*    & 39 & \textbf{1.000} & \textbf{1.000} & 0.01   \\
MP     & 33 & \textbf{0.985} & \textbf{0.985} & 0.01   \\
\bottomrule
\end{tabular}
\end{table}

 \paragraph{Results.}
Tables~\ref{tab:school-unit} and~\ref{tab:school-card} show the results. This
dataset is a challenging evaluation setting because HFD (\cite{fountoulakis2021local}) already achieves
near-perfect F1 on most classes, leaving limited room for improvement.

For unit cut-cost, TL-UHFD wins on three classes: 2BIO3 \((+0.027)\), MP*2
\((+0.081)\), and PC \((+0.121)\). It ties HFD on four classes and loses on
two classes, 2BIO2 \((-0.036)\) and MP*1 \((-0.071)\), where HFD is already
perfect or near-perfect.

For cardinality cut-cost, TL-CHFD ties or wins on eight out of nine classes.
MP*1 improves from \(0.964\) to \(1.000\), and MP*2 improves from \(0.796\) to
\(0.833\). The only loss is 2BIO2 \((-0.054)\), where HFD achieves perfect
F1. TL-CHFD achieves F1 \(=1.0\) on five classes: 2BIO3, MP*1, PSI*, PC, and
PC*. Overall, these results show that TL-HFD largely preserves quality on a dense,
well-clustered hypergraph where the baseline is already strong, while still
providing gains on some classes.

 \subsection{Oil Trade Ranking}\label{app:oil-trade}
 The oil-trade dataset~\citep{fountoulakis2021local} is a 4-uniform hypergraph
of international oil-trade flows with \(n=229\) countries, \(m=100{,}639\)
hyperedges, and mean degree \(\bar d\approx 1{,}758\). Following
\citet{fountoulakis2021local}, we use this dataset for a node-ranking task
rather than a clustering task: given a seed country, Iran, the goal is to rank
other countries by proximity in the diffusion embedding. There are no
ground-truth cluster labels, so evaluation is qualitative.

\paragraph{Setup.} We set $\Delta[\text{Iran}] = 500{,}000$, $\sigma = 0.1$, and run the submodular cut-cost with $T = 20$ iterations, matching the HFD parameter setting. We also run unit and cardinality cut costs for $T=100$ to observe the ranking generated from TL-HFD. For TL-SHFD (TL-HFD with submodular cut costs), rankings are computed from the best dual iterate: $\text{score}(v) = \sigma \cdot x^{\mathrm{dual}}_v$. TL-HFD uses a grid over $k$ with $k \in \{3, 5, 10\}$ and $\gamma \in \{0.25, 0.5, 1.0\}$.

 \paragraph{Why submodular is the correct model.}
The oil-trade hypergraph encodes $4$-way trade motifs in which two countries
export and two countries import within each transaction. The submodular
cut-cost distinguishes between exporter--exporter splits, which receive low
penalty, and exporter--importer splits, which receive high penalty. This
captures the directional structure of trade flows. Unit and cardinality
cut-costs treat all nontrivial splits symmetrically, and therefore discard this
exporter--importer structure.

\begin{table}[h]
\centering
\caption{Oil-trade ranking (seed: Iran): top-10 countries under submodular cut-cost. TL-HFD produces identical rankings across all $k$ and $\gamma$ configurations.}
\label{tab:oil-trade}
\begin{tabular}{c l l}
\toprule
Rank & S-HFD (baseline) & TL-SHFD (all configs) \\
\midrule
1  & Turkmenistan    & Turkmenistan    \\
2  & Iraq            & Iraq            \\
3  & Bahrain         & Guinea $\uparrow$         \\
4  & Brunei          & Bahrain $\downarrow$        \\
5  & Guinea          & Brunei $\downarrow$         \\
6  & Kazakhstan      & Kazakhstan      \\
7  & Libya           & Libya           \\
8  & Oman            & Qatar $\uparrow$          \\
9  & Qatar           & Congo $\uparrow$          \\
10 & Congo           & Mozambique (new)\\
\bottomrule
\end{tabular}
\vspace{2pt}
 
\emph{9 of 10 countries overlap. Oman (baseline rank 8) is replaced by Mozambique. All are oil-producing or oil-exporting nations with similar trade network positions to Iran.}
\end{table}

 \paragraph{Results.}
Table~\ref{tab:oil-trade} compares the submodular rankings. Under the
submodular cut-cost, TL-HFD produces identical rankings across all tested
\(k\) and \(\gamma\) configurations, indicating that thresholding does not
materially change the ranking on this dense instance. The top-10 ranking
overlaps with the HFD baseline on 9 of 10 countries, with only minor reordering:
Guinea moves from rank 5 to rank 3, and Oman drops out in favor of Mozambique.
Both methods identify countries with oil-trade network positions similar to
Iran, such as Turkmenistan, Iraq, Bahrain, Kazakhstan, Libya, and Qatar.

Under unit and cardinality cut-costs, the corresponding HFD and TL-HFD variants
rank countries near Iran that are less clearly related to Iran's oil-trade
position, such as Kenya, Bangladesh, Armenia, and Tanzania. This behavior is
consistent across methods and reflects model mismatch: unit and cardinality
cut-costs cannot encode the directional exporter--importer structure of the
4-way trade motifs. The submodular cut-cost is therefore necessary for
meaningful rankings on this dataset, as also observed by
\citet{fountoulakis2021local}.

\subsection{Hypergraph Stochastic Block Model (SBM)}
\label{app:sbm} 

We evaluate TL-HFD on synthetic hypergraphs generated from the hypergraph stochastic block model used by~\citet{fountoulakis2021local}. The hypergraph has \(n=100\) vertices partitioned into two equal blocks of size \(50\), and
all hyperedges are 3-uniform. We use the first block as the target cluster \(\mathcal T\). The intra-block triples are sampled with probability \(p=0.0765\), which gives approximately \(1500\) intra-block hyperedges in expectation for each block. The inter-block probability \(q\) is varied to control the ground-truth conductance \(\Phi(\mathcal T)\) of the target cluster.

\paragraph{Conductance sweep.}
The sweep over \(q\) corresponds to approximately \(500\) to \(9000\) expected
cross-block hyperedges, and hence to total expected hyperedge counts $\approx$
 \(3500\) to \(12000\). This varies the target-cluster conductance from
approximately \(0.10\) to \(0.50\). For each setting of \(q\), one hypergraph
is sampled, and every vertex in the target cluster is used once as the single
seed, yielding \(50\) single-seed trials per hypergraph. This keeps the number
of vertices, hyperedge cardinality, and intra-block density fixed while varying
the amount of cross-block noise and hence the separation of the planted
cluster.

\paragraph{Methods compared.}
We compare U-HFD against TL-UHFD with fixed top-$k$: activate the top $k$ boundary vertices ranked by the score
$s(u) = \kappa(u) \left( d_{\mathrm{in}}(u) / d(u) \right)^{\gamma}$, with $k \in \{3, 5, 10\}$ and $\gamma \in \{0.3, 0.5, 1\}$. Because cluster size is fixed at $|\mathcal T|=50$ and the boundary is bounded by $|V \setminus \mathcal T| \le 50$, we sweep $k$ directly rather than via the auto-$k$ rule.

\paragraph{Hypotheses.}
The synthetic SBM tests two hypotheses about TL-HFD's behavior:
\begin{enumerate}
     
\item \textbf{H1 (Effect of \(\gamma\)).} \label{h1}
The planted partition model has a clean, symmetric block structure. Unlike
real-world hypergraphs, it contains no semantically ambiguous bridge
vertices --- such as products belonging naturally to multiple categories ---
so we expect the structural commitment factor
\((d_{\mathrm{in}} / d)^\gamma\) to have only a minor effect, with most of
the variation between methods driven by the top-\(k\) activation cap itself.
SBM thereby provides a controlled setting for isolating the effect of
thresholding from the effect of structural commitment.

\item \textbf{H2 (Effect of conductance).} \label{h2}
We expect TL-HFD to recover the planted cluster well when conductance is low
(clean separation, little inter-cluster noise) and to degrade as conductance
increases. Specifically, restricting boundary expansion may slow recovery on
low-conductance clusters where broad expansion is beneficial, but the same
cap may help on high-conductance clusters by limiting over-inclusion of
non-target vertices. The crossover --- if it exists --- locates the
conductance regime in which top-\(k\) thresholding becomes net beneficial.
\end{enumerate}

\paragraph{Setup.}
We use $\sigma = 0.01$ and injection factor $\delta_{\mathrm{exp}} = 3$.
The choice $\sigma = 0.01$ matches the baseline's synthetic protocol and
is larger than the $\sigma = 10^{-4}$ used on real-world hypergraphs,
which have noisier and more heterogeneous local structure. Both U-HFD and
TL-HFD are run for $T = 500$ iterations; U-HFD converges in fewer than
$100$ iterations under its alternating-minimization update. We report
median F1 across the $50$ single-seed trials per hypergraph.

\paragraph{Results.}
We report results at low conductance (\(\Phi(T) \in [0.10, 0.20]\),
Table~\ref{tab:sbm-low}) and at high conductance
(\(\Phi(T) \in [0.40, 0.50]\), Table~\ref{tab:sbm-high}). In the intermediate
range \(\Phi(T) \in [0.20, 0.40]\), TL-UHFD with top-\(k\) tracks
U-HFD closely (within \(\pm 0.02\) median F1), so we omit these rows to keep the focus on the regimes where the methods diverge. 

\paragraph{Hypothesis \ref{h1}: Effect of \(\gamma\) .} Across the SBM sweep, varying \(\gamma \in \{0.3, 0.5, 1.0\}\) produces no
meaningful change in F1 scores (differences within \(\pm 0.005\) at all
conductance levels). This is consistent with the Hypothesis \ref{h1}: in a planted partition without bridge vertices, all boundary vertices have similar \(d_{\mathrm{in}} / d\) ratios, so the structural commitment factor rescales the score uniformly and does not change the top-\(k\) ranking. We therefore report all SBM results at \(\gamma = 1\).

\begin{table}[h]
\centering
\caption{Synthetic SBM: median F1 at \emph{low conductance}
(clean, well-separated clusters). The columns labeled topk=\(k\) give
committed top-\(k\) at fixed activation count \(k\). U-HFD achieves perfect F1; TL-UHFD remains very close, with a
small loss only near the higher end of this low-conductance range. }
\label{tab:sbm-low}
\begin{tabular}{c c c c c c }
\toprule
\#edges & \(\Phi(T)\) & U-HFD & topk=3 & topk=5 & topk=10  \\
\midrule
3500 & 0.095 & 1.000 & 1.00 & 1.00 & 1.00 \\
3600 & 0.111 & 1.000 & 1.00 & 1.00 & 1.00  \\
3700 & 0.126 & 1.000 & 1.00 & 1.00 & 1.00 \\
3800 & 0.141 & 1.000 & 1.00 & 1.00 & 1.00 \\
3900 & 0.154 & 1.000 & 1.00 & 1.00 & 1.00 \\
4000 & 0.167 & 1.000 & 0.99 & 0.99 & 0.99  \\
4100 & 0.179 & 1.000 & 0.99 & 0.99 & 0.99  \\
4200 & 0.191 & 1.000 & 0.99 & 0.99 & 0.99  \\
\bottomrule
\end{tabular}

\vspace{2pt}
\emph{On well-separated clusters, U-HFD achieves perfect F1. TL-HFD matches
within \(\le 0.015\) across all configurations, confirming that thresholding
causes only a small loss when clusters are easy to find.}
\end{table}

\begin{table}[h]
\centering
\caption{Synthetic SBM: median F1 at \emph{high conductance}
(noisy, poorly separated clusters). The columns labeled topk=\(k\) give
committed top-\(k\) at fixed activation count \(k\). Best per row in \textbf{bold}. TL-HFD improves over U-HFD in
the intermediate high-conductance regime.}
\label{tab:sbm-high}
\begin{tabular}{c c c c c c }
\toprule
\#edges & \(\Phi(T)\) & U-HFD & topk=3 & topk=5 & topk=10  \\
\midrule
7700  & 0.407 & 0.74 & \textbf{0.75} & \textbf{0.75} & 0.74  \\
8000  & 0.417 & 0.72 & 0.73 & \textbf{0.74} & 0.73  \\
8500  & 0.432 & 0.68 & \textbf{0.69} & \textbf{0.69} & \textbf{0.69}  \\
9000  & 0.445 & 0.64 & 0.66 & \textbf{0.68} & 0.67  \\
10000 & 0.468 & 0.59 & \textbf{0.60} & \textbf{0.60} & \textbf{0.60}  \\
11000 & 0.486 & \textbf{0.56} & 0.55 & 0.55 & \textbf{0.56}  \\
12000 & 0.502 & \textbf{0.51} & 0.50 & \textbf{0.51} & 0.50  \\
\bottomrule
\end{tabular}

\vspace{2pt}
\emph{TL-HFD improves over U-HFD in the intermediate range
\(\Phi(T) \approx 0.42\)--\(0.47\), with the largest rounded gain at
\(\Phi(T) = 0.445\), where U-HFD obtains \(0.64\) and top-\(k = 5\) obtains
\(0.68\). Beyond \(\Phi(T) \approx 0.49\), all methods degrade as the planted
cluster becomes weakly separated.}
\end{table}

\textbf{Low conductance.}
For \(\Phi(T) \le 0.20\), U-HFD achieves perfect F1 across all tested
conductance levels. TL-HFD remains very close, with median F1 between
\(0.99\) and \(1.00\). This is the cost of thresholding in a clean planted
partition: when most positively pushed boundary vertices are useful,
restricting expansion can only slow or slightly limit recovery.

\textbf{Higher conductance.}
For \(\Phi(T) \in [0.40, 0.50]\), the planted cluster is less well separated.
In the intermediate range \(\Phi(T) \approx 0.42\)--\(0.47\), TL-HFD improves
over U-HFD, with the largest rounded gain at \(9000\) edges
(\(\Phi(T) = 0.445\)): U-HFD obtains \(0.64\), while TL-HFD with
top-\(k = 5\) obtains \(0.68\). Beyond \(\Phi(T) \approx 0.49\), all methods
degrade together, reflecting the weak separation of the planted cluster.

\paragraph{Hypothesis  \ref{h2}: Effect of conductance.}
The two regimes in Tables~\ref{tab:sbm-low} and~\ref{tab:sbm-high} confirm H2
and locate the crossover. On low-conductance clusters
(\(\Phi(T) \le 0.20\)), U-HFD achieves perfect F1 across all tested levels,
and TL-HFD remains very close (median F1 between \(0.99\) and \(1.00\), gap
\(\le 0.015\)). This is the predicted cost of thresholding in a clean planted
partition: when most positively pushed boundary vertices belong to the
target, restricting expansion can only slow or slightly limit recovery. On
high-conductance clusters (\(\Phi(T) \in [0.40, 0.50]\)), the planted cluster
is less well separated, and the relative ordering between U-HFD and TL-HFD
flips in the intermediate range \(\Phi(T) \approx 0.42\)--\(0.47\): TL-HFD
improves over U-HFD, with the largest rounded gain at \(\Phi(T) = 0.445\),
where U-HFD obtains \(0.64\) and TL-HFD with top-\(k = 5\) obtains \(0.68\).
Beyond \(\Phi(T) \approx 0.49\), all methods degrade together as the planted
structure becomes too weak for any local method to recover. The crossover
near \(\Phi(T) \approx 0.40\) therefore locates the regime in which
top-\(k\) thresholding becomes net beneficial: TL-HFD's restriction protects
against over-inclusion of cross-block vertices once the boundary expansion
itself becomes noisy.

\paragraph{Tradeoff summary.}
The synthetic experiment isolates the role of top-\(k\) thresholding under a
clean planted-partition model. On well-separated clusters, thresholding
slightly hurts because broad expansion is beneficial. On noisier clusters,
thresholding can help by limiting over-expansion. The larger gains observed
on real-world datasets are consistent with the additional role of committed
weighting, which is designed to filter structurally ambiguous boundary
vertices; this effect is largely absent in the planted SBM.

\paragraph{Visualizing the SBM sweep.}
Figures~\ref{fig:sbm-f1-low}--\ref{fig:sbm-cond-high} present the SBM
results graphically, complementing the median values in the tables. Each figure plots a metric (F1 or output conductance) against the ground-truth target-cluster conductance $\Phi(T)$. Solid and dashed lines show the median across the $50$ single-seed trials per hypergraph, with line style and color identifying the method (U-HFD or TL-HFD at a specific top-$k$ cap). The shaded band around each line shows the interquartile range ($25$th--$75$th percentile) across the same trials: a narrow band indicates insensitivity to seed choice on that hypergraph, while a wide band indicates substantial seed-to-seed variability. We report each metric across two panels: a low-$\Phi$ panel covering $\Phi(T) \in [0.10, 0.20]$ (clean, well-separated clusters) and a high-$\Phi$ panel covering $\Phi(T) \in [0.40, 0.50]$ (noisy, weakly separated clusters). The intermediate range $\Phi(T) \in [0.20, 0.40]$ is omitted because all methods track each other within $\pm 0.02$ median F1 there. For F1, higher is better; for output conductance, lower is better. The methods are visually distinguished by both color and dash pattern, so the figures remain legible in grayscale. Figures~\ref{fig:sbm-f1-low}, and ~\ref{fig:sbm-f1-high} show the F1 scores for U-HFD and TL-HFD  under low conductance setting. Figures ~\
~\ref{fig:sbm-cond-low}, and ~\ref{fig:sbm-cond-high} depicts how the methods track the ground truth conductance. 

\begin{figure}[!htbp]
\centering
\includegraphics[width=0.65\linewidth]{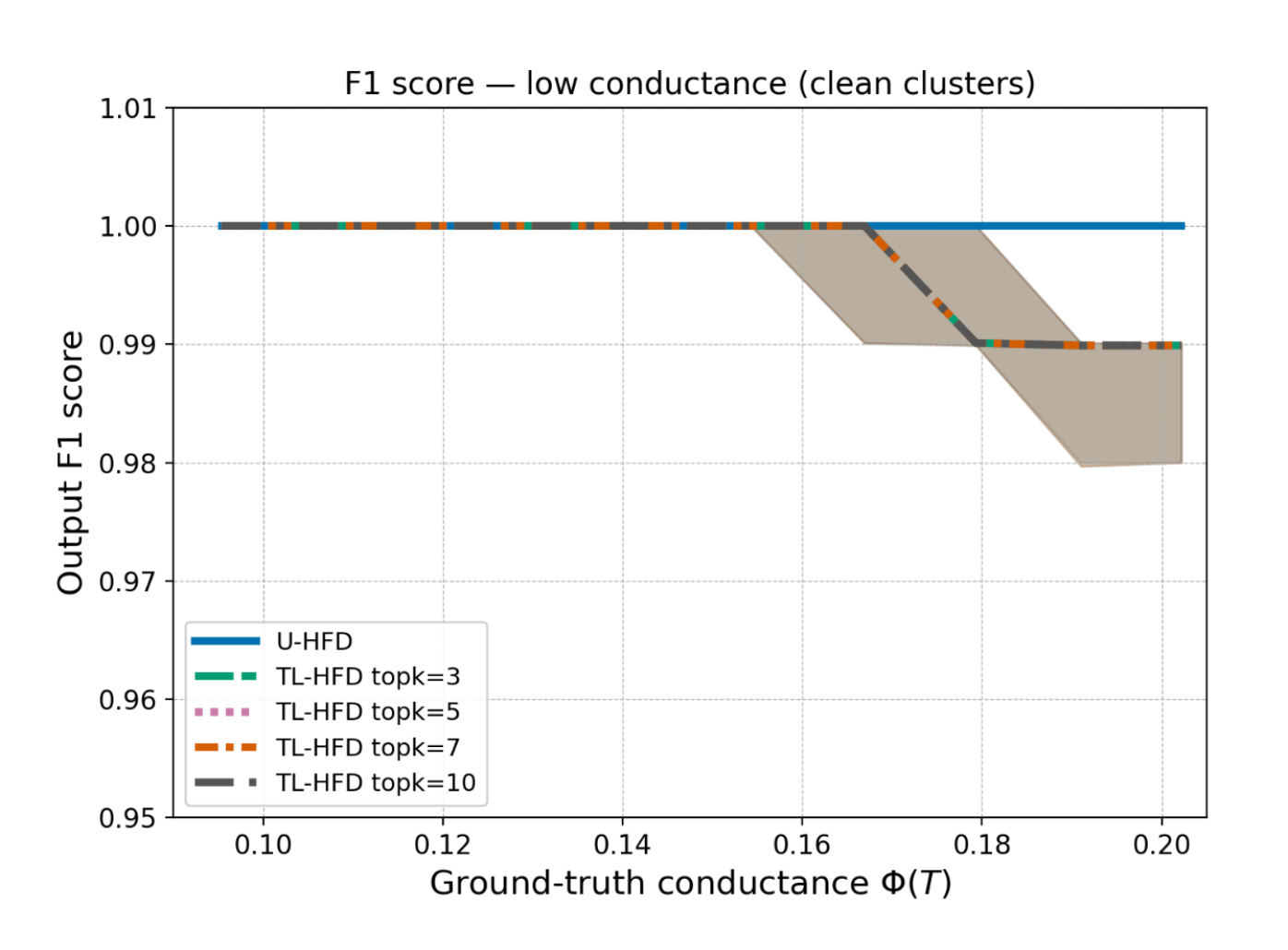}
\caption{Synthetic SBM: F1 score at low conductance ($\Phi(T) \in [0.10, 0.20]$, easy clusters). U-HFD achieves F1 $= 1.0$; TL-HFD with $k \in \{3, 5, 7, 10\}$ matches within $\pm 0.015$, indicating that top-$k$ thresholding incurs only a small cost when clusters are clean.}
\label{fig:sbm-f1-low}
\end{figure}

\begin{figure}[!htbp]
\centering
\includegraphics[width=0.65\linewidth]{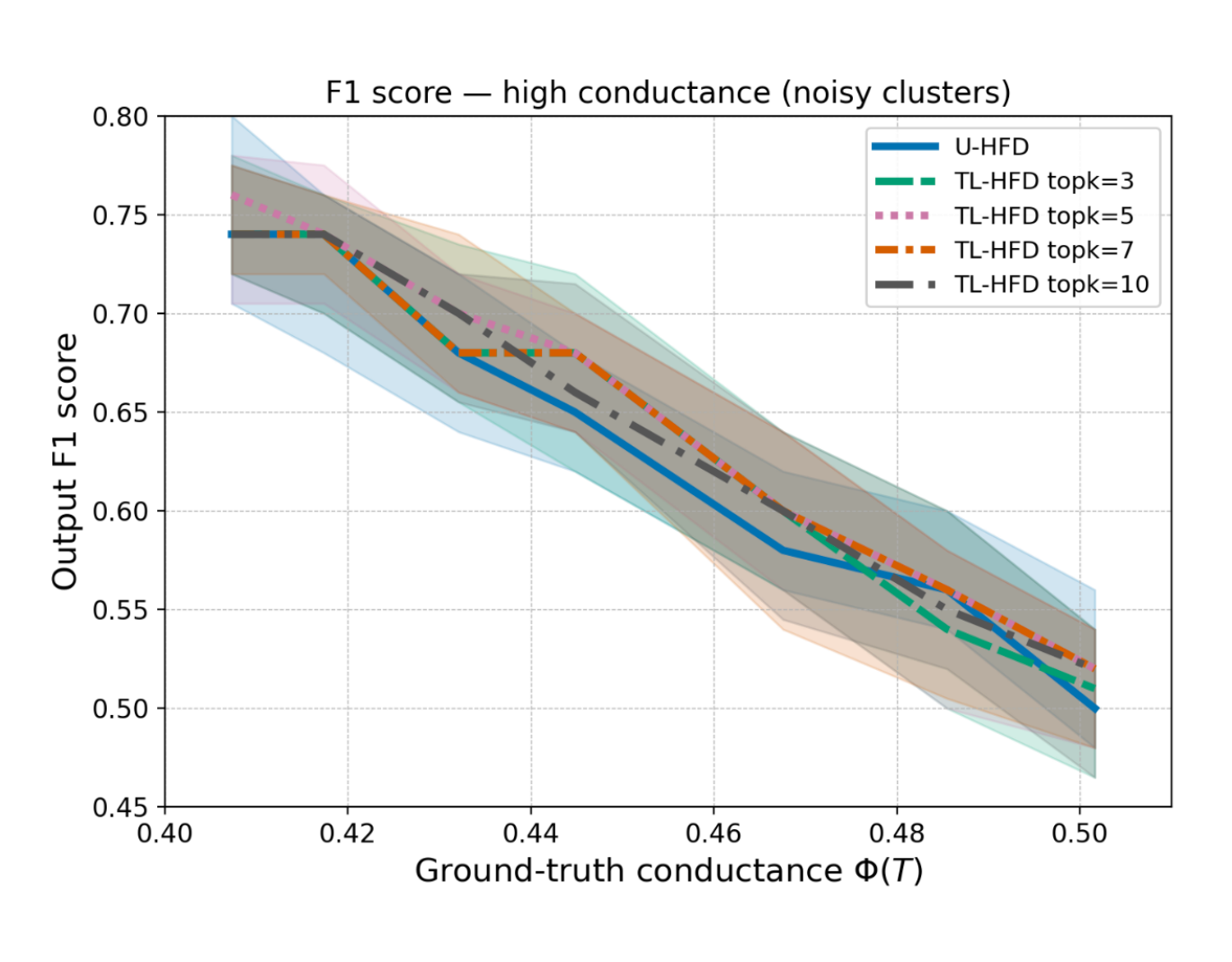}
\caption{Synthetic SBM: F1 score at high conductance ($\Phi(T) \in [0.40, 0.50]$, noisy clusters). TL-HFD overtakes U-HFD for $\Phi(T) \in [0.42, 0.47]$, with the largest gain at $\Phi(T) = 0.445$. Beyond $\Phi(T) \approx 0.49$, all methods degrade together as the planted structure weakens.}
\label{fig:sbm-f1-high}
\end{figure}

\begin{figure}[!htbp]
\centering
\includegraphics[width=0.65\linewidth]{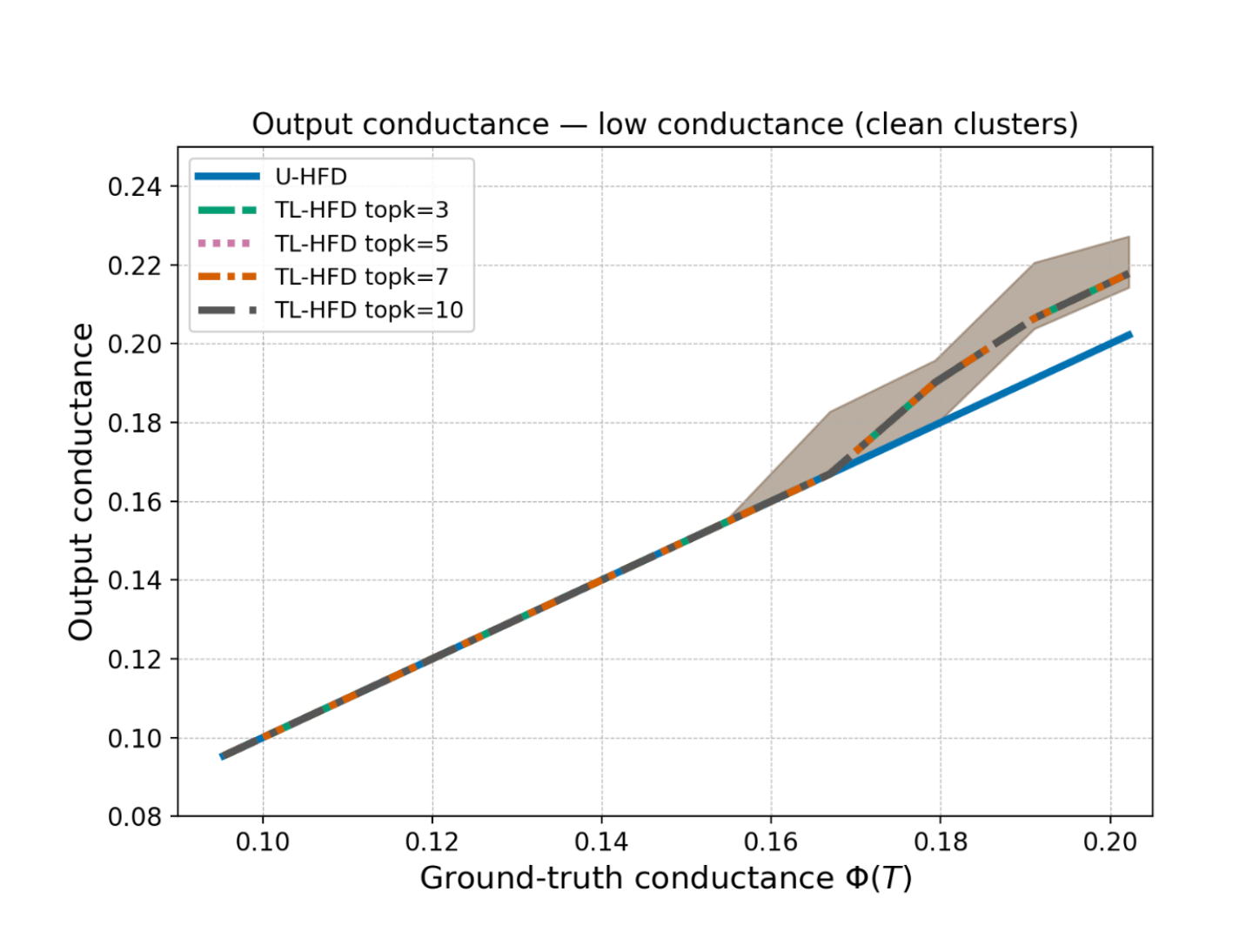}
\caption{Synthetic SBM: output conductance at low ground-truth conductance ($\Phi(T) \in [0.10, 0.20]$). All methods track the ground-truth conductance closely.}
\label{fig:sbm-cond-low}
\end{figure}

\begin{figure}[!htbp]
\centering
\includegraphics[width=0.65\linewidth]{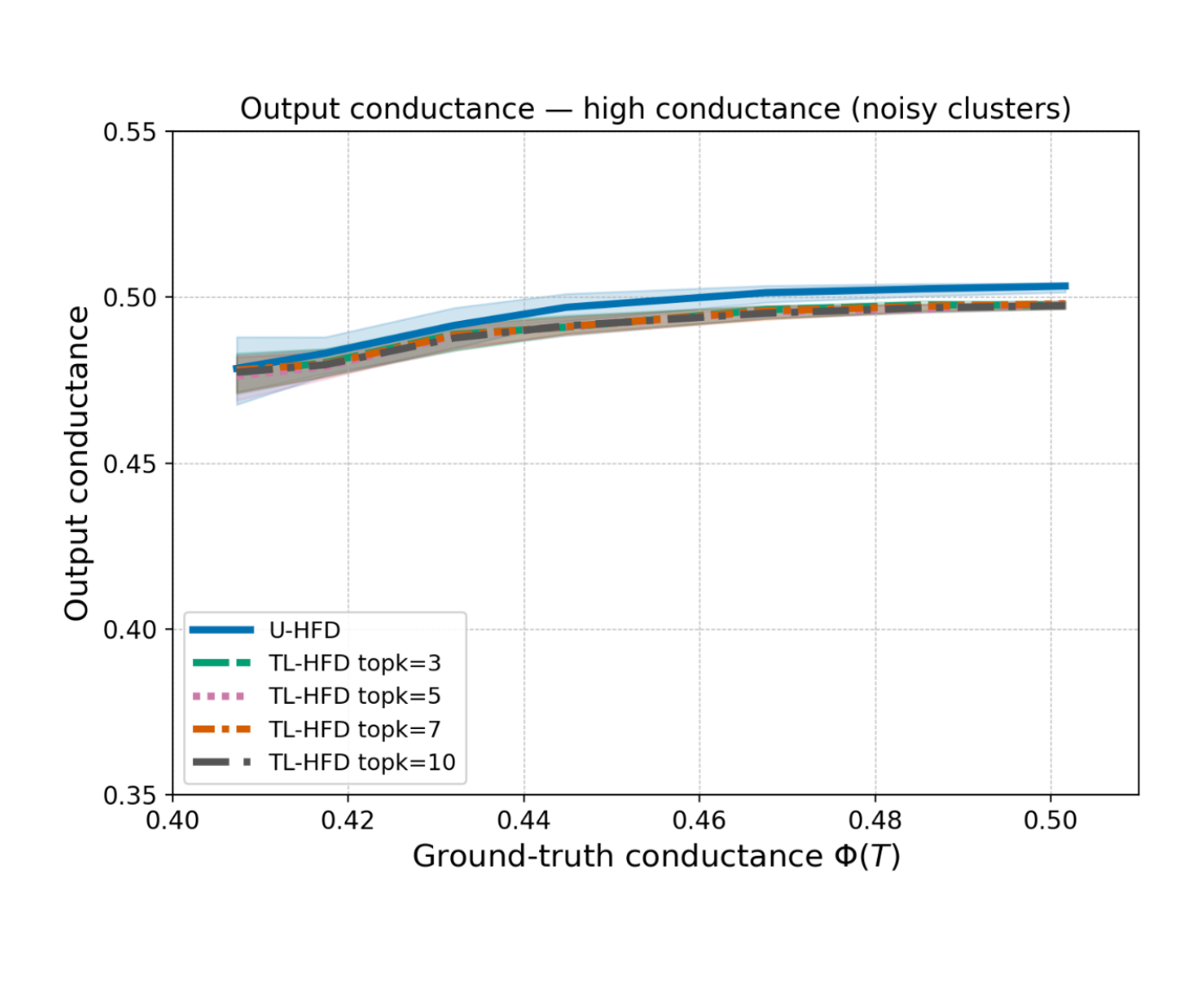}
\caption{Synthetic SBM: output conductance at high ground-truth conductance ($\Phi(T) \in [0.40, 0.50]$). TL-HFD produces slightly lower (better) conductance clusters than U-HFD in this regime.}
\label{fig:sbm-cond-high}
\end{figure}

\begin{figure}[t]
    \centering
    \includegraphics[width=\linewidth]{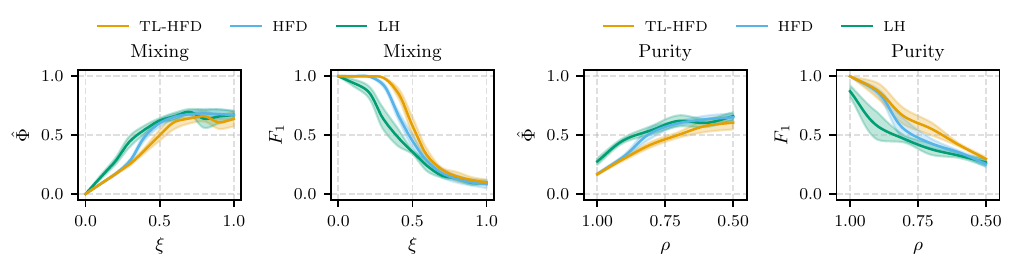}
    \caption{h-ABCD synthetic benchmark: TL-HFD vs.\ HFD and LH across two robustness sweeps. \textbf{Top row (mixing):} output conductance $\hat{\Phi}$ and F1 as a function of the mixing parameter $\xi$ with purity $\rho = 1$. At $\xi = 0$ all hyperedges are internal to a single community; at $\xi = 1$ all hyperedges are drawn uniformly from the full vertex set. \textbf{Bottom row (purity):} $\hat{\Phi}$ and F1 as a function of $\rho$ with mixing $\xi = 0.2$. At $\rho = 1$ each community edge is a clique within one community; as $\rho \to 0.5$, community edges contain only a bare majority of same-community nodes. Each panel shows medians across communities and $10$ independently generated hypergraphs, with shaded regions indicating confidence bands across the same trials. TL-HFD maintains higher F1 in the intermediate-mixing regime (top right) and shows a widening advantage as community edges become more heterogeneous (bottom right), consistent with the structural-commitment scoring's de-prioritization of boundary vertices that span communities.}

    \label{fig:habcd-mix-pur-unit}
\end{figure}

\subsection{Experiments on h-ABCD Synthetic Datasets}\label{app:synth}

\paragraph{Synthetic data generation.}
We use the h-ABCD benchmark~\citep{kaminski2023hypergraph}, the hypergraph extension of the ABCD random graph model~\citep{kaminski2021artificial}. We generate hypergraphs on $n = 5000$ nodes partitioned into communities whose sizes follow a $(1.9, 100, 500)$ truncated power-law distribution (TPD) with exponent $1.9$ and range $[100, 500]$. Each node has a degree drawn from a $(2.2, 3, 50)$-TPD, and hyperedge sizes follow a $(2.0, 3, 50)$-TPD, so edges range from small triangles to large group interactions.

We generate two kinds of hyperedges. \emph{Community edges} draw their
members primarily from a single community; \emph{background edges} draw
members uniformly from the full vertex set. The fraction of each node's
degree assigned to background edges is the \emph{mixing parameter}
$\xi \in [0, 1]$: at $\xi = 0$ all hyperedges are internal and community structure is perfect; at $\xi = 1$ all edges are background and the structure is erased.

The \emph{purity} parameter $\rho \in [0.5, 1]$ controls the homogeneity of community edges. A community edge of size $d$ has its number of within-community members drawn from $\mathrm{Binomial}(d, \rho)$, restricted to a strict majority $c > d/2$. At $\rho = 1$ each community edge is a clique within one community; as $\rho \to 0.5$ edges become increasingly mixed, containing only a bare majority of same-community nodes.

\paragraph{Protocol.}
For each community we randomly select one of its members as the seed.
Hyperparameters for each method are tuned per community via grid search
to minimize output conductance --- an unsupervised criterion matching
the protocol used for real-world datasets. All reported values are medians and confidence bands over communities across $10$ trials.

\paragraph{Robustness to cluster strength (Mixing).}
We vary $\xi$ uniformly over $11$ values in $[0, 1]$ while holding
$\rho = 1$. This sweep tests whether each method can recover a planted
community as the surrounding noise level increases: at low $\xi$ the
target is well-separated from the background, while at high $\xi$
community structure approaches uniform random. Figure~\ref{fig:habcd-mix-pur-unit} shows F1 and output conductance across the sweep. All methods recover the planted community in the low-noise regime; the methods diverge in the intermediate regime, where TL-HFD's structural-commitment scoring filters boundary vertices whose incidences are dominated by background edges. In the high-noise regime, no local method can recover the community, and all methods degrade together.

\paragraph{Robustness to hyperedge heterogeneity (Purity).}
We fix $\xi = 0.2$ and vary $\rho$ uniformly in $[0.5, 1.0]$. This sweep
isolates within-edge heterogeneity from cross-community noise: the total
background-edge mass is held fixed, but community edges become
progressively more mixed as $\rho$ decreases. Figure~\ref{fig:habcd-mix-pur-unit} shows F1 and output conductance across the sweep. The decrease in $\rho$ exposes genuine bridge vertices --- nodes that participate
heavily in mixed-membership hyperedges --- which are precisely the
boundary vertices that TL-HFD's $(d^{(t)}_{\text{in}}/d)^{\gamma}$
weighting (Algorithm~\ref{alg:pq-psgd-theory}, line~10) is designed to
de-prioritize. The widening gap between TL-HFD and the baselines as
$\rho \to 0.5$ is consistent with this mechanism: thresholded selection
avoids activating bridge vertices that other methods would include,
leading to lower output conductance and higher F1 on the recovered
cluster.

\subsection{Matched-Convergence and Support-Locality on Amazon-reviews}
\label{app:runtime-locality}

In this section we compare the computational behavior of TL-HFD and HFD on the
Amazon-reviews dataset using two diagnostics. First, we ask how many TL-HFD
iterations are needed before its returned cluster matches the final F1 score of
HFD. Second, we compare the size of the nonzero state maintained by each method
during the run. The latter should be interpreted as a support-locality measure,
not as a complete per-iteration runtime bound: TL-HFD also scores the current
boundary \(\partial A^{(t)}\), so the full scanned region is
\(L_t = A^{(t)} \cup \partial A^{(t)}\).

Throughout this section, the active set at iteration \(t\) denotes the set of
vertices on which the method stores nonzero state:
\begin{itemize}
    \item For HFD,
    \[
    A_t^{\mathrm{HFD}} = \{v \in \V : e_t(v) > 0\},
    \]
    where \(e_t(v)\) is the excess flow at vertex \(v\) after the \(t\)-th
    alternating-minimization step.

    \item For TL-HFD,
    \[
    A_t^{\mathrm{TL}} = \{v \in \V : x_t(v) > 0\},
    \]
    where \(x_t\) is the projected-subgradient iterate.
\end{itemize}

We also track the returned cluster at iteration \(t\), defined for both methods
as the sweep-cut set with minimum conductance observed up to iteration \(t\).
This is the cluster the method would output if stopped at iteration \(t\).

All results in this section use the nine Amazon categories from Section~\ref{sec:exp-amazon}. Each category is evaluated using \emph{$10$} random single-seed trials drawn from the same $\min(|\mathcal{T}|, 500)$-seed pool used in Table~\ref{tab:amazon-f1}. The smaller sample is sufficient for illustrating per-iteration convergence behavior; precise F1 estimates are reported in Table~\ref{tab:amazon-f1}. All other settings match the headline protocol: TL-HFD uses the conductance-selected fraction $f$ for each category from Table~\ref{tab:amazon-frac}, $\sigma = 10^{-4}$, and runs for $1000$ iterations; HFD is run for $200$ iterations (HFD typically converged in fewer than 200 iterations). We report
medians across seeds.

\subsubsection{Matched-Convergence Analysis}\label{app:analysis}

A direct comparison of total wall-clock time is not the most informative
comparison, because HFD uses a primal-dual alternating-minimization routine that
typically converges in few iterations, whereas TL-HFD uses projected
subgradient steps and is run for many more iterations. We therefore compare the
time required for TL-HFD to reach the final F1 score of HFD.

For each seed and category, let \(\mathrm{F1}^{\mathrm{HFD}}\) be the F1 score
of HFD's final returned cluster after \(200\) iterations. Let
\(\mathrm{F1}^{\mathrm{TL}}_t\) be the F1 score of TL-HFD's returned cluster
when stopped at iteration \(t\). We define
\[
t^\star
=
\min\{t : \mathrm{F1}^{\mathrm{TL}}_t \ge \mathrm{F1}^{\mathrm{HFD}}\},
\]
when such an iteration exists. The projected wall-clock time at \(t^\star\) is
estimated by linearly scaling the measured \(1000\)-iteration TL-HFD runtime:
\[
\mathrm{time}_{\mathrm{proj}}
=
\frac{t^\star}{1000}\cdot \mathrm{time}_{\mathrm{TL},1000}.
\]
This projection is a coarse estimate. The activated set grows monotonically, but the per-iteration cost is governed by the scanned region $L_t = A^{(t)} \cup \partial A^{(t)}$, which can fluctuate; we therefore treat the projection as an approximate  diagnostic.

Table~\ref{tab:matched-convergence} reports HFD's final F1 and wall-clock
time, TL-HFD's final F1 after \(1000\) iterations, the median \(t^\star\) over that matched 
 seeds, and the projected time at \(t^\star\).

\begin{table}[h]
\centering
\caption{Matched-convergence analysis on Amazon-reviews, computed over
$10$ single-seed trials per cluster (subsample of the $\min(|\mathcal{T}|, 500)$-seed
protocol used in Table~\ref{tab:amazon-f1}). $t^\star$ is the median first
iteration at which TL-HFD's returned cluster reaches HFD's final F1, taken
over seeds where this occurs within $1000$ iterations; $t^\star = \text{---}$
indicates TL-HFD never reached HFD's F1 on any seed. Projected time scales
TL-HFD's 1000-iteration wall-clock by $t^\star/1000$.}

\label{tab:matched-convergence}
\small
\begin{tabular}{l c c c c c c}
\toprule
Cluster & \(\delta\) & HFD F1 & HFD time (s) & TL F1
& \(t^\star\) & Projected (s)  \\
\midrule
Amazon Fashion          & 200 & 0.145 &  5.81 & 0.278 &  2  & 0.01  \\
All Beauty              & 200 & 0.061 &  7.16 & 0.087 &  7  & 0.12  \\
Appliances              & 200 & 0.476 &  1.09 & 0.648 & 10  & 0.01  \\
Gift Cards              & 200 & 0.811 &  5.57 & 0.887 &  7  & 0.11  \\
Magazine Subscriptions  & 200 & 0.710 &  6.09 & 0.777 & 11  & 0.10 \\
Industrial \& Scientific & 50 & 0.010 & 23.15 & 0.008 & --- & ---  \\
Luxury Beauty           &  50 & 0.046 & 12.79 & 0.132 & 10  & 0.91  \\
Prime Pantry            &  50 & 0.799 & 31.00 & 0.755 &  7  & 1.32 \\
Software                &  50 & 0.077 &  8.11 & 0.269 & 59  & 1.67 \\
\bottomrule
\end{tabular}
\end{table}

\paragraph{Findings.}
On six of the nine categories---Amazon Fashion, Appliances, Gift Cards,
Magazine Subscriptions, Luxury Beauty, and Software---TL-HFD reaches HFD's
final F1 on a majority of seeds, with median \(t^\star \le 59\) among matched
seeds. The projected wall-clock time is often one to two orders of magnitude
smaller than HFD's runtime. For example, on Gift Cards, TL-HFD
matches HFD's final F1 by iteration \(7\), corresponding to approximately
\(0.11\) seconds of projected wall-clock time, compared with \(5.57\) seconds
for HFD. Continuing TL-HFD beyond this point further improves the median F1 on
this category from \(0.811\) to \(0.887\). The remaining categories illustrate limitations of the top-\(k\) restriction.
Industrial \& Scientific is difficult for both methods (HFD F1 \(=0.010\)),
and TL-HFD performance is not competitive on this category (F1 \(=0.008\)), suggesting
the cluster lacks local structure that either method can recover. 

\subsubsection{Active-Set Growth}\label{app:active_set_growth}

We next compare the size of the nonzero state maintained by each method during
the run. For HFD, \citet{fountoulakis2021local} provide a sparsity guarantee
for the final active set, but this does not directly control the size of
intermediate active sets. In contrast, since all Amazon trials are single-seed
runs, TL-HFD activates at most \(k\) new boundary vertices per iteration, so
\[
|A_t^{\mathrm{TL}}| \le 1+k t,
\]
with equality only if every activated coordinate remains positive.

Table~\ref{tab:active-set} compares HFD's peak and final active-set sizes with
TL-HFD's final support size, reporting both absolute counts and ratios to the
target-cluster size \(|\mathcal T|\).

\begin{table}[h]
\centering
\caption{Active-set sizes across  HFD and  TL-HFD iterations
on Amazon-reviews, reported as medians across the $10$ seeds used in this
sub-experiment (cf.~Table~\ref{tab:matched-convergence}). Parentheses give the
ratio to the target cluster size \(|\mathcal T|\). \(|A_{\max}^{\mathrm{HFD}}|\) is the
largest HFD active set during the run, \(|A_T^{\mathrm{HFD}}|\) is HFD's final
active set, and \(|A_T^{\mathrm{TL}}|\) is TL-HFD's final support size. The
column \(k\) reports the per-iteration activation cap used by TL-HFD on each
category.}
\label{tab:active-set}
\small
\begin{tabular}{l c c c c c}
\toprule
Cluster & \(|\mathcal T|\) & \(|A_{\max}^{\mathrm{HFD}}|\)
& \(|A_T^{\mathrm{HFD}}|\) & \(|A_T^{\mathrm{TL}}|\) & \(k\) \\
\midrule
Amazon Fashion          &   31 &  3596 (\(116\times\)) &  3408 (\(110\times\)) &    24 (\(0.77\times\)) &   15 \\
All Beauty              &   85 &  4519 (\(53\times\))  &  3444 (\(41\times\))  &     5 (\(0.06\times\)) &    4 \\
Appliances              &   48 &   202 (\(4.2\times\)) &   195 (\(4.1\times\)) &    26 (\(0.54\times\)) &    5 \\
Gift Cards              &  148 &  2870 (\(19\times\))  &  1992 (\(13\times\))  &   106 (\(0.72\times\)) &   59 \\
Magazine Subscriptions  &  157 &  2274 (\(14\times\))  &  1456 (\(9.3\times\)) &   121 (\(0.77\times\)) &   69 \\
Industrial \& Scientific & 5334 & 14354 (\(2.7\times\)) & 12504 (\(2.3\times\)) &    73 (\(0.01\times\)) &   72 \\
Luxury Beauty           & 1581 &  7313 (\(4.6\times\)) &  7083 (\(4.5\times\)) &  2024 (\(1.28\times\)) & 1404 \\
Prime Pantry            & 4970 & 33759 (\(6.8\times\)) & 29405 (\(5.9\times\)) &  4980 (\(1.00\times\)) & 2622 \\
Software                &  802 &  4350 (\(5.4\times\)) &  4211 (\(5.3\times\)) &   371 (\(0.46\times\)) &  356 \\
\bottomrule
\end{tabular}
\end{table}

\paragraph{Findings.}
HFD's active set often expands far beyond the target cluster size during its
intermediate iterations. This effect is most pronounced on small target
categories: Amazon Fashion peaks at \(116\times |\mathcal T|\), All Beauty at
\(53\times |\mathcal T|\), and Gift Cards at \(19\times |\mathcal T|\). HFD's active set can
then shrink between its peak and final iteration as excess is redistributed,
but the final HFD active set remains larger than the target cluster on all nine
categories.

TL-HFD maintains substantially smaller nonzero supports. On seven of nine
categories, the final TL-HFD support is strictly below the target size. The two
exceptions are Luxury Beauty, where the final support is \(1.28\times |\mathcal T|\),
and Prime Pantry, where it is essentially equal to the target size
(\(4980\) versus \(4970\) vertices). These numbers should be interpreted as
support-locality measurements rather than full scanned-work measurements,
because TL-HFD still scores the current boundary at each iteration.

\paragraph{Discussion.}

The two diagnostics tell a consistent story. The matched-convergence
analysis shows that TL-HFD often reaches HFD-level F1 well before the end
of its $1000$-iteration run. The active-set analysis shows that it does so
while maintaining a much smaller nonzero state than HFD. On most
categories, the final TL-HFD support is below the target-cluster size; the
exceptions are Luxury Beauty (support $1.28\times |\mathcal{T}|$) and Prime
Pantry (support $\approx |\mathcal{T}|$).

These results highlight the tradeoff between the two solvers. HFD converges in
fewer iterations but can maintain large transient active sets. TL-HFD uses a
slower first-order method, but its top-\(k\) boundary activation directly
controls the growth of the maintained support. This support-locality property
does not by itself imply a full worst-case runtime bound, since each TL-HFD
iteration still scores the current boundary. Together, however, these
diagnostics indicate that on the Amazon categories where TL-HFD succeeds, its
top-\(k\) schedule achieves fast F1 progress while keeping the nonzero support
small.

\subsection{Sensitivity to the structural-commitment exponent $\gamma$}
\label{app:gamma-ablation}

The boundary score in Algorithm~\ref{alg:pq-psgd-theory} reweights the gradient push $\kappa(u)$ by a structural-commitment factor:
\begin{equation*}
s^{(t)}(u) \;=\; \kappa^{(t)}(u)\left(\frac{d_{\text{in}}^{(t)}(u)}{d_u}\right)^{\!\gamma},
\end{equation*}
where the exponent $\gamma$ controls how strongly the score favors boundary vertices that are well attached to the current active region. The Amazon-reviews results in Section~\ref{sec:exp-amazon} use $\gamma = 0.5$ as it provide the lowest conductance. This appendix studies how the cluster-recovery quality of TL-HFD varies with $\gamma$.

\paragraph{Setup.} We sweep $\gamma \in \{0.25, 0.50, 0.75, 1.00\}$ on the nine Amazon-reviews clusters of Section~\ref{sec:experiments}. Each ablation run uses $n=50$ random target seeds per cluster and the same fraction grid as Table~\ref{tab:amazon-frac}. All other hyperparameters ($\sigma$, $\delta_{\exp}$, iteration budget, cut-cost) are unchanged. For each cluster and each $\gamma$, we report the F1 and output conductance at the fraction $f$ selected by minimum output conductance over the sweep, with ties broken by the smallest $f$. We analyze sensitivity in relative terms with respect to $\gamma = 0.5$ (for the ablation run), defining
\begin{equation*}
\Delta_{F1}(\gamma, c) \;=\; \frac{F1(\gamma, c) - F1(0.5, c)}{F1(0.5, c)}, \qquad
\Delta_{\Phi}(\gamma, c) \;=\; \frac{\Phi(\gamma, c) - \Phi(0.5, c)}{\Phi(0.5, c)}.
\end{equation*}

\paragraph{Relative F1 change.} Table~\ref{tab:gamma-f1} reports $\Delta_{F1}$ relative to $\gamma = 0.5$. Industrial \& Scientific has $F1(\gamma) < 0.04$ at every $\gamma$; we mark its row degenerate, since relative changes amplify near-zero F1 values without conveying meaningful sensitivity.

\begin{table}[h]
\centering
\caption{Amazon-reviews: relative F1 change vs. $\gamma = 0.5$, computed at the conductance-selected operating point of each $\gamma$. ``Range'' is the spread of $\Delta_{F1}$ across the four $\gamma$ values.}
\label{tab:gamma-f1}
\begin{tabular}{lrrrr}
\toprule
Cluster & $\Delta_{F1}(0.25)$ & $\Delta_{F1}(0.75)$ & $\Delta_{F1}(1.00)$ & Range \\
\midrule
Fashion & $-7.6\%$ & $0.0\%$ & $0.0\%$ & $7.6\%$ \\
All Beauty & $+183.5\%$ & $+7.2\%$ & $+237.1\%$ & $244.7\%$ \\
Appliances & $-4.5\%$ & $-4.5\%$ & $-4.5\%$ & $4.5\%$ \\
Gift Cards & $-2.6\%$ & $+1.4\%$ & $+3.2\%$ & $5.8\%$ \\
Magazine & $-0.4\%$ & $+7.9\%$ & $+12.0\%$ & $12.4\%$ \\
Indust./Sci. & \multicolumn{3}{c}{degenerate ($F1 < 0.04$ at all $\gamma$)} & --- \\
Lux. Beauty & $-70.6\%$ & $-57.3\%$ & $+57.6\%$ & $128.2\%$ \\
Pantry & $-1.5\%$ & $+34.9\%$ & $+103.8\%$ & $ 138.7\%$ \\
Software & $3.1\%$ & $-1.8\%$ & $-5.6\%$ & $8.7\%$ \\
\bottomrule
\end{tabular}
\\[2pt]
\end{table}

\paragraph{Sensitivity classes.} The two tables together induce a natural grouping of clusters by their $\gamma$-response:

\begin{itemize}
    \item \textbf{Class A ($\gamma$-robust).} Fashion, Appliances, Gift Cards, and Software exhibit relative F1 ranges below $10\%$ across the entire $\gamma$ grid. The conductance-selected operating point and the resulting clusters are largely insensitive to $\gamma$ on these instances.
    
    \item \textbf{Class B (monotone in $\gamma$).} Magazine, and Pantry, with both F1 and $\Phi$ improving toward larger $\gamma$. The structural-commitment factor de-prioritizes boundary vertices weakly attached to the active region, and these clusters benefit from that filtering at every step of the grid.
    
    \item \textbf{Class C (non-monotone).} All Beauty and Lux. Beauty exhibit the largest F1 ranges in Table~\ref{tab:gamma-f1} ($244.7\%$ and $128.2\%$, respectively) with non-monotone shapes. Both clusters have a local F1 minimum at or near $\gamma = 0.5$. We also observed that the corresponding conductance values are non-monotone, indicating that the F1 swings reflect genuine changes in the conductance landscape rather than seed-sampling noise alone.

    \item \textbf{Class D (degenerate).} Indust./Sci.\ has $F1(\gamma) < 0.04$ at every $\gamma$. Its conductance trend is approximately monotone, but its F1 is too small for relative changes to be meaningful.
\end{itemize}

\paragraph{Interpretation.} The ablation supports the following  observations.

First, on the four Class A clusters, $\gamma$ is effectively a free parameter: the conductance-selected operating point and the resulting F1 are stable within the noise floor across a $4\times$ range of $\gamma$. The value $\gamma = 0.5$ is a defensible default for these clusters.

Second, the structural-commitment factor and the conductance-based selection rule do interact. On Class B and Class C clusters, varying $\gamma$ changes both the score ranking of boundary vertices and the operating point that the selection rule lands on. The two effects compound: the same $\gamma$ that produces a different score ordering also induces a different conductance landscape over the f-grid, and the rule then selects a different f. This is visible in the non-monotone Class C entries, where small changes in $\gamma$ produce large changes in both $F1$ and conductance.

\paragraph{Compute resources.} All experiments were run on a single core of a 12-core Apple M4 Pro processor with 24 GB of RAM. No GPU acceleration was used; both TL-HFD and all baselines are implemented in Julia and executed sequentially.

\newpage

\end{document}